\documentclass{article}
\usepackage[margin=1in]{geometry}    
\usepackage[T1]{fontenc}
\usepackage[utf8]{inputenc}
\usepackage[cal=euler]{mathalpha}
\usepackage{cochineal}
\usepackage{amsmath,amssymb,amsthm,mathtools,bm}
\renewcommand{\epsilon}{\varepsilon}  
\usepackage{booktabs}
\usepackage{capt-of}
\usepackage{longtable}
\usepackage{tabularx}
\usepackage{enumitem}
\usepackage{microtype}
\usepackage{float}
\usepackage{hyperref}
\usepackage[numbers,sort&compress]{natbib}
\hypersetup{colorlinks=true,linkcolor=blue!60!black,citecolor=blue!60!black,urlcolor=blue!60!black}
\usepackage{xcolor}

\mathtoolsset{showmanualtags}
\allowdisplaybreaks
\title{Exact information accounting for SGD methods}
\author{Akshay Balsubramani \\ {\small \texttt{akshay@vac.bio}}}
\date{}

\newtheorem{theorem}{Theorem}[section]
\newtheorem{proposition}[theorem]{Proposition}
\newtheorem{corollary}[theorem]{Corollary}

\newcommand{\R}{\mathbb{R}}
\newcommand{\E}{\mathbb{E}}
\newcommand{\KL}{\mathrm{KL}}
\newcommand{\tr}{\mathrm{tr}}
\newcommand{\diag}{\mathrm{diag}}
\newcommand{\Var}{\mathrm{Var}}
\newcommand{\Cov}{\mathrm{Cov}}
\newcommand{\Bias}{\mathrm{Bias}}
\newcommand{\calF}{\mathcal{F}}
\newcommand{\calN}{\mathcal{N}}
\newcommand{\calB}{\mathcal{B}}

\newcommand{\calK}{\mathcal{K}}
\newcommand{\ip}[2]{\left\langle #1,#2\right\rangle}

\newcommand{\half}{\frac12}

\providecommand{\akshay}[1]{}
\providecommand{\vac}[1]{}
\providecommand{\dd}{\mathrm{d}}
\providecommand{\ip}[1]{\langle #1 \rangle}

\providecommand{\KL}{\mathrm{KL}}

\providecommand{\Gibbs}{\Pi}

\providecommand{\citep}[1]{[\textbf{#1}]}
\providecommand{\cite}[1]{[\textbf{#1}]}
\makeatletter
\@ifundefined{argmax}{}{}
\@ifundefined{argmin}{}{}
\@ifundefined{diam}{}{}
\@ifundefined{rank}{}{}
\@ifundefined{supp}{}{}
\@ifundefined{tr}{\DeclareMathOperator{\tr}{tr}}{}
\@ifundefined{sign}{\DeclareMathOperator{\sign}{sign}}{}
\makeatother

\providecommand{\E}{\mathbb{E}} \providecommand{\R}{\mathbb{R}}
 \providecommand{\calB}{\mathcal{B}}
 
 \providecommand{\calF}{\mathcal{F}}

\providecommand{\calK}{\mathcal{K}} 
 \providecommand{\calN}{\mathcal{N}}
 \providecommand{\calP}{\mathcal{P}}

\makeatletter
\def\bvapx@title{}
\AtBeginDocument{\providecommand{\phantomsection}{}}
\let\bvapx@maketitle\maketitle
\def\maketitle{\global\let\bvapx@title\@title\bvapx@maketitle}
\newcommand{\appendixtitle}{%
  \clearpage
  \suppressfloats[t]%
  \phantomsection
  \begin{center}
    {\LARGE\bfseries Appendices\ifx\bvapx@title\@empty\else\space to ``\bvapx@title''\fi\par}%
  \end{center}
  \vspace{1.5\baselineskip}%
}
\newcommand{\appendixcontents}{%
  \let\bvOrigAddcontentsline\addcontentsline
  \renewcommand{\addcontentsline}[3]{%
    \def\bvAclFile{##1}\def\bvAclToc{toc}%
    \ifx\bvAclFile\bvAclToc \bvOrigAddcontentsline{apc}{##2}{##3}%
    \else\bvOrigAddcontentsline{##1}{##2}{##3}\fi}%
  \setcounter{tocdepth}{2}%
  {\centering\large\bfseries Contents of the appendices\par}%
  \vspace{0.5\baselineskip}%
  \@starttoc{apc}%
}
\makeatother

\makeatletter
\newcommand{\appendixletterskipI}{\renewcommand{\thesection}{\ifnum\value{section}>8 \@Alph{\numexpr\value{section}+1\relax}\else\@Alph{\value{section}}\fi}}
\makeatother

\allowdisplaybreaks[2]
\AtBeginDocument{%
  \setlength{\abovedisplayskip}{6pt plus 2pt minus 2pt}%
  \setlength{\belowdisplayskip}{6pt plus 2pt minus 2pt}%
  \setlength{\abovedisplayshortskip}{3pt plus 1pt minus 1pt}%
  \setlength{\belowdisplayshortskip}{3pt plus 1pt minus 1pt}%
}

\providecommand{\akshay}[1]{}\renewcommand{\akshay}[1]{}%
\providecommand{\vac}[1]{}\renewcommand{\vac}[1]{}%
\begin{document}
\maketitle

\begin{abstract}
As an alternative to the standard geometric analyses, we give an exact, information-theoretic analysis of stochastic gradient descent (SGD) and its variants.
We show that a preconditioned SGD step is the posterior-mean update of a Gaussian Bayes model, and that its one-step regret splits into an intrinsic-time cost and a change in comparator information.
The split extends to an identity for the objective itself.
Convex convergence, strict-saddle-point escape, the link between flatness and generalization, the standard learning-rate schedules, adaptive optimizers, and the noisy, momentum, heavy-tailed, and gradient-free variants of SGD each correspond to a term or a special case of this identity.
We measure its terms on synthetic and real training runs.
On real networks it attributes the slack of classical convergence bounds to the terms dropped by their derivations and separates optimizers that reach the same training loss.
That separation follows the number and consistency of their steps.
Its relation to which of them generalizes better differs between networks.
For gradient-free SGD the identity determines how a curvature preconditioner should enter the update.
The sharpness-based generalization certificate it yields, with a data-independent isotropic prior, is vacuous at network scale unless the curvature spectrum is nearly flat across all parameters.
\end{abstract}

\section{Introduction}

Stochastic gradient descent (SGD) and its variants are used to train deep networks \citep{leon2018optimization} and large language models \citep{hong2023sophia}, to learn online \citep{Zinkevich2003}, and to sample from Bayesian posteriors \citep{welling2011bayesian}.
SGD is simultaneously a numerical method, a stochastic approximation algorithm, and a source of implicit regularization. Classical analysis usually begins with the identity
\[
\eta_t \ip{g_t}{\theta_t-u}
=\half\|\theta_t-u\|^2-\half\|\theta_{t+1}-u\|^2+\frac{\eta_t^2}{2}\|g_t\|^2
\]
then makes statements about excess risk through convexity, smoothness, or stochastic assumptions.
This usual geometric interpretation loses information every round.
We reinterpret and extend it as an exact Bayesian information balance.

\paragraph{Information ledger framework.} We specialize an adaptive Bayes identity to continuous Gaussian posteriors and local gradient scores. The specialization yields SGD and its preconditioned and adaptive variants as the posterior-mean recursion of a Gaussian adaptive Bayes update. It splits each round's excess composite loss into a centered cumulant cost and the change in a comparator's relative-entropy information. The cumulant cost is proportional to the round's intrinsic-time density $Q_t$, displayed below. We call this split the \emph{information ledger} of SGD. Classical convergence theorems are its projections: they follow after geometry-specific assumptions license dropping nonnegative Bregman curvature, terminal information, or favorable drift terms.

\paragraph{Contributions.} A Bayes update of a Gaussian posterior with a local linear score preserves the covariance and shifts the posterior mean by the SGD step. A proper Gaussian likelihood with process-noise reheating recovers this update in its weak-observation limit (\S\ref{sec:gab}). For fixed covariance $\Sigma$, learning rate $\eta_t$, and comparator $u$ the update satisfies the identity
\[
\eta_t\ip{g_t}{\theta_t-u}
=I_t(u)-I_{t+1}(u)+\eta_t^2Q_t,
\quad
I_t(u)=\half\|u-\theta_t\|_{\Sigma^{-1}}^2,
\quad
Q_t=\half g_t^\top\Sigma g_t
\]
and adding the Bregman remainder $B_f(u,\theta)=f(u)-f(\theta)-\ip{\nabla f(\theta)}{u-\theta}$ converts it into an objective-level identity in which convexity, strong convexity, and smoothness enter only when that term is simplified (\S\ref{sec:ledger}, \S\ref{sec:loss-identity}, \S\ref{sec:convex}).

The intrinsic clock $V_T=\sum_tQ_t$ organizes the adaptive, momentum, and variance-reduced optimizers (Table~\ref{tab:optimizer-ledger}). Square-root clock schedules give AdaGrad-Norm.
Choosing each step size so that the normalizing constant of the round's Bayes update reaches a prescribed level gives the Polyak step (\S\ref{sec:schedules}).
Gradient forecasts give momentum, Nesterov acceleration, and optimistic SGD.
Variance reduction is intrinsic-time reduction (\S\ref{sec:stochastic}).
Chebyshev step-size cycling and Schedule-Free SGD are read off the same identity.
Alternating the step that maximizes each round's gain in comparator information with a large escape step whenever the realized decrease per unit of intrinsic time stalls gives a cyclic descent-and-escape schedule (\S\ref{sec:schedules}).

Curvature enters the local free energy of \S\ref{sec:flatness} through a log-determinant penalty on the Hessian spectrum, in place of a raw Hessian eigenvalue. The same identity incorporates terminal flatness into a PAC-Bayes certificate (\S\ref{sec:flatness}, \S\ref{sec:gen}). With a data-independent isotropic prior the certificate is vacuous at network scale unless the curvature spectrum is nearly flat across all parameters. Near a strict saddle point $s$, the component of $I_t(s)$ along a direction of curvature $-\gamma<0$ is multiplied by $(1+\eta\gamma)^2$ per step and receives an additive noise injection (\S\ref{sec:saddle}). SAM, ASAM, and Fisher SAM add a path-averaged-Hessian directional curvature term (\S\ref{sec:flatness}, Table~\ref{tab:optimizer-ledger}).

The information ledger extends to the noisy and momentum variants of SGD (\S\ref{sec:noise-momentum}). Every finite-cumulant perturbation law adds one cumulant entry and leaves the other terms of the information ledger unchanged. Under heavy tails the pathwise ledger and the relative-entropy identity of the noise channel still hold. Phase-space momentum, as in SGHMC and underdamped Langevin dynamics, resolves into conservative transport plus a thermostat channel.

For zeroth-order SGD a random finite-difference probe drives the update with closed-form conditional moments. The sampling covariance acts as an implicit preconditioner exposing the dimension penalty (\S\ref{sec:zeroth}).
The second moment decides how a preconditioner should enter.
On a real objective with a wide curvature spread, a query-only curvature probe drawn into the sampling covariance recovers nearly all of the gain of an oracle that knows the Hessian diagonal.
Applied to the step, the same preconditioner makes no progress (\S\ref{sec:eval-practice}).

On real networks the information ledger attributes the slack of classical convergence bounds to the terms dropped by their derivations (\S\ref{sec:eval-slack-vs-bound}).
The intrinsic-time fraction divides the intrinsic-time cost by the sum of that cost and the change in comparator information, with the comparator at the run's own endpoint.
It separates optimizers that reach the same training loss, according to the number and consistency of their steps (\S\ref{sec:eval-regime-discrimination}).
Its relation to the generalization gap differs between networks (\S\ref{subsec:findings}).
Single terms with a closed form, such as the Bregman half-gain of \S\ref{sec:convex} and strict-saddle-point escape times, behave as the identity implies (\S\ref{sec:eval-closed-form}).
Sharpness tracks the generalization gap along the batch-size axis on two networks (\S\ref{sec:eval-practice}).
Across optimizers the bare Hessian statistics depend on coordinates, while function-space curvature does not (\S\ref{subsec:findings}).

\paragraph{Scope and roadmap.} The information ledger holds exactly for the local gradient score and the SGD mean recursion. The objective-level identities hold once the Bregman remainder is retained. Conventional rates still require assumptions, which enter visibly and locally where terms of the identity are replaced by envelopes.
\S\ref{sec:ledger-main} builds the information ledger, from which \S\ref{sec:readings} and \S\ref{sec:flatgen} read off the classical theory. \S\ref{sec:experiments} measures its terms on real training runs and on problems where a single term has a closed form. The appendices hold full derivations, experimental protocols, and deferred proofs.

\section{Related work}\label{sec:related}

As the most widely used training algorithm in modern machine learning, SGD has a correspondingly vast literature.
This section gives a representative overview of the parts closest to the present analysis.

Classical and modern stochastic approximation analyzes SGD through expected gradient alignment, variance, and averaging \citep{herbert1951stochastic,arkadi2009robust,leon2018optimization}, in the online-gradient-descent vocabulary of \cite{Zinkevich2003,hazan2016introduction}.
In the information ledger the gradient alignment equals the change in comparator information plus the intrinsic-time cost (\S\ref{sec:ledger}).
Gradient variance enters through the stochastic-gradient intrinsic time (\S\ref{sec:stochastic}).
Averaging enters through the weighted average identity (\S\ref{sec:convex}).
The high-probability analysis of the last iterate under non-smoothness is sharpened in \cite{harvey2019tight}.
For unconstrained smooth convex objectives with sub-Gaussian gradient noise, \cite{li2026exact} characterize the time-uniform rates attainable by the last iterate of SGD under deterministic step-size schedules fixed before the run.
A boundary of order $(\log T)^{p}/\sqrt{T}$ that holds at all iterations simultaneously with high probability, uniformly over that class, is attainable if and only if $p > 1/2$.
Anytime-valid confidence sequences make time-uniform guarantees observable: \cite{aolaritei2025stopping,aolaritei2026trajectory} certify the suboptimality and, under strong convexity, the last-iterate distance of projected SGD uniformly over time, stitching time-uniform envelopes \citep{howard2020b} over an accumulated conditional-variance clock that they also call intrinsic time.
Their certified envelopes relax the one-step identity from which \S\ref{sec:convex} derives the contraction.
The information ledger names the slack they retain: the curvature gain beyond the strong-convexity quadratic, the terminal information, and the certified envelope standing in for the true distance.

Existing derivations of optimizers from Bayesian objects approximate the posterior, recovering SGD and AdamW as limits of Bayesian filtering on the gradients \citep{aitchison2020bayesianfiltering} or a family from ridge regression to RMSProp as instances of one Bayesian learning rule \citep{khan2023bayesian}.
They are exact only in the limit of that approximation.
Others add noise so that iterates approach posterior samples \citep{welling2011bayesian}, tune constant-step SGD so that its stationary law approximates a posterior \citep{stephan2017stochastic}, or approximate the realized update by diffusion limits and modified equations \citep{qianxiao2019stochastic,david2020implicit}.
Safe Bayes \citep{grunwald2017inconsistency} tempers the likelihood with an exponent $\eta\in(0,1]$ to repair consistency under misspecification.
The inverse temperature $\eta_t$ of \S\ref{sec:gab} plays that role.
An exact correspondence does precede ours for one member of the family: an extended Kalman filter estimating a static parameter under exponential-family observation noise is algebraically identical, step for step, to online stochastic natural gradient on the log-likelihood \cite{ollivier2018kalman}.
That equivalence is specific to the Fisher-preconditioned update and couples the step to the filter's adaptively contracting covariance.
The present identity holds the covariance fixed, applies to any preconditioned first-order step, and is an equality on the realized trajectory with no approximation step, so its terms are measured.
SGD's stationary distribution admits a variational reading \cite{chaudhari2018stochastic}: the steady-state law minimizes a free energy whose potential coincides with the training loss only when the mini-batch gradient noise is isotropic.
That statement concerns the stationary distribution, in continuous time; the identity here is per step, on the realized discrete path.

The terms of the information ledger meet several further literatures.
Mutual-information and conditional-mutual-information bounds control the generalization of adaptive algorithms \citep{russo2016controlling,raginsky2017informationtheoretic,thomas2020reasoning,gergely2021informationtheoretic}, and \cite{hellstrom2025generalizationpacbayes} survey the resulting body of information-theoretic and PAC-Bayesian generalization bounds.
The closest of these to the present accounting \cite{gergely2021informationtheoretic} advances a Gaussian-perturbed copy of the SGD path using gradients taken at the unperturbed iterates.
That recursion is an instance of the noisy update accounted for in Appendix~\ref{sec:addendum-langevin-momentum}, whose pathwise identity is stated for an arbitrary realized gradient and perturbation.
The gradient need not be taken at the iterate advanced by the recursion.
The mini-batch term of the expected intrinsic-time density is half the local gradient variance of that analysis, taken in the metric $\Sigma$ and referred to the conditional gradient mean, identified in its discussion as the more natural choice.
This paper is complementary in scope: it supplies the optimization-side information ledger.
Its terminal information $I_{T+1}(u)$ is the relative entropy between two Gaussians of covariance $\Sigma$, centered at the final iterate and at the comparator, which is the complexity of a PAC-Bayes bound whose prior is the second.
The sharpness literature runs from \cite{sepp1997flat} through the large-batch link \citep{nitish2017largebatch} and the edge of stability of full-batch gradient descent \citep{cohen2021gradient}, with \cite{laurent2017sharp} showing that many parameter-space flatness notions are not reparameterization-invariant.
PAC-Bayes bounds make the flatness link explicit through Poincar\'e and log-Sobolev inequalities without explicit dependence on predictor dimension \citep{haddouche2025flatminima}.
Local-entropy objectives bias optimization toward wide valleys \citep{pratik2017entropysgd}.
In stochastic convex optimization, flat empirical minima can incur $\Omega(1)$ population risk while sharp minima of the same problems generalize optimally \citep{schliserman2025flat}.
Sharpness-aware minimization and its metric variants \citep{pierre2021sharpness,kwon2021asam,kim2022fishersam} sit inside the identity as choices of ascent direction (Appendix~\ref{app:flatness-full}).
The informationally natural sharpness is the log-determinant of the local Hessian (\S\ref{sec:flatness}).
Gradient methods avoid or escape strict saddle points under perturbation and noise \citep{jason2016gradient,rong2015escaping,chi2017how,cong2019sharp,hadi2018escaping}.
The contribution here is the local saddle-point mechanism expressed as information amplification plus noise injection (\S\ref{sec:saddle}).
No sharper escape rate is claimed.

RMSProp, Adam, momentum and acceleration, variance reduction, natural gradient, K-FAC, Shampoo, gradient clipping, parameter-free SGD, and optimistic online learning are each a specific reading of the same information ledger (Table~\ref{tab:optimizer-ledger}, \S\ref{sec:stochastic}, and Appendices~\ref{app:schedules-full} and~\ref{app:identities}).
Cyclical, super-convergence, warm-restart, and warmup-stable-decay schedules \citep{smith2017cyclical,smith2019superconvergence,loshchilov2017sgdr,wen2024wsd} redistribute when intrinsic-time costs are paid and when comparator information is transported.
At large-language-model scale, realized loss curves under constant schedules with linear cooldown track a suboptimality bound from nonsmooth convex theory closely enough that optimal learning rates transfer across schedules \citep{schaipp2025agreement}.
Optimizers introduced since 2023 \citep{xiangning2023symbolic,hong2023sophia,keller2024muon,nikhil2024soap,aaron2024road} choose different summaries of the realized gradient stream, each a reading of the information ledger through its choice of preconditioner $\Sigma_t$, forecast $m_t$, and schedule $\eta_t$ (Table~\ref{tab:optimizer-ledger}, \S\ref{sec:schedulefree-ledger}).
Shampoo's second-moment estimate has itself been recast as covariance estimation under relative-entropy minimization \citep{lin2025shampoo}.
Choosing the preconditioner to minimize expected intrinsic time reduces to that estimation problem once the entropy of the prior is subtracted (\S\ref{sec:precon}).

\section{Preliminaries}\label{sec:prelims}

Let $\Theta=\R^d$. A stochastic loss on round $t$ is denoted by $f_t:\Theta\to\R$, with gradient $g_t=\nabla f_t(\theta_t)$.
For stochastic optimization, $F(\theta)=\E_z f(\theta;z)$, $\calF_t$ denotes the history before drawing the stochastic gradient at $\theta_t$, and $\E[g_t\mid \calF_t]=\nabla F(\theta_t)$, with conditional gradient-noise covariance $C_t=\E[(g_t-\nabla F(\theta_t))(g_t-\nabla F(\theta_t))^\top\mid\calF_t]$.

For $\Sigma\succ0$, define the Gaussian information distance $D_\Sigma(u,\theta)=\half\|u-\theta\|_{\Sigma^{-1}}^2$.
If $q_{\theta,\Sigma}=\calN(\theta,\Sigma)$ and $\rho_{u,\Sigma}=\calN(u,\Sigma)$, then the relative entropy between them, written $\KL(\cdot\|\cdot)$ throughout, is $\KL(\rho_{u,\Sigma}\|q_{\theta,\Sigma})=D_\Sigma(u,\theta)$.
For a differentiable function $f$, the first-order Bregman remainder is $B_f(u,\theta)=f(u)-f(\theta)-\ip{\nabla f(\theta)}{u-\theta}$, so that
\begin{equation}\label{eq:bregman-exact}
\ip{\nabla f(\theta)}{\theta-u}=f(\theta)-f(u)+B_f(u,\theta)
\end{equation}
When $f$ is convex, $B_f(u,\theta)\ge0$. When $f$ is nonconvex, $B_f$ is signed and records the curvature.

The preconditioned SGD recursion is
\begin{equation}\label{eq:sgd}
\theta_{t+1}=\theta_t-\eta_t\Sigma_t g_t
\end{equation}
Ordinary SGD corresponds to $\Sigma_t=\sigma^2 I$ (often $\sigma^2=1$). A scalar covariance $\Sigma=\sigma^2I$ merely changes the unit in which the inverse temperature $\eta_t$ is measured; we write $\alpha_t=\eta_t\sigma^2$ for the resulting effective step size.

Throughout the paper the \emph{composite loss} fed to the Bayesian machinery is the affine $c_t(\theta):=\ip{g_t}{\theta}$, a linearization of $f_t$ at $\theta_t$.
For \emph{convex} round losses, the subgradient inequality gives $f_t(\theta_t)-f_t(\theta_*)\le\ip{g_t}{\theta_t-\theta_*}$, so any identity for $c_t$ yields a regret bound for $f_t$ at zero inflation.
For \emph{nonconvex} $f_t$, the subgradient inequality fails, but the one-step identity for $c_t$ still holds as a trajectory-level Bayesian bookkeeping identity (\S\ref{sec:saddle}).

\section{The SGD information ledger}\label{sec:ledger-main}

\subsection{SGD as Gaussian adaptive Bayes}\label{sec:gab}

The update \eqref{eq:sgd} is a Bayesian posterior-mean recursion. 
It follows from a Gaussian prior over parameters and the local linear gradient score.

Fix $\Sigma\succ0$ and let $q_t=\calN(\theta_t,\Sigma)$.
After observing $g_t$, define the local score $c_t(\theta)=\ip{g_t}{\theta}$. The adaptive Bayes update with inverse temperature $\eta_t>0$ is
\begin{equation}\label{eq:linear-tilt}
q_{t+1}(d\theta)=\frac{\exp(-\eta_t c_t(\theta))q_t(d\theta)}{Z_t},
\qquad
Z_t=\int \exp(-\eta_t c_t(\theta))q_t(d\theta)
\end{equation}

Reweighting a Gaussian by the exponential of a linear function gives back a Gaussian with the same covariance and a mean shifted by the SGD step $-\eta_t\Sigma g_t$. So the SGD iterate is the posterior mean of an exponential-family update with the local gradient score as the likelihood factor; the same calculation also writes down the log-partition function $\log Z_t$ as an explicit closed-form quadratic in $g_t$.

\begin{proposition}[The Gaussian linear-score update is SGD]\label{prop:gaussian-tilt}
If $q_t=\calN(\theta_t,\Sigma)$ and $c_t(\theta)=\ip{g_t}{\theta}$, then the update \eqref{eq:linear-tilt} satisfies
\[
q_{t+1}=\calN(\theta_t-\eta_t\Sigma g_t,\Sigma)
\]
Moreover,
\[
\log Z_t=-\eta_t\ip{g_t}{\theta_t}+\frac{\eta_t^2}{2}g_t^\top\Sigma g_t
\]
Thus the posterior mean follows the preconditioned SGD step $\theta_{t+1}=\theta_t-\eta_t\Sigma g_t$.
\end{proposition}

The proof is in Appendix~\ref{app:proofs-deferred}.

Define the mix loss
\[
\ell^{\mathrm{mix}}_t=-\eta_t^{-1}\log Z_t
=\ip{g_t}{\theta_t}-\eta_tQ_t,
\qquad
Q_t=\half g_t^\top\Sigma g_t
\]
The mean score under $q_t$ is $\E_{q_t}c_t=\ip{g_t}{\theta_t}$, so the centered cumulant cost is $\delta_t=\E_{q_t}c_t-\ell^{\mathrm{mix}}_t=\eta_t Q_t$.
This increment, the finite-temperature Gaussian cumulant incurred by the update, is a quadratic form in the observed gradient.

The Gaussian prior $\pi=\calN(0,\Sigma)$ plays distinct roles at once: it is a \emph{regularizer} ($\log\pi$ supplies a ridge penalty of strength $\Sigma^{-1}/2$), a \emph{decoder for comparators} ($\KL(\calN(\theta_*,\Sigma)\|\pi)=\half\theta_*^\top\Sigma^{-1}\theta_*$), and a \emph{Bayesian error bar} maintained at every step ($q_t$ always has covariance $\Sigma$ under linearization). Choosing $\Sigma$ adapts all three simultaneously: a non-isotropic prior yields preconditioned SGD; a Fisher-shaped prior yields natural gradient (\S\ref{sec:precon}).

Bayes updating admits two geometries that differ under variable $\eta_t$. The \emph{local} update \eqref{eq:linear-tilt} starts from $q_t$ each round and gives SGD itself. The \emph{prior-retempered} update recomputes the posterior from the prior $\pi$ at the current temperature: $p_t^{\mathrm{retemp}}=\calN(-\eta_t\Sigma G_{t-1},\Sigma)$ with $G_{t-1}=\sum_{s<t}g_s$. At fixed step size the two coincide when the local run starts at the prior, $q_1=\pi$. The SGD recursion as it is conventionally implemented in deep-learning frameworks (PyTorch, JAX, TensorFlow) is the local branch.

The linear score $\exp(-\eta_t\ip{g_t}{\theta})$ is also the weak-observation limit of a proper Gaussian likelihood. A finite Gaussian observation model whose covariance-contracting update is reheated by a process-noise step reproduces the SGD mean recursion term for term (Proposition~\ref{prop:proper-gaussian-likelihood}, Appendix~\ref{app:gab-likelihood}).
\subsection{The one-step identity}\label{sec:ledger}

Let $\rho_u=\calN(u,\Sigma)$ be a Gaussian comparator with the same covariance as $q_t$, and define $I_t(u)=\KL(\rho_u\|q_t)=D_\Sigma(u,\theta_t)=\half\|u-\theta_t\|_{\Sigma^{-1}}^2$.
Against any Gaussian comparator $\rho_u$, the inner-product regret $\eta_t\ip{g_t}{\theta_t-u}$ at round $t$ equals an immediate intrinsic-time cost $\eta_t^2 Q_t$ plus the change in the comparator's Gaussian relative-entropy information $I_t(u)-I_{t+1}(u)$. The identity involves no inequalities and no smoothness or convexity assumptions. It is the Gaussian specialization of the adaptive-Bayes one-step balance \cite{balsubramani2026adaptive}, which states the same accounting for reweighting updates over a general prior.

\begin{theorem}[Exact Gaussian-Bayes SGD identity]\label{thm:sgd-ledger}
For the update $\theta_{t+1}=\theta_t-\eta_t\Sigma g_t$, every comparator $u\in\R^d$ satisfies
\begin{equation}\label{eq:one-step-ledger}
\ip{g_t}{\theta_t-u}
=\eta_tQ_t+\frac{I_t(u)-I_{t+1}(u)}{\eta_t},
\qquad
Q_t=\half g_t^\top\Sigma g_t
\end{equation}
Equivalently,
\begin{equation}\label{eq:weighted-ledger}
\eta_t\ip{g_t}{\theta_t-u}
=\underbrace{I_t(u)-I_{t+1}(u)}_{\text{comparator-information transport}}
+\underbrace{\eta_t^2Q_t}_{\text{intrinsic-time cost}}
\end{equation}
Consequently,
\begin{equation}\label{eq:sum-weighted-ledger}
\sum_{t=1}^T \eta_t\ip{g_t}{\theta_t-u}
=I_1(u)-I_{T+1}(u)+\sum_{t=1}^T\eta_t^2Q_t
\end{equation}
\end{theorem}

The proof is in Appendix~\ref{app:proofs-deferred}.

The identity recovers the classical one-line proof as a special case. When $\Sigma=\sigma^2I$ and $\alpha=\eta\sigma^2$, identity \eqref{eq:weighted-ledger} becomes the classical SGD contraction
\[
\ip{g_t}{\theta_t-\theta_*}=\tfrac{\alpha_t}{2}\|g_t\|^2+\tfrac{1}{2\alpha_t}\bigl(\|\theta_*-\theta_t\|^2-\|\theta_*-\theta_{t+1}\|^2\bigr)
\]
usually obtained by expanding $\|\theta_{t+1}-\theta_*\|^2=\|\theta_t-\alpha_tg_t-\theta_*\|^2$.

The identity is the standard Euclidean SGD algebra when $\Sigma=I$. The telescoping structure it exposes is the Euclidean instance of the Bregman three-point identity \cite{chen1993convergence}, which likewise holds as an equality and underlies the usual mirror-descent regret decomposition. Telescoped potential arguments organize classical first-order convergence proofs more broadly \cite{bansal2019potential}. The present reading adds the meaning of its terms; the equality itself is classical.

For unweighted linearized regret, summing \eqref{eq:one-step-ledger} by parts gives
\begin{equation}\label{eq:unweighted-abel}
\sum_{t=1}^T\ip{g_t}{\theta_t-u}
=\sum_{t=1}^T\eta_tQ_t+\frac{I_1(u)}{\eta_1}-\frac{I_{T+1}(u)}{\eta_T}
+\sum_{t=2}^T I_t(u)\left(\frac1{\eta_t}-\frac1{\eta_{t-1}}\right)
\end{equation}
This is the local-update bookkeeping of the temperature $1/\eta_t$. If the learning rate decreases, the temperature variation is nonnegative. If it increases, the identity records the negative variation.

\subsection{Covariance drift for adaptive preconditioning}\label{sec:cov-drift}

Many practical methods change geometry. Let $q_t=\calN(\theta_t,\Sigma_t)$, perform the linear-score update $q_t^+=\calN(\theta_t-\eta_t\Sigma_tg_t,\Sigma_t)$, and then set $q_{t+1}=\calN(\theta_{t+1},\Sigma_{t+1})$ with $\theta_{t+1}=\theta_t-\eta_t\Sigma_tg_t$.
Let $\rho=\calN(u,S)$ be any Gaussian comparator and define $K_t(\rho)=\KL(\rho\|q_t)$ and $\Xi_{t+1}(\rho)=\KL(\rho\|q_{t+1})-\KL(\rho\|q_t^+)$.
Then
\begin{equation}\label{eq:cov-drift}
\eta_t\left(\E_{q_t}c_t-\E_\rho c_t\right)
=K_t(\rho)-K_{t+1}(\rho)+\Xi_{t+1}(\rho)+\eta_t^2Q_t,
\qquad
Q_t=\half g_t^\top\Sigma_tg_t
\end{equation}
For Gaussian $\rho=\calN(u,S)$, the drift is explicit:
\begin{align}\label{eq:xi-gaussian}
\Xi_{t+1}(\rho)=\half\Big[&\tr\{(\Sigma_{t+1}^{-1}-\Sigma_t^{-1})S\}
+(u-\theta_{t+1})^\top(\Sigma_{t+1}^{-1}-\Sigma_t^{-1})(u-\theta_{t+1}) \\
&+\log\frac{\det\Sigma_{t+1}}{\det\Sigma_t}\Big]\nonumber
\end{align}
Thus adaptive preconditioners stay within the accounting. They add an information term measuring the cost or benefit of changing covariance geometry.

\subsection{Objective-level identities}\label{sec:loss-identity}

The information ledger so far concerns the linearized score $g_t^\top\theta$. To express actual objective loss, insert the Bregman identity \eqref{eq:bregman-exact}.

The linearized identity lifts to the realized excess loss by accounting for the gap between $f_t$ and its first-order Taylor expansion at $\theta_t$. The Bregman remainder $B_{f_t}(u,\theta_t)$ stays in the information ledger as a signed entry. For convex $f_t$ it is nonnegative and yields the curvature gain that classical regret bounds drop. For nonconvex $f_t$ it records the signed curvature that \S\ref{sec:saddle} uses to analyze saddle points.

\begin{theorem}[Exact sample-path loss identity]\label{thm:loss-identity}
Assume $g_t=\nabla f_t(\theta_t)$ and $\theta_{t+1}=\theta_t-\eta_t\Sigma g_t$. For every comparator $u$,
\begin{align}
\sum_{t=1}^T\eta_t\left[f_t(\theta_t)-f_t(u)+B_{f_t}(u,\theta_t)\right]&=I_1(u)-I_{T+1}(u)+\sum_{t=1}^T\eta_t^2Q_t\label{eq:loss-path}\\
\sum_{t=1}^T\eta_t\left[f_t(\theta_t)-f_t(u)\right]&=I_1(u)-I_{T+1}(u)+\sum_{t=1}^T\eta_t^2Q_t-\sum_{t=1}^T\eta_t B_{f_t}(u,\theta_t)\label{eq:excess-exact}
\end{align}
\end{theorem}

The proof is in Appendix~\ref{app:proofs-deferred}.

The two forms are the same identity, the second isolating the realized excess loss on the left.

Under unbiased mini-batch noise, with the step sizes and the starting point fixed, the pathwise identity extends to expectation, splitting the intrinsic-time cost into a signal piece $\half\nabla F^\top\Sigma\nabla F$ and a noise piece $\half\tr(\Sigma C_t)$.

\begin{corollary}[Expected stochastic identity]\label{cor:stochastic-identity}
Let $\theta_t$ be $\calF_t$-measurable, $g_t$ be the stochastic gradient drawn at round $t$ given $\calF_t$, and assume $\E[g_t\mid\calF_t]=\nabla F(\theta_t)$ with conditional covariance $C_t=\Cov(g_t\mid\calF_t)$. Fix $\Sigma$, the step sizes $\eta_t$ and the starting point $\theta_1$ deterministically. Then for every $u\in\R^d$,
\begin{equation}\label{eq:expected-stochastic}
\E\sum_{t=1}^T\eta_t\left[F(\theta_t)-F(u)+B_F(u,\theta_t)\right]=I_1(u)-\E I_{T+1}(u)+\sum_{t=1}^T\frac{\eta_t^2}{2}\E\left[\nabla F(\theta_t)^\top\Sigma\nabla F(\theta_t)+\tr(\Sigma C_t)\right]
\end{equation}
\end{corollary}

The proof is in Appendix~\ref{app:proofs-deferred}.

For mini-batches of size $b$ drawn independently with replacement, $C_t$ is the single-example gradient covariance divided by $b$; drawn without replacement from $N$ examples, it is that covariance times $(N-b)/(b(N-1))$. The realized inflation of intrinsic time follows this factor within $3.1\%$ at every batch size tested (Figure~\ref{fig:single-term}(a)). Larger batches reduce the stochastic component of intrinsic time.

\subsection{Intrinsic time}\label{sec:intrinsic-time}

The identification of the intrinsic-time density $Q_t=\half g_t^\top\Sigma_tg_t$ follows from the Gaussian/linearized structure.
The intrinsic time depends on the instance only through the realized preconditioned-gradient-norm sequence $g_t^\top\Sigma_tg_t$, while the comparator-information transport $I_1(u)-I_{T+1}(u)$ also depends on the endpoint $\theta_{T+1}$.
Curvature, noise, conditioning, and convexity act on the intrinsic time through the trajectory $\theta_t$, and hence through the sequence $g_t$, by setting how fast $g_t^\top\Sigma_tg_t$ decays.

Hence the heuristic successes of adaptive optimizers (AdaGrad's $1/\sqrt{\sum\|g\|^2}$ step size, Adam's second-moment tracker, RMSProp's EMA of squared gradients) can all be read as tracking the same information quantity. The learner's excess regret tracks the clock $V_T=\sum_tQ_t$ along the realized trajectory.
As an index of loss decay, $V_T$ gives the same forecasts as the iteration count wherever it grows as a power of $T$.
A power law in $V_T$ is then a power law in $T$, which fits the same curve and forecasts the same step at which a target loss is reached.
Gradient norms that decay or stay level as a power of $t$ make $V_T$ such a power law.
Along all $150$ runs of a real-data curriculum sweep $\log V_t$ is affine in $\log t$ to within $0.4\%$ of its variance.
Along real-data training prefixes the logarithm of the step-weighted clock $\sum_{s\le t}\eta_s^2Q_s$ is affine in $\log t$ to within $4\%$, so a clock separates from the step counter only where it bends against $t$ (Appendix~\ref{sec:eval-vclock-fits}).

To define the intrinsic-time fraction, set the comparator at the run's own endpoint, $u=\theta_{T+1}$, so that $I_{T+1}(u)=0$, and write $\Delta_t=\theta_t-\theta_{t+1}=\eta_t\Sigma g_t$ for the realized step, so that $\half\|\Delta_t\|_{\Sigma^{-1}}^2=\eta_t^2Q_t$ and, in the Euclidean metric $\Sigma=I$, $I_1(u)=\half\|\theta_{T+1}-\theta_1\|^2$.
With the comparator at the endpoint the transport telescopes to half the squared net displacement, so for any update sequence the intrinsic-time fraction, the cumulative cost over the cost plus the transport, is
\begin{equation}\label{eq:path-fraction}
\frac{\sum_{t\le T}\half\|\Delta_t\|^2}{\sum_{t\le T}\half\|\Delta_t\|^2+\half\|\theta_{T+1}-\theta_1\|^2}=\frac{1}{1+R},\qquad R=\frac{\|\theta_{T+1}-\theta_1\|^2}{\sum_{t\le T}\|\Delta_t\|^2}
\end{equation}
The intrinsic-time fraction is a statistic of the path's shape.
By Cauchy--Schwarz $R/T\le1$, with equality when every step is the same vector, so $R/T$ measures how consistently the steps point along the net displacement, and $R$ grows with the number of steps at a fixed consistency.
Optimizers that reach the same training loss can do so along paths with different step counts and different consistencies, so the fraction separates trajectories that the achieved losses do not (Figure~\ref{fig:eval-regime-discrimination}(a), \S\ref{sec:eval-regime-discrimination}).
Across a learning-rate sweep it rises monotonically through the test-error optimum.
At a double-descent peak it has no feature (Figure~\ref{fig:eval-regime-discrimination}(b,c)).

\begin{figure}[t]
\centering
\includegraphics[width=\textwidth]{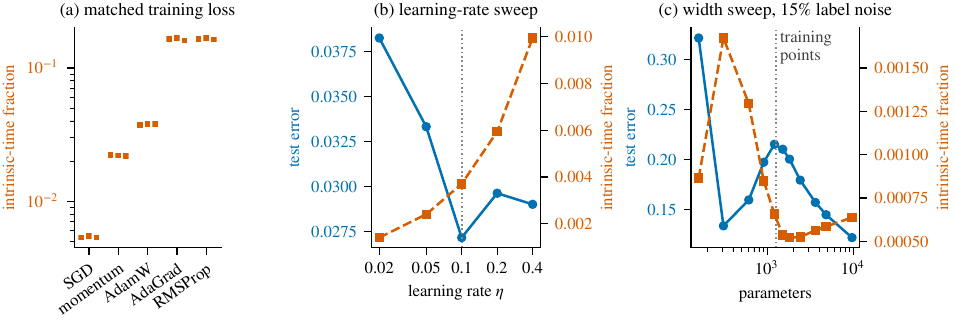}
\caption{\textbf{At a matched training loss the information ledger separates five optimizers by $31\times$ where their losses differ by less than $1.6\times$.}
In the learning-rate and width sweeps the intrinsic-time fraction follows the shape of the path.
(a) Five optimizers stopped at a common training loss, with a $1.04\times$ spread in training loss and a $1.53\times$ spread in test error. Each marker is one seed.
(b) A learning-rate sweep at a fixed budget: the test error (blue) is smallest at $\eta=0.1$, while the fraction (vermillion) rises with $\eta$ as larger steps make the path less consistent.
(c) A width sweep under $15\%$ label noise: the test error's double-descent peak sits at the interpolation threshold (dotted, $1257$ training points), where the fraction has no feature.}
\label{fig:eval-regime-discrimination}
\end{figure}

The intrinsic clock $V_T$ has a relaxed counterpart $W_T$, built from the variances of the scores:
\[
V_T(c)=\sum_{s\le T} Q_s,\qquad W_T(c)=\sum_{s\le T}\Var_{\theta\sim q_s}(c_s(\theta))
\]
The distinction collapses for SGD: the Gaussian posterior makes $Q_s=\half\Var_{\theta\sim q_s}(c_s(\theta))$ at every temperature. Here $\Var_{\theta\sim q_s}(\ip{g_s}{\theta})=g_s^\top\Sigma g_s$, so $V_T\equiv\half W_T$ and both reduce to the cumulative (preconditioned) squared-gradient quantity. For general exponential-family posteriors the two clocks diverge.

The clock $V_T$ can grow linearly, grow sub-linearly, or saturate.
\begin{itemize}[leftmargin=1.5em,itemsep=2pt]
\item \emph{Linear growth $V_T\asymp T$:} gradients remain bounded away from zero (adversarial or iid around a non-degenerate objective).
\item \emph{Sub-linear growth $V_T\asymp T^\beta$, $\beta\in(0,1)$:} gradients decay toward a critical region.
\item \emph{Saturation $V_T\to V_\infty<\infty$:} the optimizer converges fast enough that $\sum\|g_t\|^2$ converges, as with deterministic gradients on a strongly convex problem, where $\|g_t\|^2$ decays geometrically.
\end{itemize}
Regret bounds inherit the regime through $V_T$: the adversarial worst case $V_T\asymp T$ gives $\sqrt T$-type regret.
Under stochastic \emph{luckiness} (condition \eqref{eq:luckiness} below) fixed-rate SGD attains constant cumulative regret with a noiseless oracle, while with gradient noise $O(\log T)$ regret needs the decaying rate $\eta_t=1/(\mu t)$ (Appendix~\ref{app:identities}).

\section{Classical results as readings of the information ledger}\label{sec:readings}

\subsection{Convex convergence}\label{sec:convex}

When $F$ is convex, $B_F(u,\theta_t)\ge0$. Equation \eqref{eq:expected-stochastic} gives the weighted average identity
\begin{equation}\label{eq:convex-exact-average}
\frac{\E\sum_{t=1}^T\eta_t[F(\theta_t)-F(u)]}{S_T}
=\frac{I_1(u)-\E I_{T+1}(u)+\sum_{t=1}^T\eta_t^2\E Q_t-\E\sum_{t=1}^T\eta_tB_F(u,\theta_t)}{S_T},
\quad S_T=\sum_{t=1}^T\eta_t
\end{equation}
Only after dropping the terminal information and the nonnegative curvature gain do we obtain the familiar bound
\begin{equation}\label{eq:convex-bound}
\E[F(\bar\theta_T)-F(u)]
\le
\frac{I_1(u)+\sum_{t=1}^T\eta_t^2\E Q_t}{S_T},
\qquad
\bar\theta_T=\frac{\sum_t\eta_t\theta_t}{S_T}
\end{equation}
where Jensen's inequality is the only additional step. This is Nemirovski--Yudin's classical regret bound and Zinkevich's bound for online gradient descent: in our framework it is the identity minus two nonnegative remainders, the terminal relative entropy and the curvature gain.

Suppose $F$ is $\mu$-strongly convex and $u=\theta_\star\in\arg\min F$. With $\Sigma=I$,
\[
B_F(\theta_\star,\theta_t)\ge \frac{\mu}{2}\|\theta_t-\theta_\star\|^2=\mu I_t(\theta_\star)
\]
Applying \eqref{eq:weighted-ledger} with $u=\theta_\star$ and rearranging to isolate $I_{t+1}(\theta_\star)$ gives $I_{t+1}(\theta_\star)=I_t(\theta_\star)-\eta_t\ip{g_t}{\theta_t-\theta_\star}+\eta_t^2Q_t$.
Since $F$ is convex with minimizer $\theta_\star$, $\ip{\nabla F(\theta_t)}{\theta_t-\theta_\star}\ge F(\theta_t)-F(\theta_\star)\ge0$.  Taking conditional expectation given $\calF_t$, using $\E[\ip{g_t}{\theta_t-\theta_\star}\mid\calF_t]=\ip{\nabla F(\theta_t)}{\theta_t-\theta_\star}$ ($\theta_t$ is $\calF_t$-measurable), and inserting the strong-convexity lower bound $\ip{\nabla F(\theta_t)}{\theta_t-\theta_\star}\ge F(\theta_t)-F(\theta_\star)+B_F(\theta_\star,\theta_t)\ge\mu I_t(\theta_\star)$ (where the last step uses $B_F(\theta_\star,\theta_t)\ge\mu\|\theta_t-\theta_\star\|^2/2=\mu I_t(\theta_\star)$ and $F(\theta_t)-F(\theta_\star)\ge0$) yields
\begin{equation}\label{eq:strong-recursion}
\E[I_{t+1}(\theta_\star)\mid\calF_t]
\le (1-\mu\eta_t)I_t(\theta_\star)+\eta_t^2\E[Q_t\mid\calF_t]
\end{equation}
This is the standard stochastic-approximation recursion written as information contraction (factor $1-\mu\eta_t$) plus intrinsic-time injection. For the canonical Robbins--Monro schedule $\eta_t=1/(\mu t)$ with $\E[Q_t\mid\calF_t]\le Q_{\max}$ it telescopes to $\E I_{T+1}(\theta_\star)\le Q_{\max}\sum_{t\le T}t^{-1}/(\mu^2T)=O(\log T/(\mu^2T))$.
Keeping $F(\theta_t)-F(\theta_\star)\ge\mu I_t(\theta_\star)$ as well makes the factor $1-2\mu\eta_t$ and gives the classical $O(1/(\mu^2T))$ distance rate, without the $\log T$ \citep{arkadi2009robust}.
For $F$ itself the classical last-iterate rate is $O(\log T/(\mu T))$. Tail-averaging or the linearly weighted average removes that rate's $\log T$ factor. The $\eta_t$-weighted average decays only as $1/\log T$, since its weights concentrate on the earliest iterates.

The stochastic-luckiness condition of \cite{balsubramani2026adaptive}, a comparator-centered low-noise condition, specializes to SGD as follows. For i.i.d. stochastic gradients,
\begin{equation}\label{eq:luckiness}
\E[\ip{g_t}{\theta-\theta_*}^2]\le\varkappa\E\ip{\bar g}{\theta-\theta_*},\qquad\bar g=\E[g_t]
\end{equation}
For an $L$-smooth, $\mu$-strongly convex $F$ with $\bar g=\nabla F(\theta_t)$, iterates within distance $D$ of $\theta_*$ and gradient noise of variance $\sigma^2$ along $\theta-\theta_*$, condition \eqref{eq:luckiness} holds with finite $\varkappa=O(LD^2+\sigma^2/\mu)$.
The regional analysis behind the constant, the failure of the condition on an unbounded domain, and the noise floor of fixed-rate SGD are in Appendix~\ref{app:identities}.

The weighted average identity is more informative than the bound it relaxes to. The right side of \eqref{eq:convex-exact-average} shows the quantities normally hidden by analysis: terminal information $I_{T+1}$, curvature gain $\sum_t\eta_tB_F$, and the realized intrinsic-time cost $\sum_t\eta_t^2Q_t$. Two SGD runs can have identical worst-case bounds but very different decompositions, so the decomposition is a diagnostic object in its own right.

On quadratics, against the minimizer, the curvature gain is half the linearized regret, a relation we call the Bregman half-gain. For a quadratic objective $F(\theta)=\half(\theta-\theta_*)^\top A(\theta-\theta_*)$ with $A\succ0$, the first-order Bregman remainder is $B_F(\theta_*,\theta_t)=\half\|\theta_*-\theta_t\|_A^2$, which equals the excess loss $F(\theta_t)-F^*$. Consequently $\sum_t\eta_t[F(\theta_t)-F^*]=\sum_t\eta_t B_F(\theta_*,\theta_t)$ on the LHS of \eqref{eq:loss-path}, so identity \eqref{eq:loss-path} with noiseless gradients makes the realized regret $\sum_t\eta_t[F(\theta_t)-F^*]$ exactly half of $I_1(\theta_*)-I_{T+1}(\theta_*)+\sum_t\eta_t^2Q_t$. Under unbiased gradient noise a zero-mean martingale term is added. Against the bound $I_1(\theta_*)+\sum_t\eta_t^2Q_t$ the ratio is therefore one half less half the terminal-information share, up to that martingale term. On noisy quadratics at stable step sizes it lies between $0.492$ and $0.500$ (Appendix~\ref{app:sanity}). The curvature gain is the half that classical worst-case analyses hide.

\begin{proposition}[Exact tracking identity for a shifting comparator]\label{prop:tracking-ledger}
For any comparator sequence $\{u_t\}_{t=1}^T$ and $\theta_{t+1}=\theta_t-\eta_t\Sigma g_t$,
\begin{equation}\label{eq:tracking-ledger}
\sum_{t=1}^T\eta_t\ip{g_t}{\theta_t-u_t}
=I_1(u_1)-I_{T+1}(u_T)+\sum_{t=1}^T\eta_t^2Q_t+D_T,
\qquad
D_T=\sum_{t=1}^{T-1}\bigl[I_{t+1}(u_{t+1})-I_{t+1}(u_t)\bigr]
\end{equation}
where the comparator drift $D_T=\tfrac12\sum_{t=1}^{T-1}\ip{u_{t+1}-u_t}{u_{t+1}+u_t-2\theta_{t+1}}_{\Sigma^{-1}}$ satisfies $|D_T|\le P_T\max_t\half\|u_{t+1}+u_t-2\theta_{t+1}\|_{\Sigma^{-1}}$ with path length $P_T=\sum_{t}\|u_{t+1}-u_t\|_{\Sigma^{-1}}$.
\end{proposition}

Setting $u_t\equiv u$ makes $D_T=0$ and recovers \eqref{eq:sum-weighted-ledger}. The proposition is the SGD instance of the dynamic-regret path-length penalty of shifting-comparator online learning.

The proof is in Appendix~\ref{app:proofs-deferred}.

\subsection{Learning-rate design from intrinsic time}\label{sec:schedules}

The intrinsic clock is
\begin{equation}\label{eq:intrinsic-clock}
V_T=\sum_{t=1}^TQ_t=\half\sum_{t=1}^Tg_t^\top\Sigma_t g_t
\end{equation}
If a comparator budget $\Gamma$ approximates $I_1(u)$, then the fixed-rate part of \eqref{eq:sum-weighted-ledger} balances $\Gamma/\eta+\eta V_T$.
The hindsight equalizer is
\begin{equation}\label{eq:eta-star}
\eta^\star=\sqrt{\frac{\Gamma}{V_T}}
\end{equation}
The predictable online analogue is
\begin{equation}\label{eq:sqrt-schedule}
\eta_t=\min\left\{\eta_{\max}, C\sqrt{\frac{\Gamma}{V_{t-1}+\varepsilon}}\right\},
\qquad
V_{t-1}=\sum_{s<t}Q_s
\end{equation}
For $\Sigma=I$, this is a global AdaGrad-like schedule driven by $\sum_{s<t}\|g_s\|^2$. In the Gaussian-Bayes interpretation it equalizes comparator information with intrinsic time.

AdaGrad is the square-root schedule \eqref{eq:sqrt-schedule} on this clock.
Substituting $V_{t-1}=\half\sum_{s<t}\|g_s\|^2$ into \eqref{eq:sqrt-schedule} gives, in the unclipped regime,
\begin{equation}\label{eq:adagrad-scalar}
\eta_t=C\sqrt{\frac{2\Gamma}{\sum_{s<t}\|g_s\|^2}},\qquad\alpha_t=\frac{C\sqrt{2\Gamma}}{\sqrt{\sum_{s<t}\|g_s\|^2}}
\end{equation}
Comparing with the canonical AdaGrad-Norm schedule \citep{DuchiHazanSinger2011,matthew2010less}, $\alpha_t^{\mathrm{AdaGrad}}=D/\sqrt{\sum\|g_s\|^2}$, we see \eqref{eq:adagrad-scalar} is AdaGrad-Norm with $D=C\sqrt{2\Gamma}$. Per-coordinate AdaGrad arises identically from a diagonal prior $\pi=\calN(0,\diag(\sigma_i^2))$ with $Q_t=\half\sum_i\sigma_i^2g_{t,i}^2$ and a coordinate-wise schedule. Adam's $\beta_2$-EMA replaces the cumulative clock with an exponentially discounted one, $v_{t,i}=(1-\beta_2)\sum_{s\le t}\beta_2^{t-s}g_{s,i}^2$ for coordinate $i$. The $\epsilon$-stabilizer is the clipped regime of \eqref{eq:sqrt-schedule}. Adam's $1/(1-\beta_1^t)$ bias correction cancels leading-order predictable bias.

A local schedule can be obtained by choosing a reference $r_t$ and centering the score $c_t^r(\theta)=\ip{g_t}{\theta-r_t}$. The posterior update is unchanged by the constant shift. Let $S_t=\ip{g_t}{\theta_t-r_t}$, the round's linearized progress toward $r_t$, which we call the \emph{pressure}. The Gaussian mix loss of the centered score at inverse temperature $\eta$ is $\ell^{\mathrm{mix}}_t(\eta)=S_t-\eta Q_t$. Choosing $\eta$ so that this one-round mix loss hits a prescribed target is the \emph{pressure-target} construction of \cite{balsubramani2026adaptive}. If a target $a_t<S_t$ is prescribed, solving the pressure equation $\ell^{\mathrm{mix}}_t(\eta_t)=a_t$ gives
\begin{equation}\label{eq:pressure}
\eta_t=\frac{S_t-a_t}{Q_t}
\end{equation}
A fixed-fraction target $a_t=(1-\lambda_t)S_t$ yields
\begin{equation}\label{eq:polyak-info}
\eta_t=\lambda_t\frac{S_t}{Q_t}=\frac{2\lambda_t\ip{g_t}{\theta_t-r_t}}{g_t^\top\Sigma g_t}
\end{equation}
provided $S_t>0$. For $a_t=f_t(\theta_*)$ (if known) and $\lambda_t=\half$, (\ref{eq:pressure}) collapses to the classical Polyak step size \citep{boris1987introduction}: Polyak-step SGD is the local Bayes update whose current mix loss hits the target $f_t(\theta_*)$. The update calibrates the one-step normalizer, round by round, to a prescribed free-energy level. The half is the information-maximizing choice, since the round transports $\eta_tS_t-\eta_t^2Q_t$ and that is maximized at $\eta_t=S_t/(2Q_t)$; the same step reappears as the greedy schedule \eqref{eq:greedy-descent}.

Appendix~\ref{app:schedules-full} develops further constructions. Matrix clocks extend the intrinsic-time density to predictable diagonal or full-matrix $\Sigma_t$. Parameter-free SGD, the classical subgradient step, and line search run on the same clock. Schedule-Free SGD absorbs the schedule into iterate averaging. Its tight convex rate is a projection of the same identity (\S\ref{sec:schedulefree-ledger}). When the spectrum of a quadratic objective is known, cycling the step size through the reciprocals of the Chebyshev roots on $[\lambda_{\min},\lambda_{\max}]$ attains the accelerated per-round contraction $\rho_{\mathrm{cyc}}=(\sqrt\kappa-1)/(\sqrt\kappa+1)$ in the limit of long cycles (\S\ref{sec:cheby-cyclic}).
Alternating the greedy descent step with a high-temperature escape step, switched on the realized return on intrinsic time, gives pressure-cycled SGD (\S\ref{sec:pressure-cycle}).
On a synthetic two-basin landscape it returns the deep minimum on every seed.
On real networks the trigger fires while the training loss is still falling when it uses mini-batch losses, and after convergence when it uses the full-batch loss, so it separates a stall neither from continued descent nor from convergence (Appendix~\ref{app:interventions}).

\subsection{The stochastic-gradient information ledger}\label{sec:stochastic}

Let $\hat g_t=g_t+\xi_t$ with $\E[\xi_t\mid\calF_t]=0$ and $\E[\|\xi_t\|^2\mid\calF_t]\le\sigma_g^2$. The learner runs the local update $\theta_{t+1}=\theta_t-\eta_t\Sigma\hat g_t$ on $\hat c_t(\theta)=\ip{\hat g_t}{\theta}$. Applying the weighted ledger \eqref{eq:sum-weighted-ledger} to the realized score $\hat g_t$ and splitting the comparator regret of the true gradient $g_t$ off the realized one gives the weighted stochastic decomposition
\begin{equation}\label{eq:stochastic-decomposition}
\sum_{t\le T}\eta_t\ip{g_t}{\theta_t-\theta_*}
=M_T^{\mathrm{est}}+B_T(\rho)+\sum_{t\le T}\eta_t^2\hat Q_t
\end{equation}
where the (weighted) estimation martingale is
\[
M_T^{\mathrm{est}}=\sum_{t\le T}\eta_t\ip{\theta_t-\theta_*}{g_t-\hat g_t},
\qquad \E[M_T^{\mathrm{est}}\mid\calF_0]=0
\]
$B_T(\rho)=I_1(\theta_*)-I_{T+1}(\theta_*)$ is the terminal comparator-information balance.  The quantity $\hat Q_t=\half\hat g_t^\top\Sigma\hat g_t$ is the intrinsic-time density of the stochastic gradient.  With adaptive covariance $\Sigma_t$ an additional weighted covariance-drift term $\sum_t\Xi_{t+1}(\rho)$ \eqref{eq:xi-gaussian} appears on the right, vanishing for fixed $\Sigma$. The regret and the estimation martingale are each weighted by one power of $\eta_t$, the intrinsic-time term by $\eta_t^2$, and the balance $B_T(\rho)$ by none, as in \eqref{eq:sum-weighted-ledger}. Noise enters in two distinct places: (i) as a zero-mean martingale in the regret, $O(\sigma_g\sqrt{\sum_t\eta_t^2\|\theta_t-\theta_*\|^2})$ in typical realizations, zero in expectation; and (ii) as an additive inflation of the intrinsic clock, $\E[\hat V_T]=\E[V_T^{\det}]+\half\sum_t\E\tr\{\Sigma\Cov(\xi_t\mid\calF_t)\}$, where $\hat V_T=\sum_t\hat Q_t$ and $V_T^{\det}=\half\sum_tg_t^\top\Sigma g_t$.

Variance reduction is intrinsic-time reduction.
An estimator $\tilde g_t$ used as the score, with $\E[\|\tilde g_t-g_t\|^2\mid\calF_t]\to0$, has $\hat Q_t\to Q_t^{\det}=\half g_t^\top\Sigma g_t$, recovering the deterministic intrinsic time. This gives an information-theoretic reading of SVRG \citep{rie2013accelerating}, SAGA \citep{aaron2014saga}, and variance-reduced Adam: each reduces the algorithm's intrinsic time, which is the quantity recorded by \eqref{eq:stochastic-decomposition}.

Biased estimators enter through a predictable bias term. For the EMA estimator, Adam-style bias correction cancels the leading-order part of the bias term. A Catoni-robustified score keeps the clock finite under heavy-tailed gradients, close in spirit to gradient clipping. Both readings are developed in Appendix~\ref{app:identities}.

Momentum enters as a predictable forecast. If $m_t\in\R^d$ is any predictable forecast of $g_t$ ($\calF_t$-measurable), the side factor $\exp(\eta_t\ip{m_t}{\theta})$ of \cite{balsubramani2026adaptive} produces the residual composite loss $c_t(\theta)=\ip{g_t-m_t}{\theta}$, so the update becomes optimistic SGD:
\begin{equation}\label{eq:optimistic}
\theta_{t+1}=\theta_t-\eta_t\Sigma(g_t-m_t)-\eta_{t+1}\Sigma m_{t+1}
\end{equation}
The intrinsic time is computed on the residual:
\begin{equation}\label{eq:residual-V}
V_T(c)=\half\sum_{t\le T}\|g_t-m_t\|_{\Sigma}^2
\end{equation}
Well-forecast gradients shrink $V_T$ below $\half\sum\|g_t\|_\Sigma^2$, giving the SGD analogue of optimistic-Hedge improvement \citep{ChiangEtAl2012,rakhlin2013online}. The residual \eqref{eq:residual-V} is the term retained by the classical optimistic bound, obtained through a Young inequality. It is not by itself the cost term of the information ledger. Writing optimistic SGD in its dual form $z_{t+1}=z_t-\eta_t\Sigma g_t$, $\theta_t=z_t-\eta_t\Sigma m_t$, the exact one-step term is
\begin{equation}\label{eq:optimistic-exact}
\eta_t\ip{g_t}{\theta_t-u}=\half\|u-z_t\|_{\Sigma^{-1}}^2-\half\|u-z_{t+1}\|_{\Sigma^{-1}}^2+\eta_t^2\Bigl[\half\|g_t-m_t\|_\Sigma^2-\half\|m_t\|_\Sigma^2\Bigr]
\end{equation}
The correction $-\half\|m_t\|^2_\Sigma$ is not negligible. Dropping it breaks the equality by $\half\eta_t^2\|m_t\|_\Sigma^2$ per round, which equals the intrinsic-time cost $\eta_t^2Q_t$ of plain SGD when $m_t=g_t$. The bracketed term is signed, so a confident forecast can make a round's cost term negative.

The one-step-back forecast $m_t=g_{t-1}$ \citep{ChiangEtAl2012} gives the residual $\half\sum\|g_t-g_{t-1}\|_\Sigma^2$, which is small when gradients drift slowly. The EMA forecast $m_t=(1-\beta)m_{t-1}+\beta g_{t-1}$ averages past gradients with exponentially decaying weights. Nesterov acceleration \citep{yurii1983method} takes its gradient step at the look-ahead point $y_t=\theta_t+\mu(\theta_t-\theta_{t-1})$, at momentum coefficient $\mu$ and constant step $\eta$.
Its iterates are those of optimistic SGD with the forecast $m_{t+1}=m_t-g_t+\nabla f_t(y_t)-\eta^{-1}\mu\Sigma^{-1}(\theta_t-\theta_{t-1})$, which is available before round $t+1$ at the cost of one extra gradient query at $y_t$. The oracle lookahead $m_t=g_t$ drives the residual to zero but is not admissible, since admissibility requires $\calF_t$-measurability.

\subsection{Preconditioning as prior geometry}\label{sec:precon}

A non-isotropic Gaussian prior yields preconditioned SGD.
For $\pi=\calN(0,P)$ the Gaussian bookkeeping gives $q_t=\calN(\theta_t,P)$ and $\theta_{t+1}=\theta_t-\eta_tPg_t$. The intrinsic-time density becomes $Q_t=\half\ip{g_t}{Pg_t}=\half\|g_t\|_P^2$: the intrinsic clock of preconditioned SGD is half the cumulative squared preconditioned gradient norm.

Setting $P=\calF(\theta_t)^{-1}$, where $\calF(\theta_t)$ is the Fisher information matrix, yields natural gradient \citep{amari1998neural}. Within our framework this is an adaptive prior $\pi_t=\calN(0,\calF(\theta_t)^{-1})$. For exponential families with log-likelihood losses, $\half\|g_t\|_{\calF^{-1}}^2$ is, to second order, the per-round relative entropy between neighboring model predictions once multiplied by $\eta_t^2$ (in the information ledger $Q_t$ enters only through the per-round cost $\eta_t^2Q_t$). Composing this preconditioner with a worst-case ascent step gives the Fisher SAM of \citep{kim2022fishersam}. Its entry in the information ledger is a substitution into the SAM identity (\S\ref{sec:sam-ledger}).

Block-structured $P_t$ as in K-FAC \citep{james2015optimizing} or Shampoo \citep{vineet2018shampoo} gives a pathwise identity with intrinsic time $\half\sum\|g_t\|_{P_t}^2$. Minimizing cumulative intrinsic time selects a preconditioner only once the scale of $P_t$ is fixed: $Q_t$ is homogeneous of degree one in $P$, so an unconstrained minimization drives $P\to0$, at which $Q_t$ and the update both vanish. Under a normalization such as $\tr(P_t^{-1})$ fixed, the expected criterion becomes a well-posed design principle (minimize $\E\|g\|_{P_t}^2$ at fixed budget) and a diagnostic (is the realized $\|g\|_{P_t}$ actually small?).

For gradients with a fixed second moment $C=\E[g_tg_t^\top]\succ0$ the expected intrinsic time is $\E[Q_t]=\half\tr(PC)$.
Fixing $\tr(P^{-1})$ gives the minimizer $P\propto C^{-1/2}$, the power that full-matrix AdaGrad and Shampoo apply.
Subtracting the entropy of the Gaussian prior, $\half\log\det P$ up to a constant, turns the criterion into covariance estimation under relative entropy: for every $P\succ0$,
\begin{equation}\label{eq:precon-kl}
\half\tr(PC)-\half\log\det P=\KL\bigl(\calN(0,C)\,\|\,\calN(0,P^{-1})\bigr)+\frac d2+\half\log\det C
\end{equation}
which follows by expanding the Gaussian relative entropy.
Over all $P$ the minimizer is $P=C^{-1}$, the natural-gradient preconditioner.
Over Kronecker-factored $P=A\otimes B$ the right side is the objective by which \cite{lin2025shampoo} recast Shampoo's second-moment estimate, with $S=P^{-1}$ and their damping absorbed into $C$, so the criterion selects their Kronecker estimate $S^\star$.
The criterion's preconditioner is then $(S^\star)^{-1}$, where Shampoo and its relative-entropy variant apply $(S^\star)^{-1/2}$.

Exponential-weights in an ambient geometry other than relative entropy yields mirror descent. For SGD the Euclidean geometry is natural, but we could equally use the entropic geometry on the probability simplex (exponentiated gradients) or a general Bregman divergence (mirror-descent SGD). In every case the intrinsic time is the Bregman divergence $D_\psi(\theta_t,\theta_{t+1})$ between consecutive iterates, which reduces to $\tfrac{\eta_t^2}{2}\|g_t\|^2_\Sigma$ for a quadratic mirror map and is otherwise only approximated by a local-curvature expansion. The identity is exact for unconstrained mirror descent and for constraint sets that are affine within the mirror map's domain, such as the simplex in the entropic geometry, where the projection is a normalization. Under any other Bregman projection the Pythagorean step becomes an inequality, so the account holds in that direction only.

\subsection{Optimizer families as readings of the information ledger}\label{subsec:optimizer-ledger}

Each SGD-class optimizer is one reading of the same information ledger, fixed by its choice of preconditioner $\Sigma_t$, forecast $m_t$, and schedule $\eta_t$ (Table~\ref{tab:optimizer-ledger}). The quantity it records is $Q_t=\half s_t^\top\Sigma_t s_t$ in the score the method actually plays (the vector $s_t$ in its update $\theta_{t+1}=\theta_t-\eta_t\Sigma_ts_t$), which for the first-moment methods below is a moment estimate in place of the raw gradient $g_t$. When $\Sigma_t$ varies with $t$ the pathwise account includes the metric-drift term $\Delta_t^{\mathrm{metric}}(u)$ of \eqref{eq:add-variable-metric-ledger} alongside $Q_t$. The rows marked $\dagger$ below are exact only with that term included.

\begin{table}[ht]
\centering
\small
\renewcommand{\arraystretch}{1.25}
\begin{tabularx}{\textwidth}{@{}>{\raggedright\arraybackslash}p{0.17\textwidth}>{\raggedright\arraybackslash}p{0.32\textwidth}>{\raggedright\arraybackslash}X@{}}
\toprule
\textbf{Optimizer} & \textbf{Choice of $(\Sigma_t,m_t,\eta_t)$} & \textbf{Information-ledger term it incurs} \\
\midrule
Plain SGD & fixed $\Sigma$, no forecast & $Q_t=\half g_t^\top\Sigma g_t$ in fixed $\Sigma$-geometry \\
Natural gradient$^\dagger$~\cite{amari1998neural} & $\Sigma_t=\calF(\theta_t)^{-1}$, $\calF$ the Fisher matrix & Fisher-metric intrinsic time $\half\|g_t\|_{\calF^{-1}}^2$ \\
K-FAC$^\dagger$~\cite{james2015optimizing}, Shampoo$^\dagger$~\cite{vineet2018shampoo} & block-structured $\Sigma_t=P_t$ & $\half\|g_t\|_{P_t}^2$ per round \\
Muon$^\dagger$~\cite{keller2024muon} & momentum buffer $B_t$, orthogonalized (Newton--Schulz) & $\half\langle B_t,(B_tB_t^\top)^{-1/2}B_t\rangle=\half\|B_t\|_*$, the nuclear norm; the Newton--Schulz iterate is a deliberate approximation to this \\
AdaGrad-Norm~\cite{DuchiHazanSinger2011} & $\eta_t\propto V_t^{-1/2}$ (the scalar variant; per-coordinate AdaGrad instead sets $\Sigma_t$) & terminal-remainder-saturating square-root clock \\
RMSProp$^\dagger$~\cite{tijmen2012lecture}, Adam$^\dagger$~\cite{diederik2014adam} & $\Sigma_t=\diag(\sqrt{\hat v_t}+\varepsilon)^{-1}$, score $\hat m_t$, fixed $\eta$ & $\half\|\hat m_t\|^2_{\Sigma_t}$ \\
Schedule-Free SGD~\cite{aaron2024road} & averaging in place of $\eta_t$ & $V_T$-tracking accumulator (\S\ref{sec:schedulefree-ledger}) \\
SAM~\cite{pierre2021sharpness} & ascent-perturbed score & path-averaged-Hessian directional-curvature term (\S\ref{sec:sam-ledger}) \\
ASAM~\cite{kwon2021asam} & ascent metric $\diag(\theta_t)^2$ & SAM entry with an axis-magnitude metric \\
Fisher-SAM~\cite{kim2022fishersam} & natural-gradient ascent & SAM entry with the natural-gradient ascent direction \\
Lion~\cite{xiangning2023symbolic} & score $\sign(c_t)+\lambda\theta_t$, $c_t$ an interpolation of $m_t$ and $g_t$ (two distinct decay rates) & $\half\|\sign(c_t)+\lambda\theta_t\|^2_\Sigma$, which at $\Sigma=I,\lambda=0$ is the constant $d/2$: a fixed cost per step for $\ell_1$ pressure $\|c_t\|_1$ \\
Sophia$^\dagger$~\cite{hong2023sophia} & score $\mathrm{clip}(\hat m_t/\max(\rho\hat h_t,\varepsilon),1)$, $\hat h_t$ a diagonal stochastic Hessian estimate & $\half\|\mathrm{clip}(\cdot)\|^2_\Sigma$; the elementwise clip binds on most coordinates, so the unclipped $\half g_t^\top\hat H_t^{-1}g_t$ overstates the realized term \\
SOAP$^\dagger$~\cite{nikhil2024soap} & Adam in the running second-moment eigenbasis; score is the rotated moment estimate $S_t$ & $\half\|S_t\|^2_\Sigma$ \\
Momentum, Nesterov, optimistic SGD~\cite{boris1964some,yurii1983method} & $m_t$ a predictable forecast of $g_t$ & $\half\|g_t-m_t\|^2_\Sigma-\half\|m_t\|^2_\Sigma$. The classical bound keeps only the residual, while the information ledger also retains $-\half\|m_t\|^2_\Sigma$ \\
\bottomrule
\end{tabularx}
\caption{\textbf{Every SGD-class optimizer is one reading of the same information ledger.} Each row gives an optimizer's choice of preconditioner $\Sigma_t$, forecast $m_t$, and schedule $\eta_t$, and the term incurred by that choice. Rows marked $\dagger$ have a time-varying $\Sigma_t$, for which the pathwise account is exact only with the metric-drift term $\Delta_t^{\mathrm{metric}}(u)$ of \eqref{eq:add-variable-metric-ledger} included alongside $Q_t$. Derivations are in \S\ref{sec:schedules}, \S\ref{sec:stochastic}, \S\ref{sec:precon}, and \S\ref{sec:sam-ledger}.}
\label{tab:optimizer-ledger}
\end{table}

\subsection{Zeroth-order optimization: the gradient-free information ledger}\label{sec:zeroth}

The information ledger so far is driven by the local gradient $g_t$. When the gradient is unavailable and only function values can be queried, the gradient is replaced by a random finite-difference probe \citep{nesterov2017random}: draw a perturbation $\nu_t$ from a law $\calP$ on $\R^d$ and form the \emph{zeroth-order score}
\begin{equation}\label{eq:zo-score}
\hat g_t=\frac{f(\theta_t+\mu \nu_t)-f(\theta_t)}{\mu}\,\nu_t
\xrightarrow[\mu\to0]{}\ip{g_t}{\nu_t}\nu_t
\end{equation}
the finite-difference oracle at smoothing radius $\mu>0$ and its $\mu\to0$ directional limit. The learner runs the ordinary preconditioned update $\theta_{t+1}=\theta_t-\eta_t\Sigma\hat g_t$. The per-step ledger \eqref{eq:weighted-ledger} is an algebraic identity in whatever score drives the update, so it holds with $\hat g_t$ in place of $g_t$. The entire zeroth-order content is the conditional law of $\hat g_t$, which the sampling law $\calP$ fixes.

\begin{proposition}[Exact zeroth-order moments]\label{prop:zo-moments}
Let $\nu_t$ be drawn independently of $\calF_t$ from a law $\calP$ with mean zero and covariance $C=\E[\nu_t\nu_t^\top]$. In the directional limit $\hat g_t=\ip{g_t}{\nu_t}\nu_t$,
\begin{equation}\label{eq:zo-mean}
\E[\hat g_t\mid\calF_t]=Cg_t
\end{equation}
for every such $\calP$. If in addition $\nu_t\sim\calN(0,C)$, then the finite-$\mu$ statement holds as well: $\E[\hat g_t\mid\calF_t]=C\nabla f_\mu(\theta_t)$, the same identity with $g_t$ replaced by the gradient of the smoothed objective $f_\mu(\theta)=\E_{\nu\sim\calP}[f(\theta+\mu \nu)]$. This finite-$\mu$ form is Gaussian-specific (it rests on Stein's identity) and fails for other probe laws; for a Rademacher probe on a cubic the two sides differ by a term of order $\mu^2$. Under the Gaussian probe,
\begin{equation}\label{eq:zo-second}
\E[\hat g_t\hat g_t^\top\mid\calF_t]=(g_t^\top Cg_t)\,C+2\,(Cg_t)(Cg_t)^\top
\end{equation}
and the exact zeroth-order intrinsic-time entry is
\begin{equation}\label{eq:zo-intrinsic}
Q_t^{\mathrm{ZO}}=\half\,\E[\hat g_t^\top\Sigma\hat g_t\mid\calF_t]
=\half\Big[(g_t^\top Cg_t)\tr(\Sigma C)+2\,(Cg_t)^\top\Sigma(Cg_t)\Big]
\end{equation}
\end{proposition}

The proof is in Appendix~\ref{app:proofs-deferred}.

The expected step is $-\eta_t\Sigma\,\E[\hat g_t]=-\eta_t\Sigma Cg_t$, the preconditioned SGD step of \S\ref{sec:precon} with preconditioner $\Sigma C$. Isotropic sampling $C=I$ is unbiased for $g_t$; a general sampling covariance $C=B$ realizes the metric $B$ through the distribution of the probe, with no explicit matrix multiply and no extra function queries. The choice of $\calP$'s covariance is the zeroth-order counterpart of choosing a prior covariance in \S\ref{sec:precon}.

At $\Sigma=I$, $C=I$ identity \eqref{eq:zo-intrinsic} reduces to $Q_t^{\mathrm{ZO}}=\half(d+2)\|g_t\|^2$, which is $d+2$ times the first-order intrinsic time $\half\|g_t\|^2$. This dimension penalty is the form the $\Theta(d)$ slowdown of gradient-free methods \citep{nesterov2017random} takes in the information ledger: their convergence rate is the projection of \eqref{eq:sum-weighted-ledger} onto the comparator $\theta_*$.
The dimension factor enters only through $Q_t^{\mathrm{ZO}}$. For the directional-derivative oracle, before any smoothing, the analysis of \cite{nesterov2017random} bounds this factor by $d+4$.

On a quadratic the sampling covariance preconditions the expected dynamics. The sampler $C\propto H^{-1}$, which concentrates probe mass on the low-curvature directions, contracts the expected error by the same factor in every direction, whatever the condition number.

A query-only curvature-probed sampler realizes this covariance from function values alone.
Drawing the probe from a preconditioner and scaling the step by it share the same expected step \eqref{eq:zo-mean}, so the second moment \eqref{eq:zo-second} decides between them (Appendix~\ref{app:zo-full}).
The appendix also extends the zeroth-order moment identities to any sampling law through its cumulant information ledger.
\subsection{Nonconvex objectives and saddle point information}\label{sec:saddle}

The same identity holds for nonconvex functions, but $B_F$ is signed. The sign is the surface's curvature information. For general nonconvex $f_t$ the subgradient inequality fails, so the linearized regret is not an upper bound on $\sum(f(\theta_t)-f(\theta_*))$. But the identity still holds for the linearized regret $\sum\ip{\nabla f(\theta_t)}{\theta_t-\theta_*}$ itself.

Let $s$ be a critical point and suppose locally
\begin{equation}\label{eq:quadratic}
F(s+x)=F(s)+\half x^\top Hx
\end{equation}
For deterministic gradient descent with $\Sigma=I$, $x_{t+1}=(I-\eta H)x_t$. Let $v_j$ be an orthonormal eigenvector of $H$ with eigenvalue $\lambda_j$, and $z_{t,j}=v_j^\top x_t$, $I_{t,j}(s)=\half z_{t,j}^2$.

In an eigendirection of the local Hessian with eigenvalue $\lambda_j$, the squared distance to the critical point evolves under deterministic GD by the multiplicative factor $(1-\eta\lambda_j)^2$ per step, contracting in stable directions ($0<\eta\lambda_j<2$) and amplifying in unstable directions ($\lambda_j<0$). Adding zero-mean stochastic gradient noise injects an additional $\half\eta^2\Var(\xi_{t,j})$ per step into the information distance along that direction. The saddle-escape mechanism has a single-step form: an unstable direction with eigenvalue $-\gamma$ amplifies the eigendirection's information distance by $(1+\eta\gamma)^2$ per deterministic step.
The noise injection supplies the displacement along that direction when it starts at zero.

\begin{proposition}[Exact saddle point information dynamics]\label{prop:saddle}
Under the quadratic model \eqref{eq:quadratic}, deterministic gradient descent satisfies
\begin{equation}\label{eq:saddle-det}
I_{t+1,j}(s)=(1-\eta\lambda_j)^2I_{t,j}(s)
\end{equation}
If the stochastic gradient has the form $g_t=Hx_t+\xi_t$ and $\E[\xi_t\mid\calF_t]=0$, then
\begin{equation}\label{eq:saddle-stoch}
\E[I_{t+1,j}(s)\mid\calF_t]
=(1-\eta\lambda_j)^2I_{t,j}(s)+\frac{\eta^2}{2}\Var(\xi_{t,j}\mid\calF_t)
\end{equation}
where $\xi_{t,j}=v_j^\top\xi_t$.
\end{proposition}

The proof is in Appendix~\ref{app:proofs-deferred}.

\begin{itemize}[leftmargin=1.5em,itemsep=2pt]
\item \emph{Stable direction, $\lambda_j>0$:} without noise, information distance contracts geometrically when $0<\eta\lambda_j<2$ and the eigendirection term in $V_T$ saturates as a geometric series; noise adds $\eta^2\Var(\xi_{t,j})/2$ per step.
\item \emph{Unstable direction, $\lambda_j=-\gamma<0$:} $I_{t+1,j}(s)=(1+\eta\gamma)^2I_{t,j}(s)$, so the saddle point is \emph{information-repelling} in that direction; noise adds $\eta^2\Var(\xi_{t,j})/2$.
\item \emph{Flat direction, $\lambda_j=0$:} without noise, the coordinate is neutral, with $Q_t\equiv0$ along it, so no progress is made and no information cost accrues; noise adds $\eta^2\Var(\xi_{t,j})/2$ per step, so the coordinate diffuses.
\end{itemize}

Escape from a strict saddle point needs a displacement from $s$ along the unstable direction $v_-$, an eigenvector of the most negative eigenvalue $\lambda_{\min}$, either from the start or from the noise.
With $\sigma_\xi^2$ the noise variance along $v_-$, the injection of $\half\eta^2\sigma_\xi^2$ per step in \eqref{eq:saddle-stoch} acts in expectation as an extra $\eta\sigma_\xi^2/(2|\lambda_{\min}|)$ in the squared start displacement, to leading order in $\eta|\lambda_{\min}|$.
The larger of the two displacements sets the leading-order time to reach distance $r_{\mathrm{esc}}$ from $s$ along $v_-$
\begin{equation}\label{eq:escape-time}
T_{\mathrm{esc}}\asymp\frac{1}{|\lambda_{\min}|\eta}\log\frac{r_{\mathrm{esc}}}{\varepsilon_{\mathrm{eff}}},\qquad\varepsilon_{\mathrm{eff}}=\max\Bigl(|v_-^\top(\theta_0-s)|,\ \sigma_\xi\sqrt{\eta/(2|\lambda_{\min}|)}\Bigr)
\end{equation}
The escape phase of \cite{chi2017how} has the same form, $1/(\eta|\lambda_{\min}|)$ times a logarithm, up to an absolute constant, and the escape bound of \cite{rong2015escaping} has the same order in $\eta$.
The dimension factors in these statements, polynomial in \cite{rong2015escaping} and polylogarithmic in \cite{chi2017how}, depend on Hessian-Lipschitz constants and on the shape of the noise covariance.

For a nonlinear objective, the same coordinate recursion gains residual terms $r(x)=\nabla F(s+x)-Hx$. Hessian-Lipschitz assumptions bound those residuals, recovering the usual perturbation or SGD escape analyses. The information identity itself does not require those bounds; they enter when we want a clean escape-time theorem with explicit dimension dependence.

Let $P_-$ project onto negative-curvature eigenspaces (estimable, e.g., via Lanczos on Hessian-vector products). The \emph{unstable intrinsic clock}
\begin{equation}\label{eq:V-minus}
V_T^-=\half\sum_tg_t^\top P_-\Sigma P_-g_t
\end{equation}
should predict escape better than the total gradient norm when the dominant difficulty is a saddle point, as it does on synthetic strict saddle points (Appendix~\ref{app:sanity}). This matches the ``negative curvature detection'' of \cite{hadi2018escaping} but gives it an information-theoretic rationale.

The Polyak--\L ojasiewicz (PL) condition $\half\|\nabla F(\theta)\|^2\ge\mu(F(\theta)-F^*)$ is a self-bounding condition on the intrinsic-time density:
\begin{equation}\label{eq:PL}
Q_t=\half\|\nabla F(\theta_t)\|^2\ge\mu(F(\theta_t)-F^*)
\end{equation}
With the one-step descent inequality of an $L$-smooth $F$, $F(\theta_{t+1})\le F(\theta_t)-\eta Q_t$ at any fixed step $\eta\le1/L$, condition \eqref{eq:PL} gives exponential decay of $F(\theta_t)-F^*$ with rate $\eta\mu$. Smoothness cannot be dropped here: $F(\theta)=|\theta|^{3/2}$ satisfies \eqref{eq:PL} on $|\theta|\le1$ with $\mu=9/8$ but has unbounded curvature at its minimizer, so gradient descent at any fixed $\eta>0$ settles, from almost every starting point, into a two-cycle at $|\theta|=9\eta^2/16$ on which $F(\theta_t)-F^*$ stalls at $27\eta^3/64$. More generally, a \L ojasiewicz exponent $\vartheta\in[0,1/2]$ with $\|\nabla F\|^{1/(1-\vartheta)}\ge c(F-F^*)$ yields polynomial or exponential decay of $F(\theta_t)-F^*$ depending on $\vartheta$. Condition \eqref{eq:PL} is one information-theoretic form of benign non-convexity \citep{hamed2016linear}.

The information ledger also yields operational measures of loss-surface difficulty: the weighted information complexity $\mathsf C_T(u)$, the worst-case, average and least intrinsic time of trajectories, and a policy-optimized surface complexity $\calK_T(\varepsilon)$. For fixed-rate gradient descent on a quadratic objective, $\limsup_{T\to\infty}\calK_T(\varepsilon)$ is bounded by an eigenvalue-by-eigenvalue posterior-information cost plus a prior-distance penalty and a finite net path term (Proposition~\ref{prop:surface-quadratic}, Appendix~\ref{app:surface-full}).

\section{Flatness and generalization}\label{sec:flatgen}

\subsection{Flatness as terminal posterior information}\label{sec:flatness}

Flatness enters naturally through terminal posterior information. Let $q=\calN(\theta,\Sigma)$ and $\pi=\calN(0,\tau^2I)$. Then
\begin{equation}\label{eq:terminal-kl}
\KL(q\|\pi)
=\half\left[\frac{\|\theta\|^2+\tr\Sigma}{\tau^2}-d+
\log\frac{\tau^{2d}}{\det\Sigma}\right]
\end{equation}
Near a local minimum $\theta_\star$, $F(\theta_\star+x)\approx F(\theta_\star)+\half x^\top Hx$ with $H\succeq0$, so
\[
\E_{\calN(\theta_\star,\Sigma)}F
\approx F(\theta_\star)+\half\tr(H\Sigma)
\]
A flat basin permits large posterior covariance, hence large entropy and small relative entropy to the prior, while maintaining small expected loss.

Consider the free energy
\begin{equation}\label{eq:local-free-energy}
\mathcal F_\beta(q)=\E_qF+\beta^{-1}\KL(q\|\pi)
\end{equation}
Restrict to $q=\calN(\theta_\star,\Sigma)$ and use the quadratic approximation. If $H=U\diag(\lambda_i)U^\top$, the minimizing covariance in this approximation is
\begin{equation}\label{eq:free-energy-cov}
\Sigma_\star=U\diag\left(\frac{1}{\beta\lambda_i+\tau^{-2}}\right)U^\top
\end{equation}
Substituting back gives, up to constants independent of the basin,
\begin{equation}\label{eq:logdet-flatness}
\mathcal F_\beta(\theta_\star)
\approx
F(\theta_\star)+\frac{\|\theta_\star\|^2}{2\beta\tau^2}
+\frac{1}{2\beta}\sum_{i=1}^d\log(1+\beta\tau^2\lambda_i)
\end{equation}
Thus flatness is a \emph{log-determinant information penalty}; the largest Hessian eigenvalue alone does not capture it. Zero or near-zero curvature directions contribute little; sharp directions contribute logarithmically once the posterior is optimized.

The basin-evidence form $\mathcal E_\beta(\calB)=-\beta^{-1}\log\int_\calB e^{-\beta F(\theta)}\pi(d\theta)$ makes a basin informationally favorable when it combines low loss, large prior mass, and low effective curvature determinant (Appendix~\ref{app:flatness-full}). A quadratic composite loss lifts the frozen covariance of the linearized identity: precisions add as $\Sigma_{t+1}^{-1}=\Sigma_t^{-1}+\alpha H_t$, so after $T$ rounds near a minimum the posterior tightens to $\Sigma_T\approx(\alpha TH)^{-1}$, sharp directions narrowing while flat directions stay broad. Raw Euclidean Hessian flatness is not invariant under reparameterization (the warning of \cite{laurent2017sharp}): a function-preserving change of coordinates shifts the bare log-determinant by a Jacobian-volume term, while the relative entropy to a prior transformed along with the posterior stays invariant. The truncated estimators practitioners compute inherit a coordinate dependence of their own; estimating the log-determinant, with the coordinate-free observables that survive at network scale, is treated in the same appendix. Sharpness-aware minimization sits inside the identity with no truncation. The fundamental theorem of calculus along the ascent ray writes the SAM ascent gradient as the ordinary gradient plus a path-averaged Hessian correction. The SAM update is the preconditioned SGD update run with the worst-case-perturbed score. Its intrinsic-time density is
\[
Q_t^{\mathrm{SAM}}
=Q_t+\frac{\rho_t}{\|g_t\|}\,g_t^\top\Sigma\bar H_tg_t
+\frac{\rho_t^2}{2\|g_t\|^2}\,g_t^\top\bar H_t\Sigma\bar H_tg_t
\]
for $\bar H_t$ the path-averaged Hessian along the ascent ray, valid at every order in $\rho_t$ (Corollary~\ref{cor:sam-ledger}); Fisher SAM changes only the ascent direction, to the natural-gradient one. The generalization advantage documented for Fisher SAM \citep{kim2022fishersam} is not mediated by the log-determinant penalty alone (Appendix~\ref{sec:sam-ledger}).
\subsection{PAC-Bayes from the information ledger}\label{sec:gen}

The information ledger does not by itself prove test-set generalization. It explains how SGD transports information and accrues intrinsic time along the training trajectory. To obtain a generalization theorem, place a posterior $q_T$ around the final iterate and apply a PAC-Bayes bound. The natural description length becomes
\begin{equation}\label{eq:two-part}
\text{optimization information}+\text{terminal posterior information}
\end{equation}
where the first component is the path term $\sum\eta_t^2Q_t-\sum\eta_tB_t+\text{drift}$ and the second is $\KL(q_T\|\pi)$.

For $q_{T+1}=\calN(\theta_{T+1},\sigma^2I)$ and $\pi=\calN(0,\sigma^2I)$, $\KL(q_{T+1}\|\pi)=\|\theta_{T+1}\|^2/(2\sigma^2)$: the closed-form complexity is the $L^2$ regularizer. Let $\ell$ be a $[0,1]$-valued loss, let the training set consist of $N$ examples drawn i.i.d.\ from the data distribution, and let the prior $\pi$ be fixed before that sample is drawn. A PAC-Bayes bound \citep{mcallester1999pacbayesian,seeger2002pacbayesGP,dziugaite2017computing} then holds with probability at least $1-\delta$ over the draw of the sample, simultaneously for every posterior $q_{T+1}$:
\begin{equation}\label{eq:pac-bound}
\E_{\theta\sim q_{T+1}}\ell^{\mathrm{test}}(\theta)
\le\E_{\theta\sim q_{T+1}}\ell^{\mathrm{train}}(\theta)
+\sqrt{\frac{\KL(q_{T+1}\|\pi)+\log(N/\delta)}{2(N-1)}}
\end{equation}
With the complexity equal to the $L^2$ regularizer, the bound is a generalization certificate for ridge regularization.

The sharpness certificate uses the tightened posterior $\Sigma_T\approx(\alpha TH)^{-1}$ of \S\ref{sec:flatness}. Substituting $\Sigma=(\alpha TH)^{-1}$ into the Gaussian relative entropy \eqref{eq:terminal-kl} ($q_{T+1}=\calN(\theta_{T+1},\Sigma_T)$, $\pi=\calN(0,\sigma^2I)$) gives
\begin{equation}\label{eq:sharpness-kl}
\KL(q_{T+1}\|\pi)
=\half\left[\log\det(\alpha T\sigma^2H)+\frac{\tr(H^{-1})}{\alpha T\sigma^2}-d\right]+\frac{\|\theta_{T+1}\|^2}{2\sigma^2}
\end{equation}
the log-determinant term contributing the sign $+\log\det(\alpha T\sigma^2H)=d\log(\alpha T\sigma^2)+\log\det H$ and the trace term the inverse-Hessian dependence $\tr(H^{-1})/(\alpha T\sigma^2)$, which is $o(1)$ in the large-$T$ regime. The basin-curvature contribution to the complexity is therefore $+\half\log\det H$ at leading order. Sharper terminal curvature (larger $\log\det H$) tightens the posterior along its sharp directions, raising the relative entropy to the broad prior and so loosening the generalization certificate; equivalently, flatness $-\log\det H$ lowers the bound. In a large-scale empirical comparison of generalization measures, sharpness-based measures performed best overall \cite{yiding2020fantastic}. Among them the certificate singles out $\log\det H$, minus twice the log-volume of the ellipsoid $\{x:x^\top Hx\le1\}$ up to a constant: flat minima (small $\log\det H$, large volume) incur the lighter complexity. Trace- or spectral-radius-based sharpness are \emph{not} reparameterization-invariant in the sense of \S\ref{sec:flatness}.
At network scale a data-independent isotropic prior is informative at some scale only if the curvature spectrum is nearly flat.
Writing $r_i=\alpha T\sigma^2\lambda_i$ for the eigenvalues $\lambda_i$ of $H$, the log-determinant and trace terms of \eqref{eq:sharpness-kl} equal $\half\sum_i(\log r_i+1/r_i-1)$, a sum of $d$ nonnegative terms that vanish only at $r_i=1$.
The minimum of that sum over $\alpha T\sigma^2$ is $\tfrac d2\log(\mathrm{AM}/\mathrm{GM})$ for the arithmetic and geometric means of the $1/\lambda_i$.
This value is a floor on the relative entropy at every prior scale.
For a loss with values in $[0,1]$ the bound \eqref{eq:pac-bound} falls below the trivial value $1$ only if $\KL(q_{T+1}\|\pi)+\log(N/\delta)<2(N-1)$, which at $d=1.47$ million and $N=50{,}000$ requires $\mathrm{AM}/\mathrm{GM}<1.15$, an inverse-curvature spectrum nearly flat across the parameters.
A rank-$32$ Lanczos probe of trained Wide-ResNet-28-2 endpoints cannot establish that condition (Appendix~\ref{app:cifar-forensics}).
On its $512$-image probe set the top Hessian eigenvalues run from $0.3$ to $81$ and every endpoint has negative eigenvalues, where the tightened posterior is not defined. For a positive-definite surrogate the floor depends on all $d$ eigenvalues.
On the $N=50{,}000$ training images of $K=10$ classes the Gauss--Newton part of the cross-entropy Hessian has rank at most $N(K-1)=450{,}000$, so on a subspace of at least $69\%$ of the $1.47$ million parameter dimensions the cross-entropy Hessian reduces to the part weighted by the prediction residuals.
The certificate keeps its ordering, monotone in $\log\det H$: along the batch axis it has the sign the flatness bound implies, while across optimizers it orders the optimizers against the gap as the raw log-determinant does ($\rho(\KL,\mathrm{gap})=-0.50$ over $12$ endpoints).
A Gaussian prior fitted to the data lowers the floor to the extent that its covariance tracks the endpoint curvature.
For $\pi=\calN(\mu_\pi,c\,\Sigma_\pi)$ and $q_{T+1}=\calN(\theta_{T+1},(\alpha TH)^{-1})$, minimizing the relative entropy over the scale $c$ gives
\begin{equation}\label{eq:fitted-prior-kl}
\min_{c>0}\KL(q_{T+1}\|\pi)=\frac d2\log\frac{\mathrm{AM}}{\mathrm{GM}}+\frac d2\log\left(1+\frac{\|\theta_{T+1}-\mu_\pi\|^2_{\Sigma_\pi^{-1}}}{(\mathrm{AM}) d}\right)
\end{equation}
with $\mathrm{AM}$ and $\mathrm{GM}$ now the arithmetic and geometric means of the eigenvalues of $(\alpha T\,\Sigma_\pi^{1/2}H\Sigma_\pi^{1/2})^{-1}$.
The first term vanishes when $\Sigma_\pi\propto H^{-1}$ and is invariant to rescaling $H$, so the level of the curvature enters only through the second term, which vanishes when the endpoint sits at the prior mean.
Such a prior is a function of the data, so a valid bound either fits it on a held-out split or accounts for the dependence information-theoretically, through conditional mutual information \cite{thomas2020reasoning} or a mutual-information bound along the training trajectory \cite{peng2025leveraging}.
The information ledger supplies the optimization-side terminal information $I_{T+1}(u)$ that such a bound must localize.

\subsection{Noise and momentum as entries of the information ledger}\label{sec:noise-momentum}

Adding noise or momentum to SGD changes nothing structural in the information ledger: each attaches as one explicit term. The derivations are in Appendix~\ref{sec:addendum-langevin-momentum}, where Table~\ref{tab:noise-momentum} lists the correspondences.

Mini-batch score noise, post-step Gaussian injection, and Langevin heat contribute one additive quadratic entry each. The expected ledger depends on the perturbation only through its covariance, but the three attach at different places and scale differently in $\eta_t$. Finite-cumulant non-Gaussian noise replaces the Gaussian quadratic by the cumulant-generating function $\Lambda_t(v)$ (Proposition~\ref{prop:add-cgf-noise}), of which it is the second-order term. Heavy-tailed and $\alpha$-stable noise preserves the pathwise and channel relative-entropy identities, while the conditional-expectation entry is finite only after truncation, robustification, or a fractional/jump regime. State-dependent preconditioning adds a metric-drift term with an It\^o correction. Phase-space momentum resolves into conservative Hamiltonian transport plus an Ornstein--Uhlenbeck thermostat channel, with heavy-ball a linear phase-space recursion (Proposition~\ref{prop:add-phase-info-recursion}) and Nesterov a lookahead-transport correction.

\section{Experimental evaluation}\label{sec:experiments}

The terms of \eqref{eq:weighted-ledger}, \eqref{eq:loss-path}, and \eqref{eq:stochastic-decomposition} are measured along real training runs and on problems in which a single term has a closed form (settings and protocols in Appendix~\ref{app:experiments-full}).
These terms separate runs whose losses agree.

\subsection{Optimizers at a matched training loss}\label{sec:eval-regime-discrimination}

Five optimizers train a two-layer network full-batch on a handwritten-digit dataset and stop when their training loss first crosses a common target (Appendix~\ref{app:regime-discrimination}).
Their training losses then agree to within $1.04\times$ and their test errors to within $1.53\times$, while their intrinsic-time fractions span $31\times$ (Figure~\ref{fig:eval-regime-discrimination}(a)).
By \eqref{eq:path-fraction} the spread comes from the number of steps and their consistency, neither of which the losses register.

\subsection{The slack of the classical envelopes}\label{sec:eval-slack-vs-bound}

Each classical envelope of \S\ref{sec:readings} is the identity with named nonnegative terms dropped, so on a real network its slack can be measured term by term at every prefix of a run.
The measurements cover ResNet-20 and a character-level language model trained by 10 optimizers, and a vision transformer trained by four of them, with three seeds each.
For SGD the linearized envelope exceeds the realized weighted regret by exactly the comparator information $I_n(u)$ still to be transported (Figure~\ref{fig:slack-vs-bound}(a)).
That term is $45\%$ of the envelope a fifth of the way through training on ResNet-20 and $75\%$ on the language model, and zero at the end.
On the language model the objective-level remainder, which convexity would make nonnegative, is negative over nearly the whole run for every optimizer without a momentum buffer, Schedule-Free included, so along those stretches the objective-level envelope does not apply (Figure~\ref{fig:slack-vs-bound}(b), Figure~\ref{fig:slack-terminal}(b)).
At the deepest training loss reached by every optimizer, the intrinsic-time fraction spans $29\times$ across the 10 optimizers on ResNet-20 and $18\times$ on the language model, while the achieved losses span less than $1.5\times$ (Figure~\ref{fig:slack-vs-bound}(c)).
The optimizers that keep a momentum buffer have the smallest fractions.

\begin{figure}[t]
\centering
\includegraphics[width=\linewidth]{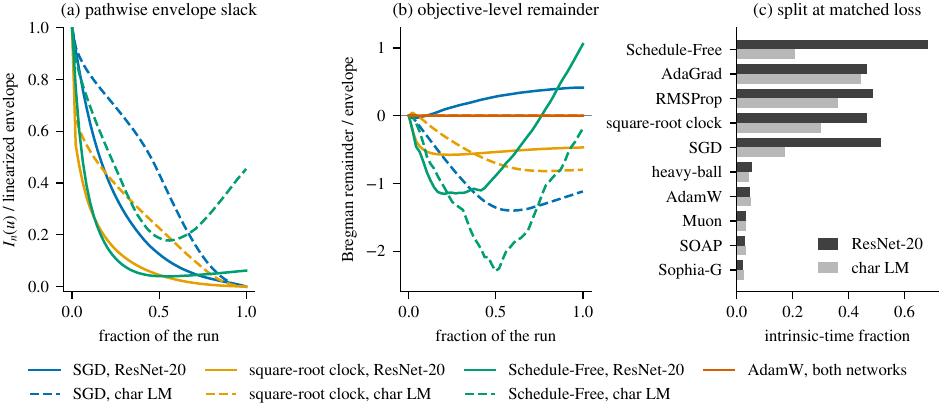}
\caption{\textbf{On real networks the information ledger attributes each classical envelope's slack to the terms dropped by its derivation.} It also separates optimizers that reach the same training loss: their intrinsic-time fractions differ by more than an order of magnitude.
(a) Comparator information $I_n(u)$ still to be transported at prefix $n$, as a share of the linearized envelope $I_0(u)+\sum_t\eta_t^2Q_t$.
(b) Cumulative Bregman remainder $\sum_t\eta_tB_F(u,\theta_t)$ over the same envelope; below zero the objective-level envelope does not apply. AdamW stays within $10^{-3}$ of zero on both networks.
(c) Intrinsic-time fraction at the deepest training loss reached by every optimizer.
Solid lines ResNet-20, dashed lines the character-level language model (char LM); each line is a median over seeds.}
\label{fig:slack-vs-bound}
\end{figure}

\subsection{Terms with a closed form}\label{sec:eval-closed-form}

Where one term of the information ledger has a closed form, the measurement tracks it: the intrinsic-time inflation follows the mini-batch factor $(N-b)/(b(N-1))$, and strict-saddle-point escape times track the leading-order time \eqref{eq:escape-time} (Figure~\ref{fig:single-term}).
Figure~\ref{fig:single-term-checks} shows five further measurements on controlled problems: the Bregman half-gain of the regret, the escape threshold set by the unstable clock, the intrinsic time under variance reduction and under a bounded-influence score, and the noise-energy entry under $\alpha$-stable noise.

\begin{figure}[t]
\centering
\includegraphics[width=\linewidth]{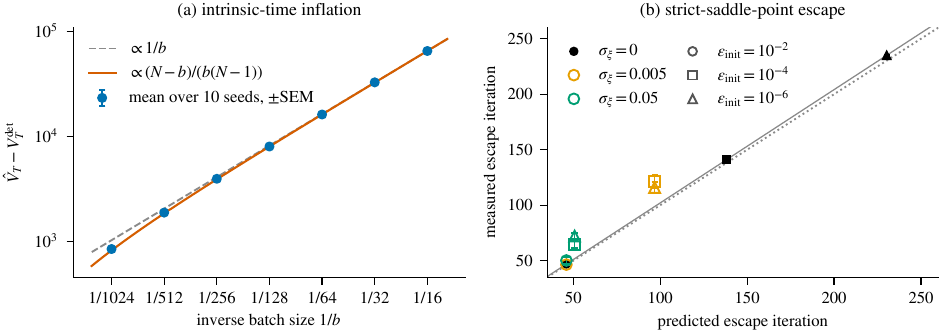}
\caption{\textbf{The intrinsic-time inflation and the strict-saddle-point escape times track their closed forms.}
(a) Realized inflation $\hat V_T-V_T^{\mathrm{det}}$ against inverse batch size on linear regression with mini-batches drawn without replacement from $N=5000$ examples; the fit proportional to $(N-b)/(b(N-1))$ solid and the plain $1/b$ law dashed.
(b) Mean escape iteration against the leading-order time \eqref{eq:escape-time} on the strict saddle point of $F=\half(x^2+y^2)-\half z^2$ ($\lambda_{\min}=-1$), started from $\theta_0=(1,1,\varepsilon_{\mathrm{init}})$; deterministic runs filled, their constant ratio solid, and equality dotted.
Here $\varepsilon_{\mathrm{eff}}$ in \eqref{eq:escape-time} is $\max(\varepsilon_{\mathrm{init}},\sigma_\xi\sqrt{\eta/2})$, so runs with the same $\varepsilon_{\mathrm{eff}}$ share a leading-order time.
Error bars are standard errors over seeds.}\label{fig:single-term}
\end{figure}

\subsection{The information ledger on real problems}\label{sec:eval-practice}

Drawn into the sampling covariance of a zeroth-order method, a query-only curvature probe recovers nearly all of the gain of an oracle that knows the Hessian diagonal, on a real objective with a wide curvature spread (Figure~\ref{fig:real-problems}(a)).
It removes $90\%$ of the initial suboptimality in a median $2{,}754$ queries.
Isotropic sampling does not reach that target within $150{,}000$ queries.
Applied to the step, the same preconditioner inflates the probe noise by the factor that the second moment \eqref{eq:zo-second} gives, so within that budget it closes at most $0.3\%$ of the initial suboptimality.
On Wide-ResNet-28-4 a larger batch gives both a sharper endpoint and a wider generalization gap, but across optimizers the relation turns negative: Lion ends flattest, fits its training set most tightly, and has the widest gap (Figure~\ref{fig:real-problems}(b)).
On ResNet-20 a larger batch also gives a sharper endpoint and a wider generalization gap (Appendix~\ref{app:cifar-forensics}).
A schedule that escapes whenever the return on intrinsic time stalls responds to mini-batch noise on ResNet-20: its stall signal is on for more than half of the steps of cosine annealing, the best-trained schedule, so acting on the signal trains worst of four schedules (Figure~\ref{fig:real-problems}(c)).

\begin{figure}[t]
\centering
\includegraphics[width=\linewidth]{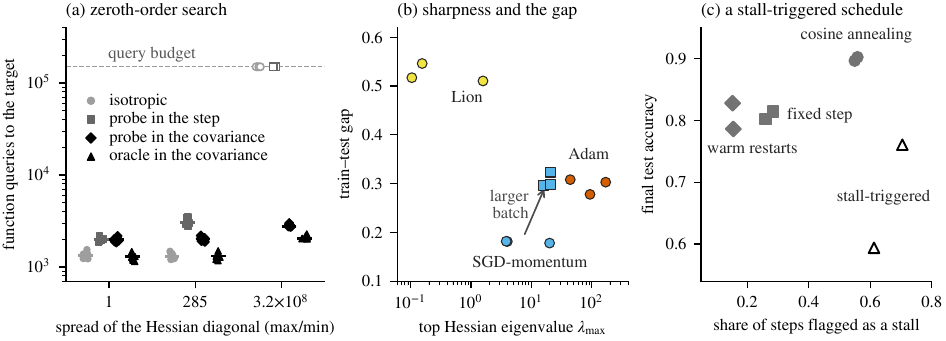}
\caption{\textbf{On a real objective with a wide curvature spread, a curvature probe drawn into the sampling covariance speeds zeroth-order search more than 50-fold, while applied to the step it makes no progress.}
(a) Function queries to close $90\%$ of the initial suboptimality of seven-class logistic regression on $49$ principal components of Covertype, at three spreads of the Hessian diagonal; small markers are seeds, bars are medians, and open markers on the dashed line did not reach the target within the budget.
(b) Top Hessian eigenvalue against the train--test gap at the end of training, Wide-ResNet-28-4 on CIFAR-10, one marker per run; circles batch $128$, squares batch $8192$. Sharpness tracks the generalization gap only along the batch axis.
(c) Final test accuracy of ResNet-20 on CIFAR-10 against the share of logged steps on which the pressure-cycled controller's stall signal is on, one marker per seed. The stall signal responds to mini-batch noise.}
\label{fig:real-problems}
\end{figure}

\subsection{Limits}\label{subsec:findings}

The intrinsic-time fraction separates trajectories by the number and consistency of their steps, but its rank correlation with the train--test gap across optimizers is $+0.41$ on ResNet-20 and $-0.18$ on the language model (Appendix~\ref{app:slack-protocol}).
The intrinsic time accumulated over the first fifth of a run, net of metric drift, ranks the optimizers' final training losses at a mean Spearman correlation of $+0.11$, positive on the language model and negative on ResNet-20, while the training loss at the same prefix reaches $+0.78$ (Appendix~\ref{app:slack-protocol}).
As an index of loss decay, the intrinsic clock gives the same forecasts as the iteration count wherever $V_T$ grows as a power of $T$. Along every real-data run measured, $V_T$ grows nearly as a power of $T$ (Appendix~\ref{sec:eval-vclock-fits}).
Across optimizers the bare Hessian statistics move under a function-preserving rescaling by as much as the distance between optimizers, while function-space curvature does not move.
On ResNet-20 the correlation of function-space curvature with the gap has the implied sign at batch $32$ and the opposite sign at batch $8192$, where $50$ epochs leave the runs far from fitting their training sets (Appendix~\ref{app:cifar-forensics}).
A terminal certificate built on a data-independent isotropic prior is vacuous at network scale unless the curvature spectrum is nearly flat across all parameters, since its minimum over the prior scale grows with the spread of that spectrum (\S\ref{sec:gen}).
On grokking and a vision transformer the pressure-cycled controller slows or degrades training, since its patience trigger fires only after the network fits its training set or while the training loss is still falling (Appendix~\ref{app:interventions}).

\section{Discussion}\label{sec:discussion}

The identity makes SGD analysis modular: any rate, any loss, any predictable side information, and any comparator produce one instance of the same accounting. It does not circumvent lower bounds: the $\Omega(1/\sqrt T)$ lower bound for convex stochastic optimization, lower bounds for gradient oracles, and minimax lower bounds for non-convex smooth optimization all still hold. The identity changes the structural position of known upper bounds and step-size rules: an upper bound becomes an identity plus explicit, quantifiable slack, and an adaptive step size becomes a controller driven by the intrinsic clock.

The work closest to this accounting was offered as a first step \cite{gergely2021informationtheoretic}, with its quantities referred to a population gradient whose meaning blurs across multiple passes, and with no account of why an SGD trajectory would hold those quantities small.
An exact accounting removes the first difficulty by construction: it records the realized update's own spending, and its expected intrinsic-time density refers the mini-batch term to the conditional gradient mean preferred in that work.
The second stands: the identity accounts exactly for the information accrued along a trajectory but supplies no mechanism that keeps that information small along the trajectories followed by SGD.
For SGD with label noise a partial mechanism is known: the algorithm converges to stationary points of the training loss plus an explicit penalty on sharp minimizers \cite{damian2021label}, and near a manifold of minimizers its limiting dynamics are a gradient flow of the Hessian trace along the manifold \cite{li2021what}.

The Gaussian-Bayes construction is local. It uses the observed stochastic gradient as a local score and gives an identity for the resulting mean recursion. It does not assert that an arbitrary network's data likelihood is globally Gaussian in parameters. The proper likelihood of Proposition~\ref{prop:proper-gaussian-likelihood} is local as well: it observes a constructed Gaussian variable at each round and reheats the covariance between rounds. The objective-level identity is exact only when the Bregman remainder is retained. Without the remainder, the linearized loss bounds $f_t(\theta_t)-f_t(\theta_*)$ under convexity (\S\ref{sec:convex}). For a nonconvex objective, smoothness with a Polyak--\L ojasiewicz condition bounds $F(\theta_t)-F^*$ through the intrinsic-time density (\S\ref{sec:saddle}). Under the linearized loss the posterior covariance stays at $\Sigma$. The quadratic composite loss of Appendix~\ref{app:flatness-full} lets the precision accumulate curvature, with heavier bookkeeping. The prior $\Sigma$ is a design parameter, so a reading is informative only for a $\Sigma_t$ whose perturbations are functionally meaningful. The linearized posterior $\calN(\theta_t,\Sigma)$ is specified by its mean, so the accounting costs no more than the update itself. Each empirical result built on the accounting holds within the boundary located for it in \S\ref{subsec:findings}.

SGD is the Gaussian-prior linearized instance of Bayes-rule updating. The identity accounts for the information accrued by this updating along each realized trajectory. The results treated here follow from the one-step identity by Gaussian bookkeeping (\S\ref{sec:readings}, \S\ref{sec:flatgen}), with intrinsic time as their common clock.

\appendix
\appendixletterskipI
\appendixtitle

\appendixcontents

\section{The proper Gaussian likelihood construction}\label{app:gab-likelihood}

The linear score $\exp(-\eta_t\ip{g_t}{\theta})$ is the weak-observation limit of a proper Gaussian likelihood: conditioning on a noisy Gaussian observation of $\theta$ and then re-inflating the posterior covariance to $\Sigma$ with independent process noise returns the SGD iterate.

\begin{proposition}[Gaussian likelihood with shrinkage and reheating]\label{prop:proper-gaussian-likelihood}
Fix $\chi_t\in(0,1)$. Let the prior be $q_t=\calN(\theta_t,\Sigma)$.
Consider the Gaussian observation model $Y_t\mid\theta\sim\calN\left(\theta,\tfrac{\chi_t}{1-\chi_t}\Sigma\right)$ and observe $y_t=\theta_t-\tfrac{\eta_t}{1-\chi_t}\Sigma g_t$.
The exact Bayesian posterior is $q_t^+=\calN(\theta_t-\eta_t\Sigma g_t,\chi_t\Sigma)$.
If after this update one adds independent process noise $\calN(0,(1-\chi_t)\Sigma)$, the resulting predictive distribution is $q_{t+1}=\calN(\theta_t-\eta_t\Sigma g_t,\Sigma)$ whose mean is the SGD iterate.
As $\chi_t\uparrow1$, the likelihood factor converges, up to constants independent of $\theta$, to $\exp(-\eta_t\ip{g_t}{\theta})$.
\end{proposition}
The proof is in Appendix~\ref{app:proofs-deferred}.

The finite update contracts the covariance and the process noise reheats it, while the linear-score update preserves the covariance directly; accounting for finite shrinkage and reheating makes the covariance change appear as the relative-entropy drift of \S\ref{sec:cov-drift}.

The linear-score calculus of \S\ref{sec:gab} extends to mixtures of Gaussians with a common covariance.

\begin{proposition}[Gaussian-mixture-prior SGD]\label{prop:mixture-prior}
Let $\pi=\sum_k w_k\,\calN(\mu_k,\Sigma)$ and reweight it by $\exp(-\eta\ip{g}{\theta})$. The posterior is again a mixture $\sum_k w_k'\,\calN(\mu_k-\eta\Sigma g,\Sigma)$ with every component mean shifted by the same SGD step $-\eta\Sigma g$ and re-weighted weights
\begin{equation}\label{eq:mixture-weights}
w_k'=\frac{w_k Z_k}{\sum_j w_j Z_j},\qquad \log Z_k=-\eta\ip{g}{\mu_k}+\tfrac12\eta^2 g^\top\Sigma g
\end{equation}
and the exact cumulative cross-entropy correction over a run is the mixture free energy $-\sum_t\log\!\bigl(\sum_k w_k^{(t)}Z_k^{(t)}\bigr)$, a log-sum-exp of the per-component free energies.
\end{proposition}
The proof is in Appendix~\ref{app:proofs-deferred}.

\section{Schedule constructions in full}\label{app:schedules-full}

\subsection{Matrix clocks and classical step-size rules}

For a predictable diagonal or full-matrix $\Sigma_t$ the intrinsic-time density $Q_t=\half g_t^\top\Sigma_tg_t$ already encodes the preconditioner, and every change of $\Sigma_t$ adds the covariance drift \eqref{eq:xi-gaussian}, so a preconditioner helps only when it lowers future intrinsic time or comparator information by more than its drift.
Parameter-free methods \citep{orabona-pal,cutkosky2018colt,zakaria2020lipschitz} attain regret that adapts to the unknown comparator without the bound $D=\|\theta_*\|$ that AdaGrad is tuned with. On the same clock, running the budgets $\Gamma_j=2^j$ in parallel and aggregating them is a dyadic multi-budget controller aimed at the same tuning problem.
The classical subgradient step $\alpha_t\asymp c/\sqrt t$ is \eqref{eq:adagrad-scalar} run with $V_t\approx t$, as if $\|g_t\|\approx1$.
Line search on $f_t$ is the pressure-target variant with $a_t=f_t(\theta_{t+1}(\alpha_t))$.

\subsection{Schedule-Free SGD: the tight convex rate as an identity projection}\label{sec:schedulefree-ledger}

The Schedule-Free family of \cite{aaron2024road} generates a base iterate $z_{t+1}=z_t-\gamma g_t$ at constant step $\gamma$, evaluates gradients at $y_t=(1-\beta)z_t+\beta x_t$ on the segment between $z_t$ and a weighted running average $x_t$, and returns $x_T$; $(\beta=0,w_t=1)$ recovers Polyak--Ruppert averaging, $(\beta=1,w_t=1)$ primal averaging, and $\beta_t=t/T$ linear-decay schedules \citep[\S 2]{aaron2024road}.
Identity \eqref{eq:weighted-ledger} applied to the base recursion at $\Sigma=I$, $\eta_t=\gamma$, summed over $t=1,\dots,T$ and divided by $\gamma T$, gives
\begin{equation}\label{eq:schedulefree-ledger}
\frac{1}{T}\sum_{t=1}^T\ip{g_t}{z_t-u}
=\frac{I_1(u)-I_{T+1}(u)}{\gamma T}+\frac{1}{T}\sum_{t=1}^T\gamma Q_t
\end{equation}
for every $\beta$.
The telescoping inequality in the proof of \cite[Theorem~2]{aaron2024road} drops nonnegative Bregman terms of $F$ among the evaluation point, the average, and the comparator.
At $\gamma=D/(G\sqrt T)$ with $D=\|z_1-u\|$ and $\sup_t\|g_t\|\le G$, dropping the terminal term $I_{T+1}(u)/(\gamma T)$ of \eqref{eq:schedulefree-ledger}, bounding $\gamma Q_t\le\gamma G^2/2$, and applying that inequality deliver the tight non-smooth convex rate $\E[F(x_T)-F(u)]\le DG/\sqrt T$ of \cite[Theorem~1]{aaron2024road}.

\subsection{Chebyshev step-size cycling: a closed-form schedule on quadratics}\label{sec:cheby-cyclic}

When the curvature interval is known, the information ledger \eqref{eq:weighted-ledger} admits a length-$T$ step-size cycle whose reciprocal steps are the Chebyshev nodes of that interval, a construction from numerical linear algebra \citep{manteuffel1977tchebychev,greenbaum1997iterative} treated beside the matching first-order lower bounds in \cite{nemirovski1995informationbased}.
Let $F(\theta)=\half(\theta-\theta_\star)^\top H(\theta-\theta_\star)$ with $H=U\Lambda U^\top$, $\lambda_i\in[\lambda_{\min},\lambda_{\max}]$, $\kappa=\lambda_{\max}/\lambda_{\min}$, and let $T_T$ be the Chebyshev polynomial of the first kind of degree $T$.
Full-gradient descent ($\Sigma=I$) runs the steps $\eta_t=1/r_t$ with
\begin{equation}\label{eq:cheby-roots}
r_t=\frac{\lambda_{\max}+\lambda_{\min}}{2}+\frac{\lambda_{\max}-\lambda_{\min}}{2}\cos\!\left(\frac{(2t-1)\pi}{2T}\right),\qquad t=1,\dots,T
\end{equation}
In eigencoordinates $z_t=U^\top(\theta_t-\theta_\star)$, write $g_t=\Lambda z_t$, $Q_t=\half g_t^\top g_t$, and $I_t(\theta_\star)=\half\|z_t\|^2$, and let $\rho_{\mathrm{cyc}}=(\sqrt\kappa-1)/(\sqrt\kappa+1)$.

\begin{proposition}[Chebyshev-cycle ledger identity on a quadratic]\label{prop:cheby-cyclic}
Under the schedule \eqref{eq:cheby-roots} on the quadratic above, the per-coordinate iterate satisfies the closed-form identity
\begin{equation}\label{eq:cheby-iterate}
z_{T,i}=z_{0,i}\cdot\frac{T_T\!\left(\frac{\lambda_{\max}+\lambda_{\min}-2\lambda_i}{\lambda_{\max}-\lambda_{\min}}\right)}{T_T\!\left(\frac{\lambda_{\max}+\lambda_{\min}}{\lambda_{\max}-\lambda_{\min}}\right)}
\end{equation}
and the cumulative objective-level ledger \eqref{eq:loss-path} evaluates against the comparator $u=\theta_\star$ to
\begin{equation}\label{eq:cheby-ledger}
\sum_{t=1}^T\eta_t\left[F(\theta_{t-1})-F(\theta_\star)+B_F(\theta_\star,\theta_{t-1})\right]
=I_0(\theta_\star)-I_T(\theta_\star)+\sum_{t=1}^T\eta_t^2Q_{t-1}
\end{equation}
where $B_F(\theta_\star,\theta_t)=\half\|z_t\|_\Lambda^2$ is the retained Bregman remainder of \eqref{eq:loss-path}.  The linearized form $\sum_t\eta_t\ip{g_{t-1}}{\theta_{t-1}-\theta_\star}=I_0(\theta_\star)-I_T(\theta_\star)+\sum_t\eta_t^2Q_{t-1}$ of \eqref{eq:weighted-ledger} contains no Bregman term, since the curvature gain is exactly the gap between the linearized regret and the realized excess loss. The geometric-mean per-round contraction of the distance $\|z_t\|=\sqrt{2I_t(\theta_\star)}$ equals $\rho_{\mathrm{cyc}}$:
\begin{equation}\label{eq:cheby-geomean}
\left(\frac{I_T(\theta_\star)}{I_0(\theta_\star)}\right)^{1/(2T)}\to\rho_{\mathrm{cyc}}\quad\text{as }T\to\infty
\end{equation}
whenever $z_0$ has a nonzero component on an eigencoordinate with $\lambda_i\in\{\lambda_{\min},\lambda_{\max}\}$.
\end{proposition}
The proof is in Appendix~\ref{app:proofs-deferred}.

At $\kappa=100$ the cycle reaches $\|z_T\|/\|z_0\|\le10^{-3}$ in every eigencoordinate after $T=38$ rounds, against $346$ rounds for the fixed step $2/(\lambda_{\max}+\lambda_{\min})$, a $9.1\times$ speedup, while its cumulative intrinsic time stays bounded in $T$.

\subsection{Pressure-cycled SGD: alternating descent and escape schedules}\label{sec:pressure-cycle}

On a nonconvex surface a second objective opposes the acquisition of comparator information: leaving a basin in which descent has trapped the iterate.
The information ledger expresses both as step-size problems in the same intrinsic-time units.

The descent schedule maximizes the gain in comparator information. Against a local reference the one-step gain \eqref{eq:weighted-ledger} is $I_t-I_{t+1}=\eta_tS_t-\eta_t^2Q_t$ with pressure $S_t=\ip{g_t}{\theta_t-r_t}$, maximized at
\begin{equation}\label{eq:greedy-descent}
\eta_t^{\mathrm{desc}}=\frac{S_t}{2Q_t}
\end{equation}
the pressure-target step \eqref{eq:polyak-info} at $\lambda_t=\half$, which transports $S_t^2/(4Q_t)$ per round.

The escape schedule maximizes the information injected along unstable directions. Near a strict saddle point, where negative curvature persists, or in a poor local basin, the pressure $S_t\to0$, so the descent schedule stalls.
In a poor local basin no direction is unstable; there the escape schedule's large step and an isotropic injection move the iterate across the barrier.
By the recursion \eqref{eq:saddle-stoch} an eigendirection of curvature $-\gamma$ amplifies the information distance by $(1+\eta\gamma)^2$ per step, so escape wants $\eta$ as large as the stable directions tolerate,
\begin{equation}\label{eq:escape-schedule}
\eta_t^{\mathrm{esc}}\asymp\frac{c}{\lambda_{\max}}
\end{equation}
together with a thermal injection of variance recorded by the unstable clock $V_T^-$ of \eqref{eq:V-minus}.
This high-temperature schedule incurs intrinsic time deliberately.

A controller switches between the two schedules using the per-round \emph{return on intrinsic time}
\begin{equation}\label{eq:rho-trigger}
\rho_t=\frac{f(\theta_t)-f(\theta_{t+1})}{\eta_t^2Q_t}
\end{equation}
the realized objective decrease per unit of the intrinsic-time cost $\eta_t^2Q_t$ of \eqref{eq:weighted-ledger}.
The controller treats a run of small returns as a stall.

Pressure-cycled SGD tracks the iterate $\theta_t$, the intrinsic clock $V_t$, the unstable clock $V_t^-$, the best objective value $f_{\mathrm{best}}$ found so far, and a mode that is either \textsc{descent} or \textsc{escape}.
It alternates between the two modes.
\begin{enumerate}[leftmargin=1.5em,itemsep=1pt,topsep=2pt]
\item \textsc{descent}: step with $\eta_t^{\mathrm{desc}}=S_t/(2Q_t)$ \eqref{eq:greedy-descent} (in practice the AdaGrad form of \S\ref{sec:schedules} on $V_t$). If $\rho_t$ \eqref{eq:rho-trigger} stays below a threshold for $K$ rounds, switch to \textsc{escape}.
\item \textsc{escape}: step with $\eta_t^{\mathrm{esc}}\asymp c/\lambda_{\max}$ \eqref{eq:escape-schedule} plus a thermal injection, isotropic or along the unstable subspace ($V_t^-$). When a strictly lower basin value $f<f_{\mathrm{best}}$ is found, record it and switch to \textsc{descent}.
\end{enumerate}
It returns the iterate at which $f_{\mathrm{best}}$ was recorded.
Each mode maximizes a quantity the identities make explicit: \textsc{descent} the information gain in \eqref{eq:weighted-ledger}, and \textsc{escape} the injected information in \eqref{eq:saddle-stoch}.

Cyclic learning-rate schedules \citep{smith2017cyclical} and warm restarts \citep{loshchilov2017sgdr} also alternate step-size phases, but on a preset clock; here the alternation is triggered by the measured ratio \eqref{eq:rho-trigger}.
Neither schedule is individually new: \eqref{eq:greedy-descent} is the Polyak-step reading of \S\ref{sec:schedules} and \eqref{eq:escape-schedule} the high-temperature reading of \S\ref{sec:saddle}.
On a synthetic two-basin potential the controller returns an iterate within $10^{-3}$ of the deep minimum on every seed, where no fixed-step or constant-temperature baseline reaches it on any seed.
Its escape keeps running once the minimum is found, since no strictly lower value exists to end it, which is why the output is the recorded best iterate.
The numerator of the switch statistic \eqref{eq:rho-trigger}, which sets its sign, is a difference of objective values.
Estimated from consecutive mini-batch losses on a real network its sign is noise, and on a full-batch run the loss stops improving at convergence just as at a stall, so the trigger separates a stall neither from mini-batch noise nor from convergence (Appendix~\ref{app:interventions}).

\section{Zeroth-order constructions and their measured reach}\label{app:zo-full}

On a quadratic $f=\half x^\top Hx$ with $\Sigma=I$ the expected zeroth-order step is $-\eta_tCHx$, so the sampling covariance $C$ preconditions the expected dynamics.
Isotropic sampling contracts the slowest direction by $1-\eta_t\lambda_{\min}$ per step, while $C=H^{-1}$ contracts every direction by the same factor $1-\eta_t$, whatever the condition number $\kappa$.
Covariance-adapting evolution strategies \citep{hansen2016cma} approach this anisotropy heuristically; on convex quadratics the adapted covariance approximates the inverse Hessian up to a constant factor.
The first-order route to it fails: the per-coordinate estimate $\hat g_{t,i}=\ip{g_t}{\nu_t}\nu_{t,i}$ is a noisy rank-one projection, so an AdaGrad accumulator built from it does not identify the curvature (\S\ref{sec:eval-zo-readings}).

\subsection{A query-only curvature-probed sampler}\label{subsec:zo-curvature}

The query-only sampler estimates curvature from second differences of function values, without using the noisy first-order score.

\begin{proposition}[Second-difference curvature]\label{prop:zo-second-diff}
For twice-differentiable $f$ and any direction $\nu$,
\begin{equation}\label{eq:zo-second-diff}
\frac{f(\theta+\mu \nu)+f(\theta-\mu \nu)-2f(\theta)}{\mu^2}\xrightarrow[\mu\to0]{}\nu^\top\nabla^2f(\theta)\,\nu
\end{equation}
exact for quadratic $f$ at every $\mu$. A coordinate probe $\nu=e_i$ recovers $\nabla^2f(\theta)_{ii}$, so a sweep over the $d$ coordinates recovers $\mathrm{diag}(\nabla^2f)$ in $2d+1$ function queries. For an isotropic random probe $\nu\sim\calN(0,I)$, writing $H=\nabla^2f(\theta)$,
\begin{equation}\label{eq:zo-hutch}
\E\big[(\nu^\top H\nu)\,\nu_i^2\big]=\tr(H)+2H_{ii}
\end{equation}
so $H_{ii}=\half\big(\E[(\nu^\top H\nu)\nu_i^2]-\tr(H)\big)$ with $\tr(H)=\tfrac{1}{d+2}\sum_i\E[(\nu^\top H\nu)\nu_i^2]$.
\end{proposition}
The proof is in Appendix~\ref{app:proofs-deferred}.

The random estimator \eqref{eq:zo-hutch} retains the common $\tr(H)$ baseline in its variance, so for ill-conditioned $H$ it needs on the order of $(\tr H/H_{ii})^2$ probes to resolve a small $H_{ii}$, where the coordinate sweep is exact in $2d+1$ queries.
The sampler therefore uses the coordinate sweep and draws its probe with covariance equal to the damped preconditioner $P=\mathrm{diag}(|H|+\lambda)^{-1}$.
Scaling an isotropic probe's step by $P$ gives the same expected step \eqref{eq:zo-mean} and a different second moment \eqref{eq:zo-second}, and on a real objective with a wide curvature spread only the covariance form makes progress (\S\ref{sec:eval-zo-readings}).

Curvature-probed zeroth-order SGD alternates a curvature sweep with probe steps.
\begin{enumerate}[leftmargin=1.5em,itemsep=1pt,topsep=2pt]
\item Every $K$ steps, sweep the coordinates with \eqref{eq:zo-second-diff} to estimate $\mathrm{diag}(H)$ in $2d+1$ queries, and set $P=\mathrm{diag}(|H|+\lambda)^{-1}$.
\item At each step, draw a probe $\nu\sim\calN(0,P)$, form the forward-difference score $\hat g$ of \eqref{eq:zo-score}, and update $\theta\leftarrow\theta-\eta_t\hat g$.
\end{enumerate}
The $O(d)$ sweep is amortized over the $K$ descent steps.
By \eqref{eq:zo-mean} the expected step is $-\eta_t Pg$.

The off-diagonal curvature comes from the same second differences: for $i\ne j$,
\begin{equation}\label{eq:zo-offdiag}
\frac{f(\theta+\delta(e_i+e_j))+f(\theta-\delta(e_i+e_j))-2f(\theta)}{\delta^2}\xrightarrow[\delta\to0]{}H_{ii}+H_{jj}+2H_{ij}
\end{equation}
so $H_{ij}$ follows from \eqref{eq:zo-offdiag} and the diagonal \eqref{eq:zo-second-diff}, recovering the full symmetric $H$ in $d^2+d+1$ queries.
The preconditioner $P=(|H|+\lambda I)^{-1}$, with $|H|$ the absolute value on the Hessian's eigenvalues, is then a query-only realization of the saddle-free Newton method \citep{dauphin2014saddle}.
It stays positive definite when $H$ is indefinite, so along a negative-curvature direction it steps away from a saddle point where a Newton step $H^{-1}$ would converge to it.

\subsection{Measured reach of the zeroth-order samplers}\label{sec:eval-zo-readings}

On $f=\half x^\top Hx$ with $H=\mathrm{diag}(\mathrm{linspace}(1,\kappa,20))$ and $10$ seeds, sampling from the oracle covariance $C=\mathrm{diag}(H)^{-1}$ reaches $f<10^{-3}$ in fewer than $400$ function-evaluation pairs at every $\kappa$ from $10$ to $1000$, while isotropic sampling needs more as $\kappa$ grows (Figure~\ref{fig:zo-queries}(a)).
An AdaGrad accumulator built from the probe scores, $C_t\propto\mathrm{diag}(\sum_{s<t}\hat g_s\hat g_s^\top+\varepsilon I)^{-1/2}$, needs about as many pairs as isotropic sampling.

Panel (b) counts single queries for the step form $\theta\mapsto\theta-\eta P\hat g$, which has the covariance form's expected step.
The quartic bowl is $f=\half\sum_i\lambda_ix_i^2+\tfrac{c}{4}\sum_ix_i^4$ in $d=12$, with $c=0.5$ and $\lambda_i$ spaced evenly from $1$ to $\kappa$.
The Rosenbrock valley is $f=\sum_{i<d}\bigl[100(x_{i+1}-x_i^2)^2+(1-x_i)^2\bigr]$ in $d=8$.
Each sampler's step size is tuned separately, with $10$ seeds on the quadratics and $8$ on the quartic bowl and the Rosenbrock valley.
Run once at initialization and counted in full, the coordinate sweep leaves the diagonal probe $6$--$12\%$ above the oracle on axis-aligned quadratics.
Refreshing the sweep every $10$ steps raises its count to $3.1$--$3.3$ times the oracle's.
The random estimator \eqref{eq:zo-hutch} needs $2{,}276$ and $10{,}463$ queries at $\kappa=10$ and $100$, more than isotropic sampling.

On rotated quadratics the diagonal probe needs $1.4$ times the queries of isotropic sampling at $\kappa=10$ and gains $1.9\times$ at $\kappa=100$.
The full-matrix probe, whose sweep costs $d^2+d+1=273$ queries at $d=16$, needs $1.3$--$1.5$ times the queries of its oracle $(|H|+\lambda I)^{-1}$ at every $\kappa$.

On the quartic bowl the sweep is refreshed along the path and counted.
The diagonal probe's advantage grows with $\kappa$ as on a quadratic.
At $\kappa=30$ the step tuned for both probes diverges on three of eight seeds.
The tuning rule, fewest median queries over three tuning seeds, admits a step one grid point below the size at which every seed diverges.
Admitting only steps that reach the target on every tuning seed selects the next smaller step for both probes and leaves the other samplers unchanged.
Every seed then reaches the target, in a median $786$ queries for the diagonal probe and $1{,}842$ for the full-matrix probe against $2{,}046$ for isotropic sampling.
On the Rosenbrock valley isotropic sampling diverges at every step size of the grid.
The full-matrix probe reaches $f<1$ on $5$ of $8$ seeds in a median $3{,}649$ queries, and under the stricter tuning on every seed, at a step three times smaller, in a median $6{,}240$.
On one seed the diagonal probe stops at a non-global local minimum, where $f=3.99$.
A step three times smaller avoids it on every seed at three times the queries.

\begin{figure}[ht]
\centering
\includegraphics[width=\linewidth]{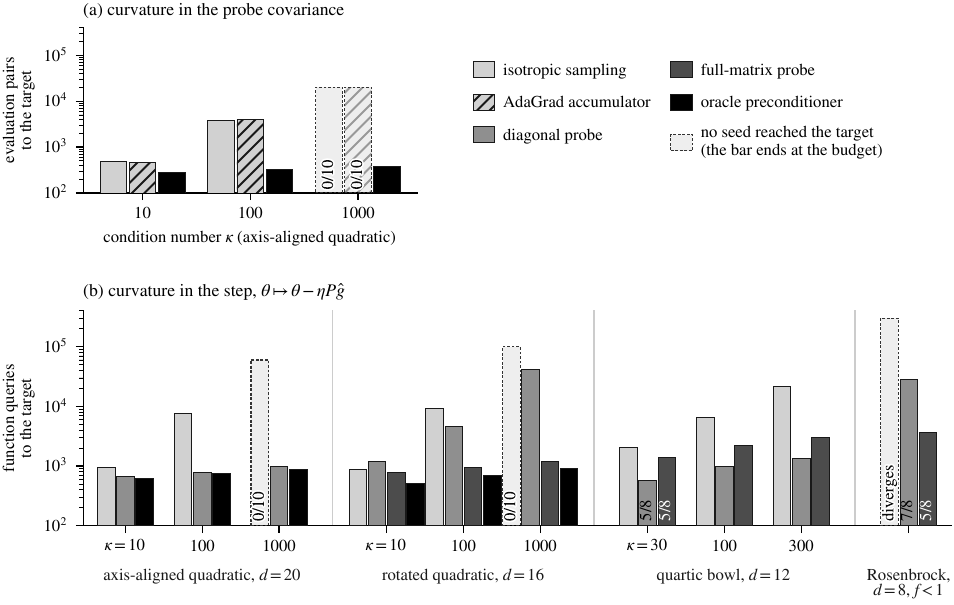}
\caption{\textbf{Queries to the target grow far more slowly with the condition number when the probe captures the curvature.}
(a) Probes on the axis-aligned quadratic drawn from the isotropic covariance, from an AdaGrad accumulator and from the oracle covariance $\mathrm{diag}(H)^{-1}$. The accumulator does not identify the curvature.
(b) Probes that apply the preconditioner to the step. The oracle is $\mathrm{diag}(H)^{-1}$ on the axis-aligned quadratics and $(|H|+\lambda I)^{-1}$ on the rotated ones. The probes' curvature sweeps count toward their queries. On rotated quadratics only the full-matrix probe captures the curvature.
Bars are medians over seeds of the function-evaluation pairs in (a) and the function queries in (b) to $f<10^{-3}$, or to $f<1$ on the Rosenbrock valley.
A pale dashed bar marks a sampler that reached the target on no seed and ends at the budget.
A fraction at the foot of a bar is the share of seeds that reached the target, where that share is below one.}
\label{fig:zo-queries}
\end{figure}

The real objective of Figure~\ref{fig:real-problems}(a) is seven-class logistic regression with an $\ell_2$ floor of $10^{-3}$ on $2{,}000$ Covertype examples written in their $49$ principal-component scores, so its Hessian at the start is block-diagonal in the features.
Three scalings of the same components set the spread of its diagonal: unit variance ($\max/\min$ ratio $1$), the component variances of the standardized features ($285$), and those of the raw features ($3.2\times10^8$).
Every query counts, the coordinate sweep included.
The target removes $90\%$ of the initial suboptimality.
Step sizes are tuned on seeds disjoint from the eight reported.
On the raw components isotropic sampling does not reach the target within $150{,}000$ queries.
The covariance-form probe reaches it in a median $2{,}754$ queries and the exact diagonal oracle in $2{,}039$, speed-ups of more than $54\times$ and $73\times$ that recover at least $93\%$ of the oracle's gain on a log scale (Figure~\ref{fig:real-problems}(a)).
Scaling the step of an isotropic probe by the same preconditioner diverges or stalls at every step size tried, its median best loss staying above $1.944$ from $1.946$ at the start.
The two forms share the mean $-\eta Pg$ of \eqref{eq:zo-mean} and differ in the second moment \eqref{eq:zo-second}.
On a quadratic with Hessian $H$ the probe noise enters the loss as $\half\eta^2[(g^\top Pg)\tr(HP)+2g^\top PHPg]$ in the covariance form and as $\half\eta^2[\|g\|^2\tr(PHP)+2g^\top PHPg]$ in the step form.
The ratio of their leading terms, $\rho=\|g\|^2\tr(PHP)/((g^\top Pg)\tr(HP))$, is $3.9\times10^6$ at the start on the raw components, $2.3$ on the standardized ones, and $1$ at unit variance.
Accordingly the step form needs $1.5$ to $1.8$ times the queries of the covariance form on the standardized components and coincides with it at unit variance.
Where the spread does not slow isotropic sampling, preconditioning has nothing to recover: the oracle gains $1.04\times$ at unit variance and $0.99\times$ on the standardized components, where isotropic sampling needs $1{,}301$ queries.
There the probe spends $700$ queries on its sweep.

At fine-tuning scale only a per-group block-scalar covariance $C=\mathrm{diag}(s_g^2)$ is affordable.
The last two blocks and a freshly initialized head of a pretrained encoder ($14.8$M parameters, $18$ groups) are fine-tuned with forward-pass-only zeroth-order SGD \citep{malladi2023mezo} for $700$ queries on $5$ seeds, every allocation sharing one learning rate and one total probe energy.
Scaling the probe with the square root of each group's curvature, $s_g\propto\sqrt{H_g}$, ends at the lowest validation cross-entropy of four allocations and $s_g\propto1/\sqrt{H_g}$ at the highest (Figure~\ref{fig:zo-allocation}).
The expected step gives group $g$ an effective rate proportional to $s_g^2$, so this design does not separate the allocation from the per-group step size it sets.

\begin{figure}[ht]
\centering
\includegraphics[width=\linewidth]{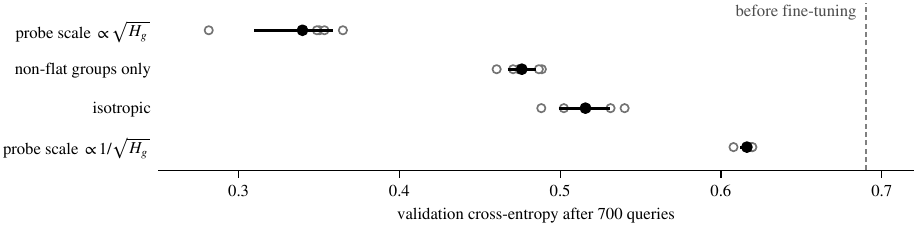}
\caption{\textbf{At fine-tuning scale the probe allocation that follows the curvature ends at the lowest validation loss.}
Validation cross-entropy of the fine-tuned encoder after $700$ queries for four per-group probe scales that share one learning rate and one total probe energy.
Open circles are seeds, filled circles their mean, and bars the $95\%$ percentile-bootstrap interval of the mean.
The dashed line is the cross-entropy before fine-tuning.}
\label{fig:zo-allocation}
\end{figure}

\subsection{Any sampling law via the cumulant information ledger}

The mean identity \eqref{eq:zo-mean} uses only the covariance of $\calP$, so it holds for Gaussian, Rademacher, uniform-on-sphere, and Laplace probes alike.
Under a non-Gaussian probe the normalizer of the adaptive-Bayes update is the cumulant-generating function $\Lambda_t$ of \S\ref{subsec:add-nongaussian} (Proposition~\ref{prop:add-cgf-noise}), of which the Gaussian second moment \eqref{eq:zo-second} is the second-order term.
Another probe law keeps the mean $Cg_t$ and changes only the fourth-moment structure of \eqref{eq:zo-second}, which is finite whenever $\calP$ has finite fourth moments.
\section{Information complexity of loss surfaces}\label{app:surface-full}

For a realized trajectory, the weighted information complexity relative to a comparator $u$ is
\begin{equation}\label{eq:path-complexity}
\mathsf C_T(u)
=I_1(u)-I_{T+1}(u)+\sum_{t=1}^T\eta_t^2Q_t-\sum_{t=1}^T\eta_tB_{f_t}(u,\theta_t)
\end{equation}
which by Theorem~\ref{thm:loss-identity} equals $\sum_{t=1}^T\eta_t[f_t(\theta_t)-f_t(u)]$ and separates the prior description cost $I_1(u)$, the intrinsic-time energy, and the signed curvature gain.
The clock gives measurable path-difficulty notions: the worst case $\sup_{\mathrm{traj}}V_T$ over SGD trajectories, the average $\E_{\mathrm{init}}[V_T]$ over random initializations, and the least intrinsic time $\min_{\mathrm{traj}\to s}V_T$ of a trajectory reaching a given point $s$.
Optimizing over admissible policies and terminal posteriors gives a surface-level complexity: for a prior $\pi$, tolerance $\varepsilon$, horizon $T$, and $B_t=B_F(\theta_\star,\theta_t)$,
\begin{equation}\label{eq:surface-def}
\calK_T(\varepsilon)
=\inf\left\{\sum_{t=1}^T\eta_t^2Q_t-\sum_{t=1}^T\eta_tB_t+\KL(q_T\|\pi):
\E_{q_T}F-F_\star\le\varepsilon\right\}
\end{equation}

\begin{proposition}[Quadratic-objective surface complexity]\label{prop:surface-quadratic}
Let $F(\theta)=\half(\theta-\theta_\star)^\top H(\theta-\theta_\star)$ with $H=U\Lambda U^\top\succ0$, let $\pi=\calN(0,\tau^2I)$, and take the policy class of fixed-rate full-gradient descent with $\Sigma=I$ and Gaussian terminal posteriors. For every $\beta>0$ with $\half\sum_{i=1}^d\lambda_i/(\beta\lambda_i+\tau^{-2})<\varepsilon$,
\[
\limsup_{T\to\infty}\calK_T(\varepsilon)\le \half\sum_{i=1}^d\log\!\left(1+\beta\tau^2\lambda_i\right)+\frac{\|\theta_\star\|^2}{2\tau^2}+R_{\mathrm{path}},
\qquad
R_{\mathrm{path}}=\half\sum_{i=1}^d\frac{\eta_\star\lambda_i-1}{2-\eta_\star\lambda_i}\,z_{0,i}^2
\]
where $\eta_\star=2/(\lambda_{\max}+\lambda_{\min})$ and $z_0=U^\top(\theta_0-\theta_\star)$.
\end{proposition}
The proof is in Appendix~\ref{app:proofs-deferred}.

The first term is the eigenvalue-by-eigenvalue posterior-information cost, the log-determinant that reappears as the flatness measure of \S\ref{sec:flatness}.
The second is the prior-distance penalty.
The feasibility condition forces $\beta$ to grow as $\varepsilon\downarrow0$, roughly as $d/(2\varepsilon)$, so certifying a tighter tolerance costs of order $\tfrac d2\log(1/\varepsilon)$.
The remainder $R_{\mathrm{path}}$ is signed and finite.
Since $\eta_\star$ depends on $\lambda_{\max}$, the coefficient of $z_{0,\max}^2$ in it is $(\kappa-1)/4$, growing linearly in the condition number where the log-determinant grows only logarithmically.

\section{Flatness as terminal posterior information in full}\label{app:flatness-full}

\subsection{Basin evidence and posterior contraction}

For a basin $\calB$, the evidence $\mathcal E_\beta(\calB)=-\beta^{-1}\log\int_\calB e^{-\beta F(\theta)}\pi(d\theta)$ has, near $\theta_\star$, the Laplace approximation
\begin{equation}\label{eq:laplace}
\mathcal E_\beta(\calB)
\approx F(\theta_\star)+\frac{1}{2\beta}\log\det H
-\frac{d}{2\beta}\log\frac{2\pi}{\beta}-\frac1\beta\log\pi(\theta_\star)
\end{equation}
Under the linearized update the posterior covariance stays at $\Sigma$.
A quadratic composite loss $c_t(\theta)=\ip{g_t}{\theta}+\half(\theta-\theta_t)^\top H_t(\theta-\theta_t)$, as in cubic regularization and trust-region methods, gives a second-order identity.
Its posterior precision accumulates the curvature,
\[
\Sigma_{t+1}^{-1}=\Sigma_t^{-1}+\alpha H_t,
\qquad\text{equivalently}\qquad
\Sigma_{t+1}=(I+\alpha\,\Sigma_t H_t)^{-1}\Sigma_t
\]
so after $T$ rounds near a minimum with $H_t\approx H$ the posterior is $\Sigma_T\approx(\alpha TH)^{-1}$, narrow along sharp directions and broad along flat ones.
This is the Laplace form of the Bayesian identification of sharpness with a narrow posterior \citep{mackay1992practical,geoffrey1993keeping}.

\subsection{Coordinate dependence and practical estimation}\label{rem:logdet-practical}

Under a change of parameters $\theta'$ with Jacobian $J=\partial\theta'/\partial\theta$, the Fisher matrix transforms as $\mathcal F'=J^{-\top}\mathcal FJ^{-1}$ at every point, and the Hessian transforms the same way at a critical point, or everywhere when the map is linear.
Their log-determinants therefore shift by $-2\log|\det J|$, a shift that differs between basins unless the map is linear.
Nor is $\log\det(I+\beta\tau^2H)$ invariant.
The relative entropy $\KL(q\|\pi)$ is invariant when posterior and prior are transformed together.
At a critical point the spectrum of $\mathcal F^{-1}H$ is also invariant, since it transforms by similarity.
The trace and the spectral radius of the Hessian have neither property \citep{laurent2017sharp}.

On a deep network only a truncated spectrum is computable, typically the top-$K$ Ritz values returned by Lanczos with $K\ll d$, cross-checked with a Hutchinson trace.
Truncation adds failure modes of its own.
A function-preserving rescaling moves which eigenvalues fall in the top-$K$ subspace.
Clamping the negative-curvature tail to a floor $\epsilon$ adds $-\log\epsilon$ per floored eigenvalue, which can dominate the estimate as $K$ grows.
The long tail of small positive eigenvalues is undercounted, the Hutchinson trace and the sum of top-$K$ Ritz values disagreeing by a residual that does not shrink reliably with $K$.
The Fisher matrix is positive semidefinite, so the floor fires only on its near-zero Ritz values, far less often than on the Hessian's.
The Fisher's softer spectrum tightens the trace residual, but its log-determinant shifts under reparameterization as the Hessian's does.

A calibration on one trained ResNet-20 (SGD, batch $128$, $25$ epochs) measures the shift $|\Delta|$ of the truncated log-determinant under the function-preserving rescaling of \cite{laurent2017sharp} at $\alpha=4$ whose full-space prediction is $|{-2\log|\det J|}|\approx6300$ ($32$ batch-normalization parameters scaled by $\alpha$, $2304$ downstream convolution parameters by $1/\alpha$):

\begin{center}
\begin{tabular}{lcccc}
\toprule
kernel & $K=32$ & $K=128$ & $K=512$ & K-share predicted \\
\midrule
Hessian (top-$K$ Lanczos)    & 22 & 56  & 182  & 0.7,\,3.0,\,11.8 \\
Fisher (top-$K$ Lanczos, FVP) & 12 & 44  & 210  & 0.7,\,3.0,\,11.8 \\
\bottomrule
\end{tabular}
\end{center}

The ratio of $|\Delta|$ to the $K$-share of the full prediction falls from about $31$ to $15$ for the Hessian as $K$ grows and stays between $15$ and $18$ for the Fisher. Growing $K$ moves $|\Delta|$ toward the full prediction without reaching a regime of invariance.
The floor fires on $4$, $28$, and $178$ Hessian Ritz values at $K=32$, $128$, and $512$, dominating the Hessian log-determinant at $K=512$, against $0$, $1$, and $4$ for the Fisher, whose Hutchinson-versus-Ritz trace residual at $K=512$ is $5.0\%$ against $8.4\%$.

The eigenvalues of $\mathcal F^{-1}H$ at a critical point, the natural-gradient norm $(\nabla L)^\top\mathcal F^{-1}\nabla L$, and the function-space curvature are coordinate-free without further accounting. The function-space curvature is the input-averaged Hessian of the loss in the network outputs ($H_L=\diag(p)-pp^\top$ for softmax cross-entropy in logit space), which depends on the function alone.
All three are measured in Appendix~\ref{app:cifar-forensics}.

\subsection{Sharpness-aware minimization as an information-ledger entry with path-averaged Hessian}\label{sec:sam-ledger}

Sharpness-aware minimization \citep{pierre2021sharpness} replaces the gradient step at $\theta_t$ by a step at the worst-case neighbor $\theta_t+\rho_tg_t/\|g_t\|$ on a Euclidean $\rho_t$-ball.
For each round fix $\rho_t\ge0$, write $\hat g_t=g_t/\|g_t\|$ (with $\hat g_t=0$ if $g_t=0$), and define the SAM ascent gradient and the path-averaged Hessian of the round's loss $f_t$
\begin{equation}\label{eq:sam-taylor}
\check g_t=\nabla f_t\!\left(\theta_t+\rho_t\hat g_t\right),
\qquad
\bar H_t=\int_0^1\nabla^2 f_t\!\left(\theta_t+s\rho_t\hat g_t\right)ds
\end{equation}
The fundamental theorem of calculus along the ascent ray gives the identity $\check g_t=g_t+\rho_t\bar H_t\hat g_t$. The SAM update is the preconditioned SGD update run with the score $\check g_t$,
\begin{equation}\label{eq:sam-update}
\theta_{t+1}=\theta_t-\eta_t\Sigma\check g_t
\end{equation}

\begin{corollary}[Exact SAM ledger]\label{cor:sam-ledger}
For every $\rho_t\ge0$, every $\eta_t>0$, every $\Sigma\succ0$, and every comparator $u\in\R^d$, the SAM update \eqref{eq:sam-update} satisfies the pathwise identity
\begin{equation}\label{eq:sam-ledger}
\eta_t\ip{\check g_t}{\theta_t-u}
=I_t(u)-I_{t+1}(u)+\eta_t^2 Q_t^{\mathrm{SAM}}
\end{equation}
where the SAM intrinsic-time density admits the closed form
\begin{equation}\label{eq:sam-Q}
Q_t^{\mathrm{SAM}}
=Q_t+\frac{\rho_t}{\|g_t\|}\,g_t^\top\Sigma\bar H_tg_t
+\frac{\rho_t^2}{2\|g_t\|^2}\,g_t^\top\bar H_t\Sigma\bar H_tg_t,
\qquad Q_t=\half g_t^\top\Sigma g_t
\end{equation}
and $\bar H_t$ is the path-averaged Hessian along the ascent ray defined in \eqref{eq:sam-taylor}. \end{corollary}
The proof is in Appendix~\ref{app:proofs-deferred}.

The identity holds at every order in $\rho_t$, the path-averaged Hessian absorbing the higher-order content of $f_t$ along $[\theta_t,\theta_t+\rho_t\hat g_t]$.
The cross-term $\rho_t\|g_t\|^{-1}g_t^\top\Sigma\bar H_tg_t$, the dominant correction for small $\rho_t$, is the per-round directional sharpness along the realized step that SAM lowers, while the terminal log-determinant $\log\det H$ records the flatness that results.

Fisher SAM and ASAM change only the ascent direction. The proof of Corollary~\ref{cor:sam-ledger} uses only smoothness of $\nabla f_t$ along the ascent ray, so any ascent direction substitutes into \eqref{eq:sam-Q}.
Replacing the Euclidean ball by the Fisher ellipsoid $\{\epsilon:\epsilon^\top\calF(\theta)\epsilon\le\gamma^2\}$, the small-radius form of a relative-entropy ball \citep{amari1998neural}, gives the natural-gradient ascent direction
\begin{equation}\label{eq:fsam-direction}
\hat g_t^{\mathrm{FSAM}}
=\frac{\calF(\theta_t)^{-1}g_t}{\sqrt{g_t^\top\calF(\theta_t)^{-1}g_t}}
=\frac{\calF(\theta_t)^{-1}g_t}{\|g_t\|_{\calF^{-1}}}
\end{equation}
in place of $g_t/\|g_t\|$, with the descent step and its preconditioner $\Sigma$ unchanged \citep{kim2022fishersam}.
Substituting \eqref{eq:fsam-direction} preserves \eqref{eq:sam-ledger}, with
\begin{equation}\label{eq:fsam-Q}
Q_t^{\mathrm{FSAM}}=Q_t
+\frac{\rho_t}{\|g_t\|_{\calF^{-1}}}\,g_t^\top\Sigma\bar H_t^{\mathrm{FSAM}}\calF^{-1}g_t
+\frac{\rho_t^2}{2\|g_t\|_{\calF^{-1}}^2}\,g_t^\top\calF^{-1}\bar H_t^{\mathrm{FSAM}}\Sigma\bar H_t^{\mathrm{FSAM}}\calF^{-1}g_t
\end{equation}
where $\bar H_t^{\mathrm{FSAM}}=\int_0^1\nabla^2 f_t(\theta_t+s\rho_t\hat g_t^{\mathrm{FSAM}})\,ds$.
ASAM \citep{kwon2021asam} is the same algebra with $\calF(\theta_t)$ replaced by $\diag(\theta_t)^{-2}$.
In synthetic convex logistic regression at matched loss and condition numbers up to $100$, Fisher SAM's terminal log-determinant is comparable to or larger than SAM's.
Fisher SAM's ascent has no first-order preference for flat directions: where the Fisher matrix equals the Hessian, the natural-gradient ascent raises a near-quadratic loss by $\rho\sqrt{2(F-F^\star)}$ to first order, which depends on the loss alone, while SAM's Euclidean ascent raises it least along the flattest directions.
The documented Fisher-SAM generalization advantage \citep{kim2022fishersam} is therefore not mediated by the terminal log-determinant alone.

\section{The full noise-and-momentum information ledger}\label{sec:addendum-langevin-momentum}

Added noise, non-Gaussian noise, and momentum, the variants used in optimization and sampling, do not break the accounting; each adds identifiable entries (Table~\ref{tab:noise-momentum}).
Gaussian Langevin noise has an optimization reading, as injected energy in the point-comparator ledger, and a sampling reading, as a heat channel that contracts distinguishability.

\begin{table}[ht]
\centering
\small
\renewcommand{\arraystretch}{1.25}
\begin{tabularx}{\textwidth}{@{}>{\raggedright\arraybackslash}p{0.34\textwidth}>{\raggedright\arraybackslash}X@{}}
\toprule
\textbf{Noise or momentum class} & \textbf{How it enters the information ledger} \\
\midrule
Mini-batch score noise, post-step Gaussian, Langevin thermal & one additive quadratic entry each. The expected ledger depends on the perturbation only through its covariance, but the three attach at different places and scale differently in $\eta_t$: score noise inflates $Q_t$, post-step injection adds a channel term, and the Langevin thermal entry is $\beta^{-1}\eta_td$ (\S\ref{subsec:add-langevin}) \\
Finite-cumulant non-Gaussian (Rademacher, Laplace, mixtures) & the Gaussian quadratic is replaced by the cumulant-generating function $\Lambda_t(v)$ (Proposition~\ref{prop:add-cgf-noise}), of which it is the second-order term \\
Heavy-tailed and $\alpha$-stable & the pathwise and channel relative-entropy identities still hold. The conditional-expectation entry is finite only after truncation, robustification, or a fractional or jump regime \\
State-dependent (Riemannian) preconditioning & a metric-drift term with an It\^o correction \\
Phase-space momentum: stochastic-gradient Hamiltonian Monte Carlo (SGHMC), underdamped Langevin, thermostats & conservative Hamiltonian transport plus an Ornstein--Uhlenbeck thermostat channel; SGHMC's friction is the covariance-balance condition.  Heavy-ball is instead a linear phase-space recursion (Proposition~\ref{prop:add-phase-info-recursion}) and Nesterov a lookahead-transport correction \\
\bottomrule
\end{tabularx}
\caption{\textbf{Every noise class and every momentum scheme resolves to one explicit entry of the information ledger.}}
\label{tab:noise-momentum}
\end{table}

A local stochastic gradient is $g_t=\nabla f_t(\theta_t)$, the objective is $F$, and the preconditioner $\Sigma\succ0$ is fixed unless stated otherwise.
The deterministic preconditioned update is $\bar\theta_{t+1}=\theta_t-\eta_t\Sigma g_t$.
The noisy update is
\begin{equation}
\label{eq:add-noisy-sgd}
        \theta_{t+1}=\theta_t-\eta_t\Sigma g_t+\xi_t
        =\bar\theta_{t+1}+\xi_t
\end{equation}
where $\xi_t$ may be injected by the algorithm, may idealize mini-batch noise, or may combine both.

\subsection{The master decomposition under added noise}
\label{subsec:add-master-noisy}

When a perturbation $\xi_t$ follows the gradient update, the comparator information is evaluated at $\theta_{t+1}=\bar\theta_{t+1}+\xi_t$, and the deterministic ledger \eqref{eq:weighted-ledger} acquires a signed displacement term and a nonnegative energy term.

\begin{theorem}[Exact pointwise noisy-SGD ledger]
\label{thm:add-noisy-ledger}
Fix \(\Sigma\succ0\), a comparator \(u\in\R^d\), and an arbitrary realized sequence \((g_t,\xi_t)\).  Define
\[
        I_t(u)=\half\|u-\theta_t\|_{\Sigma^{-1}}^2,
        \qquad
        Q_t=\half g_t^\top\Sigma g_t,
        \qquad
        \bar\theta_{t+1}=\theta_t-\eta_t\Sigma g_t
\]
For the noisy recursion \eqref{eq:add-noisy-sgd}, the exact one-step identity is
\begin{equation}
\label{eq:add-noisy-ledger}
        \eta_t\ip{g_t}{\theta_t-u}
        = I_t(u)-I_{t+1}(u)+\eta_t^2Q_t
        +\underbrace{\ip{\bar\theta_{t+1}-u}{\Sigma^{-1}\xi_t}}_{M_t^\xi(u)}
        +\underbrace{\half\|\xi_t\|_{\Sigma^{-1}}^2}_{E_t^\xi}
\end{equation}
Equivalently,
\begin{equation}
\label{eq:add-noisy-ledger-expanded}
        \eta_t\ip{g_t}{\theta_t-u}
        = I_t(u)-I_{t+1}(u)+\eta_t^2Q_t
        +\ip{\theta_t-u}{\Sigma^{-1}\xi_t}
        -\eta_t\ip{g_t}{\xi_t}
        +\half\|\xi_t\|_{\Sigma^{-1}}^2
\end{equation}
If \(\xi_t\) is conditionally mean-zero and conditionally independent of \(g_t\) after \(\calF_t\), with conditional covariance \(\Omega_t=\E[\xi_t\xi_t^\top\mid\calF_t,g_t]\), then
\begin{equation}
\label{eq:add-noisy-ledger-conditional}
        \E\!\left[\eta_t\ip{g_t}{\theta_t-u}\mid\calF_t,g_t\right]
        = I_t(u)-\E[I_{t+1}(u)\mid\calF_t,g_t]+
        \eta_t^2Q_t+\half\tr(\Sigma^{-1}\Omega_t)
\end{equation}
\end{theorem}
The proof is in Appendix~\ref{app:proofs-deferred}.

The gradient step still incurs $\eta_t^2Q_t$, and the noise only appends entries: $M_t^\xi(u)$ records whether the perturbation moved the iterate toward or away from the comparator in the information metric, $E_t^\xi$ is its realized energy, and zero-mean noise contributes $\half\tr(\Sigma^{-1}\Omega_t)$ in conditional expectation.
No Gaussianity enters this algebra.

\begin{corollary}[Exact noisy objective ledger]
\label{cor:add-noisy-objective-ledger}
Let \(F\) be differentiable and write the stochastic-gradient error as \(\varepsilon_t=g_t-\nabla F(\theta_t)\).
Then, pathwise,
\begin{align}
\label{eq:add-noisy-objective-ledger}
        \eta_t\left(F(\theta_t)-F(u)+B_F(u,\theta_t)\right)
        &= I_t(u)-I_{t+1}(u)+\eta_t^2Q_t
        +M_t^\xi(u)+E_t^\xi
        -\eta_t\ip{\varepsilon_t}{\theta_t-u}
\end{align}
If \(g_t\) is conditionally unbiased and \(\xi_t\) is conditionally mean-zero and conditionally independent of \(g_t\) after \(\calF_t\), then
\begin{align}
\label{eq:add-noisy-objective-expected}
        \eta_t\E\!\left[F(\theta_t)-F(u)+B_F(u,\theta_t)\mid\calF_t\right]
        &=I_t(u)-\E[I_{t+1}(u)\mid\calF_t]
          +\eta_t^2\bar Q_t
          +\half\tr(\Sigma^{-1}\Omega_t)
\end{align}
where
\begin{equation}
\label{eq:add-expected-gradient-clock}
        \bar Q_t
        =\half\nabla F(\theta_t)^\top\Sigma\nabla F(\theta_t)
        +\half\tr(\Sigma C_t),
        \qquad
        C_t=\Cov(g_t\mid\calF_t)
\end{equation}
\end{corollary}
The proof is in Appendix~\ref{app:proofs-deferred}.

Covariance proportional to $\eta_t\Sigma$ (Langevin) makes the expected energy proportional to $\eta_td$, the usual temperature floor.
Covariance fixed independently of $\eta_t$ makes the noise penalty of \eqref{eq:add-noisy-objective-expected}, divided by $\eta_t$, grow as $1/\eta_t$, so a fixed-variance perturbation penalizes small learning rates.

\subsection{Noise as an information channel}
\label{subsec:add-channel}

Against a fixed comparator, noise adds energy. When the learner and comparator distributions pass through the same noise kernel, they lose distinguishability. This loss of distinguishability has a chain-rule form.
Let $K$ be a Markov kernel from $\mathsf X$ to $\mathsf Y$, write $PK$ for the output law of an input law $P$, and $P\otimes K$ for the joint law of $X\sim P$ and $Y\sim K(X,\cdot)$.

\begin{theorem}[Exact information loss of a noise channel]
\label{thm:add-channel-kl}
Let \(P,Q\) be probability measures with \(P\ll Q\), and let \(K\) be any Markov kernel.  Let \((P\otimes K)_{X\mid Y}\) and \((Q\otimes K)_{X\mid Y}\) be regular conditional laws.  Then, in \([0,\infty]\),
\begin{equation}
\label{eq:add-channel-kl}
        \KL(P\|Q)
        =\KL(PK\|QK)
        +\underbrace{\E_{Y\sim PK}
        \KL\!\left((P\otimes K)_{X\mid Y}\|(Q\otimes K)_{X\mid Y}\right)}_{\mathcal D_K(P,Q)\ge0}
\end{equation}
Consequently, \(\KL(PK\|QK)\le \KL(P\|Q)\), and when \(\KL(P\|Q)<\infty\) the exact contraction gap is \(\KL(P\|Q)-\KL(PK\|QK)=\mathcal D_K(P,Q)\).
\end{theorem}
The proof is in Appendix~\ref{app:proofs-deferred}.

\begin{corollary}[Gaussian heat contraction]
\label{cor:add-gaussian-contraction}
Let
\[
        P=\calN(u,\Sigma),\qquad Q=\calN(m,\Sigma),\qquad
        K(x,\cdot)=\calN(x,\Omega)
\]
where \(\Sigma\succ0\) and \(\Omega\succeq0\).  Then
\begin{align}
\label{eq:add-gaussian-kl-contraction}
        \KL(PK\|QK)
        &=\half(u-m)^\top(\Sigma+\Omega)^{-1}(u-m) \nonumber\\
        \mathcal D_K(P,Q)
        &=\half(u-m)^\top\left[\Sigma^{-1}-(\Sigma+\Omega)^{-1}\right](u-m)
\end{align}
\end{corollary}
The proof is in Appendix~\ref{app:proofs-deferred}.
The Gaussian smoothing bound of \cite{gergely2021informationtheoretic} is tight at point masses, where it agrees with Corollary~\ref{cor:add-gaussian-contraction}.
For laws of shared covariance the corollary gives the post-channel relative entropy in closed form.
Theorem~\ref{thm:add-channel-kl} supplies the contraction gap that a data-processing step discards.

The gradient step maps $q_t=\calN(\theta_t,\Sigma)$ to $\bar q_{t+1}=\calN(\theta_t-\eta_t\Sigma g_t,\Sigma)$, and the heat kernel $K_t(x,\cdot)=\calN(x,\Omega_t)$ then maps it to $q_{t+1}=\calN(\theta_t-\eta_t\Sigma g_t,\Sigma+\Omega_t)$.
If the comparator $\rho_t=\calN(u,\Sigma)$ passes through the same kernel, the deterministic SGD identity becomes
\begin{equation}
\label{eq:add-tilt-heat-ledger}
        \eta_t\ip{g_t}{\theta_t-u}
        =\KL(\rho_t\|q_t)-\KL(\rho_tK_t\|q_{t+1})+
        \eta_t^2Q_t-\mathcal D_{K_t}(\rho_t,\bar q_{t+1})
\end{equation}
where the discount $\mathcal D_{K_t}$ is the comparator information erased by the noise after the likelihood update.

\subsection{Gaussian Langevin dynamics}
\label{subsec:add-langevin}

For a potential $U$, inverse temperature $\beta>0$, and constant $\Sigma\succ0$, overdamped Langevin dynamics
\begin{equation}
\label{eq:add-overdamped-sde}
        \dd\theta_s=-\Sigma\nabla U(\theta_s)\dd s+
        \sqrt{2\beta^{-1}\Sigma}\,\dd B_s
\end{equation}
has, under standard regularity and integrability assumptions, the Gibbs invariant law $\Gibbs_\beta(\dd\theta)=Z_\beta^{-1}\exp(-\beta U(\theta))\dd\theta$.
Its Euler--Maruyama discretization with an unbiased mini-batch gradient is stochastic gradient Langevin dynamics \citep{welling2011bayesian}; nonasymptotic analyses relate the chain to the diffusion under regularity and spectral assumptions \citep{maxim2017nonconvex,pan2018global}.
Fixed-step SGLD has a persistent discretization and stochastic-gradient bias \citep{sebastian2016exploration}.

\begin{corollary}[Exact expected Langevin ledger]
\label{cor:add-langevin-ledger}
Let \(\widehat g_t=\nabla U(\theta_t)+\varepsilon_t\), with \(\E[\varepsilon_t\mid\calF_t]=0\) and covariance \(C_t\).  Run the SGLD step
\begin{equation}
\label{eq:add-sgld-step}
        \theta_{t+1}=\theta_t-\eta_t\Sigma\widehat g_t+
        \sqrt{2\beta^{-1}\eta_t\Sigma}\,\zeta_t,
        \qquad \zeta_t\sim\calN(0,I_d)
\end{equation}
with \(\zeta_t\) independent of \(\widehat g_t\) conditional on \(\calF_t\).  Then, for every comparator \(u\),
\begin{align}
\label{eq:add-sgld-expect-ledger}
        \eta_t\E\!\left[U(\theta_t)-U(u)+B_U(u,\theta_t)\mid\calF_t\right]
        &=I_t(u)-\E[I_{t+1}(u)\mid\calF_t]\nonumber\\
        &\quad+\frac{\eta_t^2}{2}\nabla U(\theta_t)^\top\Sigma\nabla U(\theta_t)
        +\frac{\eta_t^2}{2}\tr(\Sigma C_t)
        +\beta^{-1}\eta_t d
\end{align}
Equivalently, after division by \(\eta_t\),
\begin{align}
\label{eq:add-sgld-expect-ledger-divided}
        \E\!\left[U(\theta_t)-U(u)+B_U(u,\theta_t)\mid\calF_t\right]
        &=\frac{I_t(u)-\E[I_{t+1}(u)\mid\calF_t]}{\eta_t}
        +\frac{\eta_t}{2}\|\nabla U(\theta_t)\|_\Sigma^2
        +\frac{\eta_t}{2}\tr(\Sigma C_t)
        +\beta^{-1}d
\end{align}
\end{corollary}
The proof is in Appendix~\ref{app:proofs-deferred}.

Alongside the gradient and mini-batch intrinsic time, Langevin incurs the thermal term $\beta^{-1}\eta_td$, the per-step floor $\beta^{-1}d$ in the divided form \eqref{eq:add-sgld-expect-ledger-divided}.
As $\beta\to\infty$ Langevin reduces to SGD; at finite $\beta$ it samples a Gibbs neighborhood whose objective excess is set by the temperature.
As a sampler the overdamped dynamics obey the free-energy identity: for the law $q_s$ of $\theta_s$ under \eqref{eq:add-overdamped-sde},
\begin{equation}
\label{eq:add-fp-entropy-dissipation}
        \frac{\dd}{\dd s}\KL(q_s\|\Gibbs_\beta)
        =-\beta^{-1}\int
        \left\|\nabla\log\frac{q_s(\theta)}{\Gibbs_\beta(\theta)}\right\|_\Sigma^2
        q_s(\theta)\dd\theta
\end{equation}
so the same noise that is positive energy in the point-comparator ledger dissipates relative entropy to the Gibbs law at the relative Fisher-information rate.
With $\eta_t\downarrow0$ and injected covariance $2\beta^{-1}\eta_t\Sigma$ the discretization and stochastic-gradient biases can vanish under summability and regularity conditions, which is why SGLD anneals. At a fixed step the invariant law is a perturbation of $\Gibbs_\beta$, computed in closed form below for a quadratic.

\subsection{Quadratic example: the stationary covariance of noisy SGD and SGLD}
\label{subsec:add-quadratic-noisy}

For $U(\theta)=\half(\theta-\theta_\star)^\top H(\theta-\theta_\star)$ with $H\succ0$, the error $e_t=\theta_t-\theta_\star$ of noisy SGD follows $e_{t+1}=(I-\eta\Sigma H)e_t+\xi_t$.
If $\xi_t$ is mean-zero with covariance $\Omega$, independent of $e_t$, and $A=I-\eta\Sigma H$ has spectral radius below one, the stationary covariance $S$ solves $S=ASA^\top+\Omega$ and $\E_\infty U(\theta)=\half\tr(HS)$.

\begin{proposition}[One-dimensional quadratic noise floor]
\label{prop:add-1d-noise-floor}
In dimension one, let \(U(e)=\half h e^2\), \(\Sigma=1\), and \(e_{t+1}=(1-\eta h)e_t+\xi_t\), with \(\xi_t\) independent of \(e_t\), \(\E\xi_t=0\), and \(\Var(\xi_t)=\omega^2\).  If \(0<\eta h<2\), then
\begin{equation}
\label{eq:add-1d-var}
        \Var_\infty(e)=\frac{\omega^2}{1-(1-\eta h)^2}
        =\frac{\omega^2}{\eta h(2-\eta h)}
\end{equation}
and $\E_\infty U(e)=\tfrac{\omega^2}{2\eta(2- \eta h)}$.
For Langevin noise \(\omega^2=2\beta^{-1}\eta\),
\begin{equation}
\label{eq:add-1d-langevin-excess}
        \E_\infty U(e)=\frac{\beta^{-1}}{2-\eta h}\xrightarrow[\eta\downarrow0]{}\frac{1}{2\beta}
\end{equation}
\end{proposition}
The proof is in Appendix~\ref{app:proofs-deferred}.

The limit $1/(2\beta)$ is the Gibbs value for a one-dimensional quadratic well.
Noise of fixed variance $\omega^2$ leaves a floor that diverges as $\eta\downarrow0$, while Langevin noise, with $\omega^2\propto\eta$, leaves a finite floor set by the temperature.

\subsection{State-dependent preconditioning and Riemannian Langevin corrections}
\label{subsec:add-preconditioned-langevin}

For a smooth matrix field $G(\theta)\succ0$, the preconditioned Langevin diffusion targeting $\Gibbs_\beta$ is
\begin{equation}
\label{eq:add-riemannian-langevin}
        \dd\theta_s=
        \left[-G(\theta_s)\nabla U(\theta_s)+\beta^{-1}\Gamma(\theta_s)\right]\dd s
        +\sqrt{2\beta^{-1}G(\theta_s)}\,\dd B_s,
        \qquad
        \Gamma_i(\theta)=\sum_{j=1}^d \frac{\partial G_{ij}(\theta)}{\partial \theta_j}
\end{equation}
The It\^o correction $\Gamma$ vanishes for constant $G$ and is the sampling analogue of the metric-drift term of the discrete ledger.
Preconditioned SGLD \citep{chunyuan2016preconditioned} combines adaptive preconditioners with Langevin noise. The information ledger separates the sampler's gradient intrinsic time $\half g_t^\top G_tg_t$, the drift from replacing $G_t$ by $G_{t+1}$, and the correction required by a state-dependent $G$, whose omission biases its invariant law \citep{yian2015complete}.
For the discrete deterministic part $\theta_{t+1}=\theta_t-\eta_t\Sigma_tg_t$ with $I_t(u)=\half\|u-\theta_t\|_{\Sigma_t^{-1}}^2$,
\begin{equation}
\label{eq:add-variable-metric-ledger}
        \eta_t\ip{g_t}{\theta_t-u}
        =I_t(u)-I_{t+1}(u)+\frac{\eta_t^2}{2}g_t^\top\Sigma_tg_t
        +\Delta_t^{\mathrm{metric}}(u),
        \qquad
        \Delta_t^{\mathrm{metric}}(u)=\half(u-\theta_{t+1})^\top(\Sigma_{t+1}^{-1}-\Sigma_t^{-1})(u-\theta_{t+1})
\end{equation}
A metric that expands in the comparator direction, $\Sigma_{t+1}^{-1}\preceq\Sigma_t^{-1}$, gives favorable negative drift, and one that contracts in that direction increases terminal information.

\subsection{Survival and change under non-Gaussian noise}
\label{subsec:add-nongaussian}

Without Gaussianity, posteriors no longer stay Gaussian after reweighting and convolution, nor is the heat channel the discretized Brownian dynamics with Gibbs invariant law, but the pathwise and channel ledgers survive.
For finite-variance noise Theorem~\ref{thm:add-noisy-ledger} is already complete.
The conditional mean ledger depends on the noise only through its first two moments.

\begin{corollary}[Finite-variance noise enters the mean ledger through its first two moments]
\label{cor:add-finite-variance-noise}
Assume that conditional on \((\calF_t,g_t)\), the perturbation \(\xi_t\) has mean \(\bar\xi_t\) and covariance \(\Omega_t\).  Then
\begin{align}
\label{eq:add-biased-noise-ledger}
        \E\!\left[\eta_t\ip{g_t}{\theta_t-u}\mid\calF_t,g_t\right]
        &=I_t(u)-\E[I_{t+1}(u)\mid\calF_t,g_t]
        +\eta_t^2Q_t\nonumber\\
        &\quad+\ip{\bar\theta_{t+1}-u}{\Sigma^{-1}\bar\xi_t}
        +\half\tr(\Sigma^{-1}\Omega_t)
        +\half\|\bar\xi_t\|_{\Sigma^{-1}}^2
\end{align}
\end{corollary}
The proof is in Appendix~\ref{app:proofs-deferred}.

Perturbations with the same conditional mean and covariance therefore give the same expected point-comparator ledger, Gaussian, Rademacher, or uniform alike; higher distributional features reappear in tail probabilities, exit times, and exponential-moment normalizers.
The adaptive-Bayes identity is driven by log-normalizers.
For noise with a finite moment-generating function the normalizer is its cumulant-generating function $\Lambda_t(a)=\log\E[\exp(\ip{a}{\xi_t})\mid\calF_t]$.

\begin{proposition}[Exact non-Gaussian cumulant insertion]
\label{prop:add-cgf-noise}
Let \(\xi_t\) be conditionally independent of \(g_t\) given \(\calF_t\).  For any \(v\in\R^d\) with \(\Lambda_t(v)<\infty\),
\begin{equation}
\label{eq:add-cgf-noise-normalizer}
        \log\E\left[\exp\left(\ip{v}{\bar\theta_{t+1}+\xi_t}\right)\mid\calF_t,g_t\right]
        =\ip{v}{\bar\theta_{t+1}}+\Lambda_t(v)
\end{equation}
\end{proposition}
The proof is in Appendix~\ref{app:proofs-deferred}.

The function is explicit for the common laws:
\[
\Lambda_t(v)=
\begin{cases}
\half v^\top\Omega_tv & \xi_t\sim\calN(0,\Omega_t)\\
\sum_i\log\cosh(a_iv_i) & \text{independent Rademacher coordinates }\xi_{t,i}=a_i\epsilon_i\\
\sum_i\log(1-b_i^2v_i^2)^{-1}\ \text{on }|b_iv_i|<1 & \text{independent centered Laplace coordinates with scales }b_i\\
\log\sum_kw_k\exp(\half v^\top\Omega_kv) & \text{a Gaussian mixture }\sum_kw_k\calN(0,\Omega_k)
\end{cases}
\]
The mixture's function sees the average covariance in small directions and the broadest component in large ones.
Where the moment-generating function is finite near zero, $\Lambda_t(v)=\half v^\top\Omega_tv+O(\|v\|^3)$ for centered noise, so the Gaussian analysis is the second-order approximation to the cumulant information ledger.
Truncating a heavy-tailed jump at a finite radius, or tempering a stable law, restores a finite $\Lambda_t$ on its domain, the omitted jump contribution appearing as a bias.

Heavy-tailed gradient noise has been observed in deep learning, and $\alpha$-stable models have been proposed in place of Gaussian diffusion approximations \citep{umut2019tailindex,thanh2019first}, their jumps linked to escape from narrow minima \citep{umut2019heavytailed}.
The pathwise identity \eqref{eq:add-noisy-ledger} holds for every realized jump, since the algebra needs no moments.
The conditional expectation \eqref{eq:add-noisy-ledger-conditional} fails to exist when $\E\|\xi_t\|_{\Sigma^{-1}}^2=\infty$.
For a symmetric $\alpha$-stable perturbation with $0<\alpha<2$ the cumulant $\Lambda_t(v)$ is infinite in every direction $v$ along which $\ip{v}{\xi_t}$ is not almost surely zero, so the score must be truncated or robustified (Appendix~\ref{app:identities}) or the dynamics treated with fractional-generator tools.
The channel identity of Theorem~\ref{thm:add-channel-kl} remains valid whenever its relative entropies are defined.
The L\'evy-driven continuous-time analogue does not in general keep $\exp(-\beta U)$ as its invariant law, so it models exploration and does not sample the posterior.
On a character-level transformer the per-coordinate gradient noise at batch $32$ is not heavy-tailed (Figure~\ref{fig:native-tails}).
A tail index pooled over coordinates of different scales can fall below two there even for Gaussian noise, so such an index is not by itself evidence of a heavy tail.
Heavy tails in the iterates have been associated with a large ratio of step size to batch size \citep{barsbey2021heavy}, so a measurement at one such ratio does not settle the question at another.
Where such a regime does arise, stability analyses of heavy-tailed dynamics extend to general non-convex losses \citep{raj2023algorithmic}.

\subsection{Learning-rate design with noise}
\label{subsec:add-noisy-learning-rates}

By \eqref{eq:add-noisy-ledger-conditional}, noisy optimization has the enlarged one-step cost $\bar Q_t^{\mathrm{noisy}}(\eta_t)=Q_t+\tfrac{1}{2\eta_t^2}\tr(\Sigma^{-1}\Omega_t)$.
The square-root schedule applies with this increment in place of $Q_t$.
For Langevin noise $\Omega_t=2\beta^{-1}\eta_t\Sigma$ the enlarged one-step cost becomes
\begin{equation}
\label{eq:add-langevin-noisy-clock}
        \bar Q_t^{\mathrm{SGLD}}(\eta_t)
        =Q_t+\frac{\beta^{-1}d}{\eta_t}
\end{equation}
whose thermal term is chosen to preserve a temperature; optimization anneals $\beta^{-1}$ or the injected noise, while sampling keeps $\beta$ fixed and accepts the floor.

\begin{proposition}[Exact conditional pressure target for noisy steps]
\label{prop:add-noisy-pressure}
Fix \(u\), \(\Sigma\), and a target \(a_t\).  Let the additive noise be conditionally mean-zero and conditionally independent of \(g_t\) given \(\calF_t\), with covariance \(\Omega_t=\E[\xi_t\xi_t^\top\mid\calF_t,g_t]\), and let \(\eta_t\) be \((\calF_t,g_t)\)-measurable.  The conditional expected point-comparator ledger hits the target \(\E\left[\eta_t\ip{g_t}{\theta_t-u}-I_t(u)+I_{t+1}(u)\mid\calF_t,g_t\right]=a_t\) if and only if
\begin{equation}
\label{eq:add-noisy-pressure}
        \eta_t^2Q_t+
        \half\tr(\Sigma^{-1}\Omega_t)=a_t
\end{equation}
When \(Q_t>0\) and \(\Omega_t\) does not depend on \(\eta_t\), this gives
\begin{equation}
\label{eq:add-noisy-pressure-solution}
        \eta_t=
        \sqrt{\frac{a_t-\half\tr(\Sigma^{-1}\Omega_t)}{Q_t}}
\end{equation}
provided the numerator is positive.  For Langevin noise \(\Omega_t=2\beta^{-1}\eta_t\Sigma\), \eqref{eq:add-noisy-pressure} is the quadratic \(\eta_t^2Q_t+\beta^{-1}d\,\eta_t=a_t\), with positive root \(\eta_t=\bigl(-\beta^{-1}d+\sqrt{\beta^{-2}d^2+4Q_ta_t}\bigr)/(2Q_t)\).
\end{proposition}
The proof is in Appendix~\ref{app:proofs-deferred}.

If the numerator of \eqref{eq:add-noisy-pressure-solution} is not positive, the noise energy alone already meets or exceeds the target, so no positive learning rate makes the expected ledger hit the target.
The proposition is the noisy analogue of the pressure-target rule of \cite{balsubramani2026adaptive}.

\subsection{Saddles under noisy dynamics}
\label{subsec:add-saddles-noisy}

Near a critical point with $U(\theta)=U(0)+\half\theta^\top H\theta$ and a negative eigenvalue $-\gamma$ of $H$, gradient descent with scalar step $\eta$ satisfies, along that eigenvector,
\begin{equation}
\label{eq:add-saddle-1d}
        x_{t+1}=(1+\eta\gamma)x_t+\xi_t
\end{equation}

\begin{proposition}[Exact unstable-direction second moment]
\label{prop:add-saddle-second-moment}
For \eqref{eq:add-saddle-1d}, assume \(\E\xi_t=0\), \(\Var(\xi_t)=\omega^2\), and independent increments.  Then
\begin{equation}
\label{eq:add-saddle-second-moment}
        \E x_t^2=(1+\eta\gamma)^{2t}x_0^2
        +\omega^2\sum_{s=0}^{t-1}(1+\eta\gamma)^{2s}
        =(1+\eta\gamma)^{2t}x_0^2
        +\omega^2\frac{(1+\eta\gamma)^{2t}-1}{(1+\eta\gamma)^2-1}
\end{equation}
\end{proposition}
The proof is in Appendix~\ref{app:proofs-deferred}.

For Langevin noise, $\omega^2=2\beta^{-1}\eta$, the injected unstable-direction information grows geometrically at the same rate.
Under heavy-tailed perturbations with $\text{Pr}(|\xi_t|>r)\asymp r^{-\alpha}$, escape may instead be a single jump.
The waiting time for a jump exceeding a width $w$ is of order $w^\alpha$, with no Arrhenius exponential in a barrier height.

\subsection{Momentum: the information ledger in phase space}
\label{subsec:add-momentum}

Momentum enlarges the state to position and velocity.
The same accounting applies in phase space: Hamiltonian transport preserves information, friction contracts it, and injected noise adds a thermostat channel.
The heavy-ball update $v_{t+1}=\mu_tv_t-\eta_t\Sigma g_t+\xi_t$, $\theta_{t+1}=\theta_t+v_{t+1}$ is, for a quadratic with $g_t=H(\theta_t-u)$ and $s_t=(\theta_t-u,v_t)$, the affine recursion
\begin{equation}
\label{eq:add-heavy-ball-noise-matrix}
        s_{t+1}=A_Hs_t+
        \begin{pmatrix}\xi_t\\ \xi_t\end{pmatrix},
        \qquad
        A_H=
        \begin{pmatrix}
        I-\eta\Sigma H & \mu I\\
        -\eta\Sigma H & \mu I
        \end{pmatrix}
\end{equation}

\begin{proposition}[Exact phase-space information recursion]
\label{prop:add-phase-info-recursion}
Let \(P\succ0\) be a phase-space metric and \(\mathcal I_t=\half s_t^\top Ps_t\).  For the affine recursion \(s_{t+1}=A_ts_t+b_t\),
\begin{equation}
\label{eq:add-phase-exact}
        \mathcal I_t-\mathcal I_{t+1}
        =\half s_t^\top(P-A_t^\top PA_t)s_t
        -s_t^\top A_t^\top Pb_t
        -\half b_t^\top Pb_t
\end{equation}
If \(s_t\) and \(A_t\) are \(\calF_t\)-measurable and \(b_t\) is conditionally mean-zero with covariance \(\Xi_t\), then
\begin{equation}
\label{eq:add-phase-expect}
        \E[\mathcal I_t-\mathcal I_{t+1}\mid\calF_t]
        =\half s_t^\top(P-A_t^\top PA_t)s_t
        -\half\tr(P\Xi_t)
\end{equation}
\end{proposition}
The proof is in Appendix~\ref{app:proofs-deferred}.

For a Hessian eigenvalue $h$ in the scalar unpreconditioned case, the heavy-ball transition matrix has characteristic polynomial
\begin{equation}
\label{eq:add-heavy-ball-eigenvalues}
        \lambda^2-(1+\mu-\eta h)\lambda+\mu=0
\end{equation}
For $h>0$ stability requires both roots inside the unit disk, the usual heavy-ball constraints.
For a saddle-point direction $h=-\gamma<0$ the polynomial evaluated at $1+\eta\gamma$ equals $-\mu\eta\gamma<0$, so for every $\mu>0$ the larger root exceeds the no-momentum factor $1+\eta\gamma$: momentum amplifies negative-curvature directions in phase space.
Velocity noise is amplified by the same transition matrix.

With a momentum $r$, mass matrix $M\succ0$, and friction $C\succeq0$, underdamped Langevin dynamics $\dd\theta_s=M^{-1}r_s\dd s$, $\dd r_s=-\nabla U(\theta_s)\dd s-CM^{-1}r_s\dd s+\sqrt{2\beta^{-1}C}\,\dd B_s$ has invariant law $\Pi_\beta\propto\exp[-\beta(U(\theta)+\half r^\top M^{-1}r)]$.
Its Hamiltonian part conserves energy and transports information without spending it. The law $q_s$ in phase space satisfies
\begin{equation}
\label{eq:add-underdamped-entropy}
        \frac{\dd}{\dd s}\KL(q_s\|\Pi_\beta)
        =-\beta^{-1}\int
        \left\|\nabla_r\log\frac{q_s}{\Pi_\beta}\right\|_{C}^2
        q_s(\theta,r)\dd\theta\dd r
\end{equation}
the Ornstein--Uhlenbeck thermostat dissipating it at the relative Fisher-information rate.
Stochastic-gradient Hamiltonian Monte Carlo replaces $\nabla U$ by a mini-batch estimate, whose noise already heats the momentum \citep{tianqi2014stochastic,yian2015complete}.
For the momentum update $r_{t+1}=r_t-\eta\nabla U(\theta_t)-\eta CM^{-1}r_t-\eta\varepsilon_t+\zeta_t$, with gradient error of covariance $B_t$ and injected noise of covariance $R_t$, the fluctuation--dissipation relation requires, at first order in $\eta$,
\begin{equation}
\label{eq:add-sghmc-balance}
        \eta^2B_t+R_t\approx2\beta^{-1}\eta C
\end{equation}
so $R_t\approx2\beta^{-1}\eta C-\eta^2B_t$ must be positive semidefinite; underestimating the mini-batch noise overheats the sampler.
Stochastic-gradient thermostats \citep{nan2014bayesian} add a variable that adapts the friction until the mean kinetic energy reaches its equilibrium value $d/(2\beta)$, a pressure-target controller in phase space whose information ledger follows from Proposition~\ref{prop:add-phase-info-recursion} on the enlarged state.
Nesterov momentum evaluates the gradient at $y_t=\theta_t+\mu_tv_t$, so the objective conversion is $\ip{\nabla F(y_t)}{y_t-u}=F(y_t)-F(u)+B_F(u,y_t)$ and $\ip{\nabla F(y_t)}{\theta_t-u}=\ip{\nabla F(y_t)}{y_t-u}-\mu_t\ip{\nabla F(y_t)}{v_t}$ records the lookahead transport correction.

\subsection{Mini-batch gradient noise as intrinsic time}
\label{subsec:add-minibatch-noise}

The stochastic part of the expected intrinsic-time density \eqref{eq:add-expected-gradient-clock} is $\half\tr(\Sigma C_t)$.
For a mini-batch of size $b_t$ drawn independently it is $\tfrac{1}{2b_t}\tr(\Sigma C_t^{(1)})$, with $C_t^{(1)}$ the single-example covariance.
Larger batches reduce this term and the exploration it drives, since, to leading order in $\eta_t$, the trajectory's stochastic temperature is set by the product $\eta_t\Sigma C_t$.
In a quadratic basin with constant step $\eta$ the stationary covariance solves
\begin{equation}
\label{eq:add-sgd-lyapunov}
        S=(I-\eta H)S(I-\eta H)^\top+\eta^2 C
\end{equation}
which for small $\eta$ is $HS+SH\approx\eta C$, giving $S_i\approx\eta C_i/(2h_i)$ along common eigendirections of $H$ and $C$, so at comparable noise, flat directions retain the larger variance.
Matching a Bayesian posterior requires a covariance relation between gradient noise and curvature that a scalar learning rate cannot supply \citep{stephan2017stochastic}.
Flatness has a sampling notion, small $\mathcal E_\beta$ in \eqref{eq:laplace}, and a constant-step SGD notion, large stationary covariance under \eqref{eq:add-sgd-lyapunov}. It also has a channel notion, a small contraction gap in Theorem~\ref{thm:add-channel-kl} after the algorithm's noise channel. These notions coincide only under special covariance matching.
\subsection{Loss-surface complexity with noise and momentum}
\label{subsec:add-complexity-noise-momentum}

Summing \eqref{eq:add-noisy-ledger} gives the information ledger of a noisy trajectory to a comparator,
\begin{equation}
\label{eq:add-noisy-ledger-total}
        \sum_{t=1}^T\eta_t\ip{g_t}{\theta_t-u}
        =I_1(u)-I_{T+1}(u)+
        \sum_{t=1}^T\eta_t^2Q_t+
        \sum_{t=1}^TE_t^\xi+\sum_{t=1}^TM_t^\xi(u)
\end{equation}
so a surface is difficult for an algorithm and comparator when reaching the target needs large gradient intrinsic time, injected-noise energy, unfavorable displacement, metric drift, or terminal comparator information.
Phase-space methods replace $I_t$ by $\mathcal I_t$ and add the entries of Proposition~\ref{prop:add-phase-info-recursion}.

\section{Additional identities}\label{app:identities}

\subsection{The luckiness constant on a bounded region}

For a smooth strongly-convex $F$ with $\bar g=\nabla F(\theta_t)$, condition \eqref{eq:luckiness} holds with a constant $\varkappa$ provided the iterates stay in a bounded region: writing $\sigma_\parallel^2(\theta)=\Var(\ip{g_t}{\theta-\theta_*})$, the smallest admissible constant is $\varkappa=\sup_\theta\bigl[\ip{\nabla F(\theta)}{\theta-\theta_*}+\sigma_\parallel^2(\theta)/\ip{\nabla F(\theta)}{\theta-\theta_*}\bigr]$, so two quantities require control. On a ball $\|\theta-\theta_*\|\le D$ the smoothness bound $\langle\nabla F(\theta),\theta-\theta_*\rangle\le L\|\theta-\theta_*\|^2\le LD^2$ controls the first term at $O(LD^2)$. The second term is the gradient-noise contribution. It vanishes for a noiseless oracle, where $\varkappa=O(LD^2)$ is the whole account; for an oracle with variance $\sigma^2$ along $\theta-\theta_*$ it contributes $\sigma^2/\mu$, which no function of $(L,\mu,D)$ bounds and which does not vanish as $\theta\to\theta_*$. On an unbounded domain the gradient inner product $\langle\nabla F(\theta),\theta-\theta_*\rangle$ grows at least like $\mu\|\theta-\theta_*\|^2$, without bound, so no constant $\varkappa$ depending only on $(L,\mu)$ suffices; the bounded-iterate hypothesis (or a projection or bounded-gradient assumption that implies bounded iterates) is required. Under the bounded-iterate hypothesis and bounded gradient noise on the realized region, condition \eqref{eq:luckiness} holds with finite $\varkappa=O(LD^2+\sigma^2/\mu)$.
The loss identity of Theorem~\ref{thm:loss-identity} then controls the comparator-centered regret by $\varkappa$-controlled terms on the realized region. The bounded-iterate hypothesis makes $\varkappa$ a regional quantity. With gradient noise, fixed-rate SGD on a strongly convex objective does \emph{not} attain constant regret: the iterate settles into an $O(\sqrt{\eta\sigma^2/\mu})$ ball and the cumulative regret grows $\Theta(\eta\sigma^2 T)$. $O(\log T)$ regret follows from the decaying Robbins--Monro rate $\eta_t=1/(\mu t)$, under which the contraction \eqref{eq:strong-recursion} telescopes to $\E I_{T+1}(\theta_\star)=O(\log T/(\mu^2T))$, and to $O(1/(\mu^2T))$ with the factor $1-2\mu\eta_t$ (\S\ref{sec:convex}). Constant regret at a fixed rate needs a noiseless oracle.

\subsection{Biased estimators and Adam-style bias correction}

A biased estimator (EMA, momentum, clipped gradient) introduces a predictable offset $\bar{\hat g}_t-g_t\ne0$ and with it the extra term $\Bias_T=\sum_{t\le T}\ip{\theta_t-\theta_*}{\bar{\hat g}_t-g_t}$.
Adam's $1/(1-\beta_1^t)$ correction cancels the leading-order bias in $\bar{\hat g}_t$, and $\Bias_T$ records the remaining bias.

\subsection{Heavy-tailed gradients and Catoni transforms}

For heavy-tailed gradient noise, replace $\hat g_t$ by the coordinate-wise robustified score $g_t^{\mathrm{Cat}}=c^{-1}\psi_{\mathrm{Cat}}(c\,\hat g_t)$ at a scale $c>0$ \citep{catoni2012ann}. The accounting applies with a modified bias term.
The robustified clock has finite variance even when $\E\|\hat g_t\|^4=\infty$, close in spirit to gradient clipping \citep{razvan2013difficulty}.
Below stability index two the realized noise-energy entry develops a heavy tail while the pathwise identity is unaffected.
Truncation and Catoni robustification each restore a finite entry (Figure~\ref{fig:single-term-checks}(e)).

\subsection{Comparator-centered PAC-Bayes regret}
Replacing the Dirac comparator by a Gaussian $\rho=\calN(\theta_*,\tau^2I)$ with $\tau\le\sigma$ yields
\begin{equation}\label{eq:pac-bayes-regret}
\KL(\rho\|\pi)=\half\left(d\log\tfrac{\sigma^2}{\tau^2}+d\tfrac{\tau^2}{\sigma^2}-d+\tfrac{\|\theta_*\|^2}{\sigma^2}\right)
\end{equation}
and Theorem~\ref{thm:loss-identity} gives a PAC-Bayesian regret identity against any such comparator. Shrinking $\tau$ approaches the point-comparator regret; holding $\tau$ at $\sigma$ produces the ``self-certified'' posterior that matches SGD's own Bayesian state. See \S\ref{sec:gen} for generalization implications.

\subsection{Taylor descent for smooth nonconvex functions}

For deterministic gradient descent $\theta_{t+1}=\theta_t-\eta_t\Sigma g_t$ with $g_t=\nabla F(\theta_t)$, the objective descent is exact with an integral remainder,
\[
F(\theta_{t+1})-F(\theta_t)
=-\eta_t\nabla F(\theta_t)^\top\Sigma\nabla F(\theta_t)
+\eta_t^2\int_0^1(1-s)(\Sigma g_t)^\top
\nabla^2F(\theta_t-s\eta_t\Sigma g_t)(\Sigma g_t)\,ds
\]
and the descent lemma follows by bounding the integral with a smoothness constant.

\subsection{AdaGrad two-sided envelope}\label{app:adagrad-envelope}

For the square-root step run on the clock including the current round, $\eta_t=C\sqrt{\Gamma/V_t}$ with $V_t=\sum_{s\le t}Q_s$, the clock identity gives the two-sided envelope $C\sqrt{\Gamma V_T}\le\sum_{t\le T}\eta_tQ_t\le 2C\sqrt{\Gamma V_T}$.
Its left end is attained when a single round accounts for the whole clock and its right end is approached by many equal increments, so the AdaGrad gradient regret is an identity plus a pathwise slack of at most a factor of two.
The predictable schedule \eqref{eq:sqrt-schedule}, which uses $V_{t-1}$, need not satisfy the envelope.

\subsection{Fixed-point partition-function identity for Polyak-step SGD}\label{app:polyak-partition}

For Polyak-step SGD satisfying \eqref{eq:pressure} with $a_t=f_t(\theta_*)$, each round has $-\log Z_t=\eta_ta_t$, and the normalizers of the local update compose multiplicatively, so
\[
\sum_{t\le T}\eta_t a_t=-\log\int\calN(\theta;\theta_1,\Sigma)\exp\left(-\sum_{t\le T}\eta_t\ip{g_t}{\theta-r_t}\right)d\theta
\]
the cumulative log-partition function of the linearized problem.

\section{Experimental evaluation in full}\label{app:experiments-full}

\subsection{The optimizer panel on a real network}\label{app:regime-discrimination}

A two-layer network is trained full-batch in double precision on a real handwritten-digit dataset ($1257$ training points) with SGD, SGD with momentum, AdamW, RMSProp, and AdaGrad, three seeds each.
Each step's realized update $\Delta_t=\theta_t-\theta_{t+1}$ is split exactly against the run's final iterate $u$:
\[
\langle \Delta_t,\theta_t-u\rangle
= \underbrace{\tfrac12\!\left(\|\theta_t-u\|^2-\|\theta_{t+1}-u\|^2\right)}_{\text{comparator-information transport}}
+ \underbrace{\tfrac12\|\Delta_t\|^2}_{\text{intrinsic-time cost}}
\]
Over $400$ full-batch steps the intrinsic-time fraction \eqref{eq:path-fraction} rises from plain SGD ($0.6\%$) through momentum and AdamW to RMSProp ($7.0\%$) and AdaGrad ($8.4\%$).

Stopping each optimizer the first time its full-batch training loss crosses a common target matches the training losses by construction.
Plain SGD then has the smallest fraction and AdaGrad and RMSProp the largest (Figure~\ref{fig:eval-regime-discrimination}(a)).
SGD takes $567$ steps against RMSProp's $29$.
Its steps are about twice as consistent ($R/T$ of $0.33$ against $0.17$).

Sweeping the learning rate at a fixed budget puts the smallest training loss at $\eta=0.4$ and the smallest test error at $\eta=0.1$.
Larger steps make the path less consistent, with $R/T$ falling from $0.47$ to $0.066$, so the fraction rises monotonically with $\eta$, from $0.0014$ to $0.0099$, through the test optimum (Figure~\ref{fig:eval-regime-discrimination}(b)).
Widening the network under $15\%$ label noise reproduces double descent, with the test error peaking at $1210$ parameters against $1257$ training points.
The fraction has no feature at the peak (Figure~\ref{fig:eval-regime-discrimination}(c)).
On the $14$ runs that interpolate, the fraction and the training loss rank the test error equally well (Spearman $-0.91$ and $+0.91$).

\subsection{Slack of the classical envelopes on real networks}\label{app:slack-protocol}

Each (network, optimizer, seed) run records one logged trajectory with weight decay zero, so that no part of the update falls outside the optimizer's metric and the identity in that metric needs no extra term.
The networks are CIFAR-10 ResNet-20 ($8{,}000$ steps, batch $128$), a four-layer character transformer of width and context $128$ on a bundled real-text corpus ($4{,}000$ steps, batch $32$ sequences), and a small vision transformer on CIFAR-10 ($8{,}000$ steps, batch $128$) for AdamW, Muon, SOAP and Sophia-G only.
The optimizers are SGD, heavy-ball momentum, SGD on the realized square-root clock ($\eta_t=c_0/\sqrt{V_t}$, $V_t=\sum_{s\le t}\|g_s\|^2/2$), AdaGrad, RMSProp, AdamW \citep{ilya2019decoupled}, Muon, SOAP, Sophia-G, and Schedule-Free SGD, each implemented against its reference code.
Each runs three seeds at a learning rate chosen per (network, optimizer) by a sweep on seed $0$ at a quarter of the full budget, held constant except on the square-root clock.
Heavy-ball momentum on the character LM diverged at seed $1$ and is excluded wherever a measurement needs it.
The information ledger tracks the parameter vector, except for Schedule-Free, whose information ledger runs on its base sequence $z_{t+1}=z_t-\gamma g_t$.
The comparator $u$ is each run's returned final iterate.
At $48$ grid points and the terminal point, float64 accumulators record the realized weighted regret and $I_t(u)$ in the Euclidean metric and, for the diagonal-metric optimizers, in the optimizer's own metric together with its drift term.
At every prefix $n$ each envelope's slack is attributed to the terms dropped by its derivation.
For SGD's linearized envelope the attribution is $I_n(u)$.
At the objective level it adds the cumulative Bregman remainder $\sum_t\eta_tB_F(u,\theta_t)$ as a trapezoid quadrature over the full-batch grid evaluations.
For Schedule-Free on $z$ it adds the gradient-ceiling term $\gamma^2(G^2n-\sum_t\|g_t\|^2)/2$, with $G$ the largest gradient norm seen up to the prefix.
For the square-root clock it records the position of $\sum_t\eta_tQ_t$ inside the band $[c_0\sqrt{V_n},2c_0\sqrt{V_n}]$ of Appendix~\ref{app:adagrad-envelope}.

There is no slack in the identity itself, at any prefix of the $71$ finite runs, in the Euclidean metric for every optimizer and in each diagonal-metric optimizer's own metric.
At the objective level the cumulative Bregman remainder, which convexity would make nonnegative, is negative over $94$--$98\%$ of the run for every optimizer without a momentum buffer on the character LM and for AdaGrad and the square-root clock on ResNet-20 (Figure~\ref{fig:slack-terminal}(b)).
Schedule-Free's gradient-ceiling relaxation costs $2.4\times$ the linearized envelope on ResNet-20 and $2.7\times$ on the LM by the end of the run.
The square-root clock's realized $\sum_t\eta_tQ_t$ sits $0.96$ (ResNet-20) and $0.79$ (LM) of the way up its two-sided band, so the band's upper end, the one used by the classical bound, is nearly tight there.

\begin{figure}[ht]
\centering
\includegraphics[width=\linewidth]{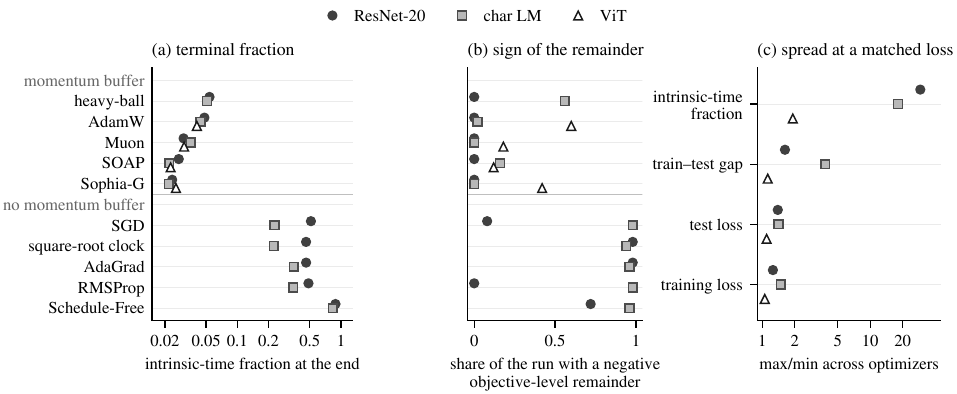}
\caption{\textbf{The information ledger separates the optimizers that keep a momentum buffer from those that do not.}
(a) Intrinsic-time fraction \eqref{eq:path-fraction} at the end of the run, in the Euclidean metric against the final iterate.
(b) Share of the logged grid points at which the cumulative Bregman remainder is negative. There the objective-level envelope does not apply.
(c) Max/min ratio across optimizers at the deepest training loss reached by every optimizer of a network. The fraction spreads by an order of magnitude on ResNet-20 and the char LM, where the training and test losses spread by less than $1.5\times$.
Markers are medians over seeds.
The vision transformer runs only AdamW, Muon, SOAP and Sophia-G.}
\label{fig:slack-terminal}
\end{figure}


%
%
%
The same trajectories are also measured at the first grid point at which each optimizer's training loss is at or below the deepest target every optimizer in the (network, seed) panel reaches, the comparator being that stopping iterate.
The crossing is located on the $48$-point grid, so the achieved training losses match only to within the spreads of Figure~\ref{fig:slack-terminal}(c).
The order-of-magnitude separation of Appendix~\ref{app:regime-discrimination} persists at loss spreads of $1.25$--$1.48\times$.
On both $10$-optimizer networks the optimizers that keep a momentum buffer sit at fractions from $0.02$ to $0.05$, below every optimizer without one (Figure~\ref{fig:slack-vs-bound}(c)).
The fraction's rank correlation with the train--test gap across optimizers averages $+0.41$ over seeds on ResNet-20 (permutation $p=0.035$).
On the LM its average, $-0.18$ ($p=0.37$), is indistinguishable from zero.
The two networks' correlations differ ($p=0.03$).
Within the buffer optimizers and within the rest no correlation is reliable.
On ResNet-20 the buffer optimizers' gap is below the others' on every seed, by $0.08$ on average.
On the LM it is above theirs on every seed, by $0.014$ on average.
Muon and SOAP account for the excess.

A ranking of the optimizers from a prefix uses a statistic chosen before training: the cumulative native-metric intrinsic time over the first $20\%$ of the steps, net of the cumulative metric drift against the prefix iterate.
For Muon and SOAP, whose metrics are not diagonal, it is the intrinsic time alone.
The baseline is the training loss at the same prefix.
Across optimizers within each (network, seed) panel, the net intrinsic time ranks the final training losses at a mean Spearman $\rho$ of $+0.11$ over the six panels of the two $10$-optimizer networks, against $+0.78$ for the prefix loss.
On the character-LM panel of seed $1$, where heavy-ball diverged, nine runs are ranked.
The rank correlation of the net intrinsic time is negative on every ResNet-20 panel ($-0.73$ to $-0.20$) and positive on every character-LM panel ($+0.57$ to $+0.64$).
For the final test losses both correlations are negative on every character-LM panel and positive on every ResNet-20 panel, so their means are $+0.07$ for the net intrinsic time and $-0.13$ for the prefix loss.
At the end of the run the intrinsic-time fraction ranks the final training losses at a mean $+0.71$ on the character LM and $+0.05$ on ResNet-20.
On ResNet-20 its rank correlation reaches $+0.44$ without Schedule-Free, whose fraction is the highest and whose final loss the lowest.
The fraction at the prefix and at the end splits the optimizers by whether they keep a momentum buffer.
On the character LM the five buffer optimizers end at a lower final training loss than the five optimizers without one in $24$ of $25$ pairs on each of seeds $0$ and $2$, the exception being heavy-ball against RMSProp.
On seed $1$, where heavy-ball diverged, they do so in all $20$ pairs.
On ResNet-20 the final training losses of the two groups overlap.

\subsection{Single terms of the information ledger}\label{app:sanity}
Figure~\ref{fig:single-term}(a) uses $d=50$ linear regression at seven batch sizes, with $10$ seeds per size.
The fit proportional to the mini-batch factor lies within $1.4$ standard errors of every mean.
Its slope $1.05{\times}10^6$ matches $\half\sum_t\tr(\Sigma C_t^{(1)})$ for the single-example covariance $C_t^{(1)}$.
Figure~\ref{fig:single-term}(b) uses $F(x,y,z)=\half(x^2+y^2)-\half z^2$ from $\theta_0=(1,1,\varepsilon_{\mathrm{init}})$, with $\eta=0.05$, $r_{\mathrm{esc}}=0.1$ and $30$ seeds per configuration.
Without noise the escapes take $1.021$ times the leading-order time; with noise the ratios lie between $1.009$ and $1.421$.

With $20$ seeds per configuration, the realized regret over $I_1+\sum_t\eta_t^2Q_t$ lies between $0.492$ and $0.500$ in the eight stable configurations of Figure~\ref{fig:single-term-checks}(a), the half-gain of \S\ref{sec:convex}.
All $144$ runs of panel (b) escape, with $8$ seeds per configuration.
At escape the coefficient of variation of $V_T^-$ \eqref{eq:V-minus} is $0.54$, against $1.19$ for $V_T$.
In panel (c), with $5$ seeds, SVRG's ratio $\hat V_T/V_T^{\det}$ of \S\ref{sec:stochastic} at inner loop $50$ is $0.98$ on Adult and $0.99$ on Covertype, against $1.44$ and $5.58$ for SGD.
The deterministic reference sums every fifth step, which lowers each ratio by $1.6$--$1.9\%$.
Against an every-step reference SVRG removes all but $0.2\%$ and $0.1\%$ of the mini-batch excess.
In panel (d), plain SGD stays finite on all $50$ seeds.
Catoni's score (Appendix~\ref{app:identities}) and clipping at $\tau=1$ lower the per-step intrinsic time on every seed.
In panel (e), with $10$ seeds, the 99th percentile of the noise-energy entry of \S\ref{subsec:add-nongaussian} rises from $2.9$ to $5.6\times10^4$ as $\alpha$ falls from $2$ to $1$.
At $\alpha=1$ truncation at $R=2$ brings its mean from $1.5\times10^6$ to $2.2$; Catoni's score brings it to $25$.

\begin{figure}[ht]
\centering
\includegraphics[width=\linewidth]{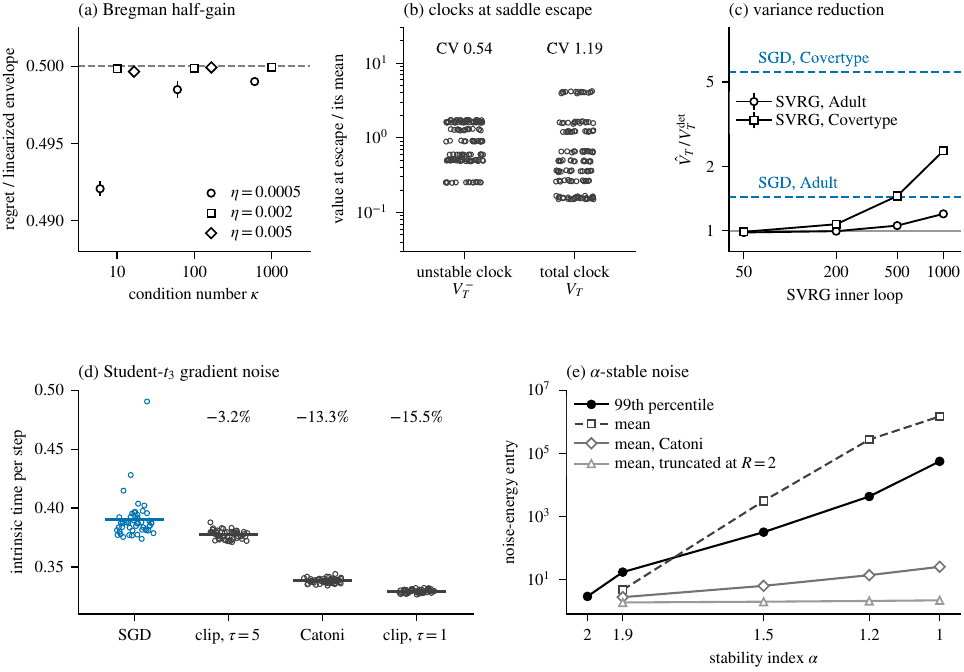}
\caption{\textbf{Further terms of the information ledger behave as its identities imply.}
(a) Realized regret over $I_1+\sum_t\eta_t^2Q_t$ over the late half of each run, for noisy SGD on quadratics with $d=100$ and $\sigma_g=0.5$, by condition number and step size. The dashed line is one half. Error bars are standard errors over seeds.
(b) The unstable clock $V_T^-$ and the total clock $V_T$ at escape from strict saddle points in $d=50$ with $\lambda_{\min}\in\{-0.1,-0.3,-1\}$, three noise scales and two initializations, each divided by its mean over the runs. CV is the coefficient of variation.
(c) Realized over deterministic intrinsic time for SVRG on logistic regression (UCI Adult and Covertype) against its inner-loop length, with SGD dashed. Error bars are $95\%$ bootstrap intervals.
(d) Per-step intrinsic time for logistic regression under Student-$t_3$ gradient noise over $T=20{,}000$ steps, for plain SGD, clipping and Catoni's score. Circles are seeds, bars their mean, and each percentage is the change of the mean against SGD.
(e) The realized noise-energy entry on a $d=10$ quadratic against the stability index of $\alpha$-stable noise: its 99th percentile, and its mean without robustification, under Catoni's score and truncated at $R=2$.}
\label{fig:single-term-checks}
\end{figure}

In Figure~\ref{fig:native-tails}, three transformer architectures ($2$ blocks at width $128$, $4$ at width $96$ and $4$ at width $128$) train on $400{,}000$ characters of text at batch $32$ for $1{,}200$ steps, with $3$ seeds each.
Each coordinate of the gradient noise is divided by its own robust scale, estimated on a held-out half of $256$ probes.
The Hill tail index of the scaled noise, pooled over coordinates, stays above $3.9$ at every tail fraction.
Along the mean-gradient direction the index, $5.39$ to $6.05$ over $2{,}048$ mini-batches, lies inside the $5.11$ to $7.23$ band of Gaussian samples of that size.
Without the scaling the pooled index falls to about $2$, as it does for Gaussian noise with the measured scales, which span $1.2$ decades.
On this substrate the per-coordinate gradient noise at batch $32$ is not heavy-tailed.
A tail index pooled over coordinates of different scales is not by itself evidence of a heavy tail.
A Gaussian with the measured covariance reproduces the variance ratio of Figure~\ref{fig:native-tails}(b) to within $18\%$, so the ratio measures correlation across coordinates.

\begin{figure}[ht]
\centering
\includegraphics[width=\linewidth]{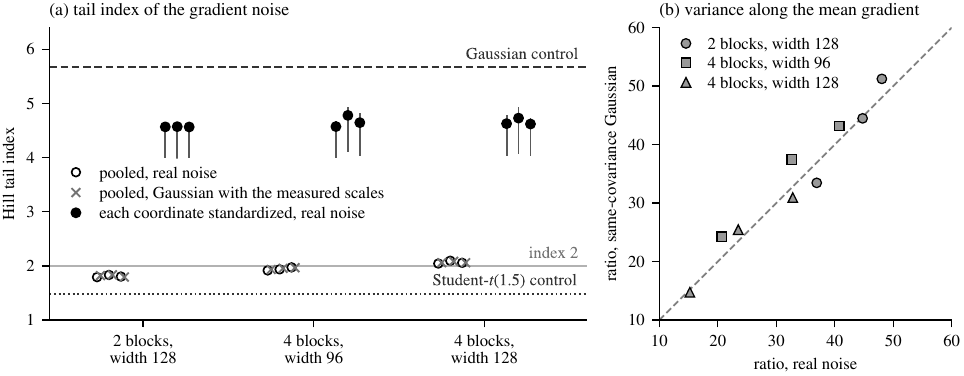}
\caption{\textbf{Scaled coordinate by coordinate, the gradient noise of a character-level transformer at batch 32 is not heavy-tailed.}
(a) Hill tail index at a $5\%$ tail fraction for three architectures, one marker per seed. Open circles pool the raw noise over coordinates, and crosses pool Gaussian noise with the measured per-coordinate scales. Filled circles pool the noise after each coordinate is divided by its own robust scale, with bars over tail fractions from $1\%$ to $10\%$. Lines mark the standardized Gaussian and Student-$t$(1.5) controls and index $2$.
(b) Ratio of the noise variance along the mean-gradient direction to its diagonal approximation, for the real noise and for Gaussian noise with the same covariance. The dashed line is equality.}
\label{fig:native-tails}
\end{figure}

\subsection{The intrinsic clock against iteration count}\label{sec:eval-vclock-fits}

If $\log V_t=a+\beta\log t$ on a fit window, a regression of log-loss on $\log V_t$ is an affine reparameterization of the regression on $\log t$ with the same $R^2$, and an offset power law fitted against $V$ returns the same forecast as one fitted against $t$.
The clock can differ from the step counter only through the part of $\log V_t$ that is not affine in $\log t$, measured on each window by the defect $1-r^2$ of the regression of $\log V_t$ on $\log t$.

\begin{figure}[ht]
\centering
\includegraphics[width=\linewidth]{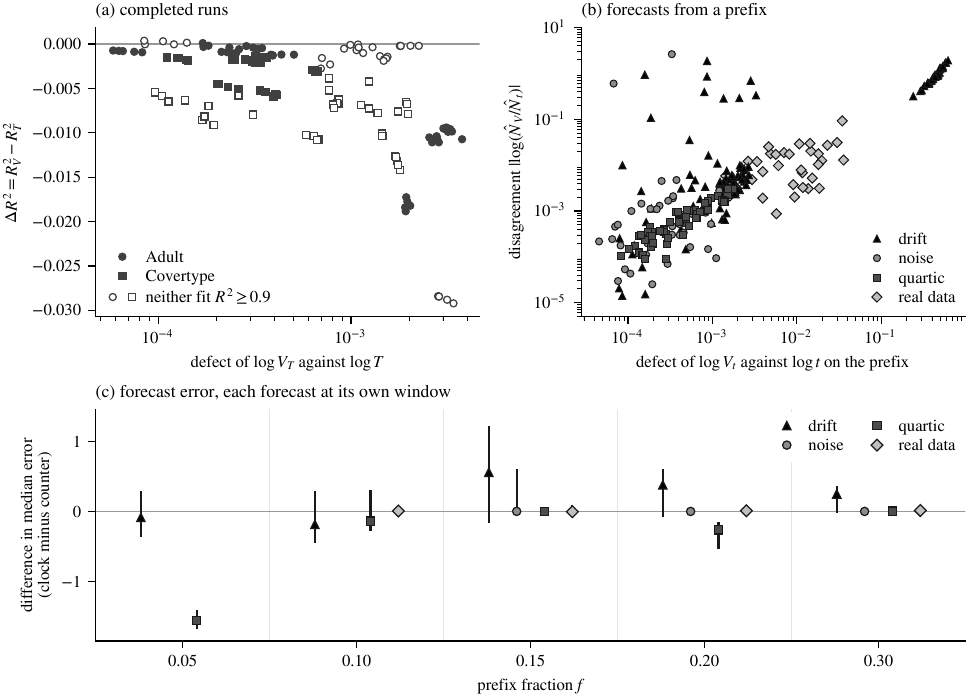}
\caption{\textbf{The intrinsic clock and its step-weighted form seldom fit or forecast the loss better than the step counter.}
Each forecast $\hat N$ estimates the step $N^\star$ at which the run reaches a target loss.
In (b) and (c) the clock is the step-weighted sum $\sum_{s\le t}\eta_s^2Q_s$.
(a) Completed runs of the curriculum sweep: $\Delta R^2=R^2_V-R^2_T$ of the two log-loss fits versus the defect $1-r^2$ of $\log V_T$ regressed on $\log T$.
Open markers are runs that neither fit describes with $R^2\ge0.9$.
(b) Held-out prefix instances: the disagreement $|\log(\hat N_V/\hat N_t)|$ of the two forecasts at a common window versus the defect on the prefix.
(c) The difference in median forecast error $|\log(\hat N/N^\star)|$, clock minus step counter, with each forecast at its own fit fraction; below zero the clock is closer to the truth.
Bars are $95\%$ paired intervals, hidden where narrower than the marker.}
\label{fig:clock-affinity}
\end{figure}

Each completed run's log-loss is fitted against $\log V_T$ and against $\log T$ over the late three-quarters of the trajectory.
The fits are compared by $\Delta R^2=R^2_V-R^2_T$.
Heteroscedastic linear regression ($d=50$, $n=10{,}000$) is trained with SGD, AdaGrad, and Adam at $\eta\in\{10^{-3},10^{-2},10^{-1}\}$, with $10$ seeds per configuration.
The mean $\Delta R^2$ is $-0.062$, $-0.007$, and $-0.041$ at the three step sizes.
Iteration count fits better in the four configurations where either fit reaches $R^2\ge0.9$.

On UCI Adult and a 50K subsample of UCI Covertype, SGD, AdaGrad-Norm, Adam, and Polyak-step at batch sizes $\{32,128,512\}$ with $5$ seeds give $120$ runs of $4{,}000$ iterations.
The median $\Delta R^2$ over the $75$ runs that either fit describes is $-0.0018$.

A curriculum knob on the same two datasets orders the examples from an i.i.d.\ shuffle to a smooth easy-to-hard ramp, under SGD, AdaGrad-Norm, and Adam with $5$ seeds, $150$ runs in all.
It raises the median of the non-stationarity statistic $S_{\mathrm{ns}}$, the coefficient of variation of the windowed gradient-norm mean, from $0.18$ to $0.74$.
On the $80$ runs that either fit describes, the median $\Delta R^2$ lies between $-0.0035$ and $-0.0015$ at every knob setting.
Re-run with its original seeds, the sweep has a defect of $\log V_T$ against $\log T$ of at most $0.004$ on all $150$ runs, with a median of $0.0004$ on Adult and $0.0003$ on Covertype over the $80$ described runs (Figure~\ref{fig:clock-affinity}(a)).

A practical form of the comparison fits $L=L_\infty+Ax^{-b}$ to a short prefix of a run, the first fraction $f$ of its logged steps, against $x=V$ and against $x=t$, and inverts each fit for the step $\hat N$ at which a target loss is reached.
In these prefix comparisons $V$ denotes the step-weighted clock $\sum_{s\le t}\eta_s^2Q_s$, the ledger's cumulative intrinsic-time cost.
Under a constant step it is proportional to the intrinsic clock of \eqref{eq:intrinsic-clock}, so the two give the same forecasts.
The score is $|\log(\hat N/N^\star)|$, where $N^\star$ is the step at which the run reaches the target loss.
It is measured on a held-out split.
The fit fraction, the share at the end of the prefix that forms the fit window, is frozen on a development split that shares no condition numbers or seeds (synthetic) and no dataset-architecture pairing (real) with the held-out split.
A comparison at a common window fits both forecasts on the window selected for the clock.
The synthetic instances come from three regimes: quadratic bowls under a decaying step with small gradient noise (drift), quadratic bowls under a constant step with heavy gradient noise (noise), and quartic bowls under a decaying step (quartic).
The real-data instances train multinomial logistic regression and a one-hidden-layer network of width $64$ on a $40{,}000$-point subsample of Covertype and on the $1{,}797$-image digits set.
Across $241$ (instance, prefix fraction) pairs from $81$ held-out instances of the synthetic drift, noise, and quartic regimes and the real-data configurations, the two forecasts at a common window disagree more where the defect is larger, at Spearman $0.81$ ($95\%$ interval $[0.72,0.88]$, resampling instances; Figure~\ref{fig:clock-affinity}(b)).
Some instances with a small defect still disagree by more than $0.1$, since the fit selects its asymptote from a discrete grid on which a small defect can change the selected value.
On the real data and in the noise and quartic regimes the defect stays below $0.04$.
In those regimes the median errors of the two forecasts at a common window agree within $0.012$ at every prefix fraction where both forecasts answer on four or more instances.
The defect is large only at the shortest drift prefix, where the window selected for the clock spans the whole prefix and $V_t$ gains most of its prefix value in the first three steps (defect $0.24$ to $0.62$).
There, at the clock's window, the clock's forecast is closer to the truth on all $32$ instances, by $0.88$ in median error ($95\%$ interval $[0.66,1.15]$).
With each forecast's fit fraction chosen separately on the development split, the paired comparisons include zero at every drift and real-data prefix (Figure~\ref{fig:clock-affinity}(c)).
On quartic bowls the clock is then ahead at $f=0.05$ and $f=0.20$, and by at most $0.003$ at $f=0.15$.
At $f=0.20$ the two forecasts coincide at a common window wherever both answer, so the difference comes from the wider window selected for the step counter.
At $f=0.05$ the step-counter fit at the clock's window declares a reachable target unreachable on every instance.
In the drift regime at $f=0.05$, the step-counter fit on the window selected for it on the development split declares a reachable target unreachable on $21$ of $32$ instances.

Each of six real-data sweeps trains logistic regression on Covertype or on the digits set, at one seed, under six optimizer settings of SGD, momentum and Adam.
Ranking the optimizers of a sweep at an equal step-weighted clock $\sum_s\eta_s^2Q_s$ runs against their final-loss order.
Over six sweeps and prefix fractions from $0.05$ to $0.30$, the median Spearman correlation lies between $-0.40$ and $-0.26$.
The median correlation of six random rankings falls that low with probability $0.05$ and $0.15$.
The common budget, the largest value of that clock reached by every optimizer, is set in every sweep by Adam at $\eta=0.003$, the setting with the smallest step.
That optimizer is measured at the end of its prefix and most others at their first logged step, so it ranks first on every sweep at $f=0.05$ and first or second at longer prefixes, but finishes fifth or sixth of six.
At an equal intrinsic clock $V_t$ of \eqref{eq:intrinsic-clock}, computed in each optimizer's own metric, the same ranking follows the final-loss order.
Its median Spearman correlation is $+0.89$ to $+0.94$.
From $f=0.10$ its top-ranked setting is the best in five of six sweeps.
Ranking by prefix loss reaches a median correlation of $+0.94$ to $+1.00$.

\subsection{Optimizers derived from the information ledger}\label{app:interventions}

The pressure-cycled controller of \S\ref{sec:pressure-cycle} runs on a synthetic potential and three real networks (Figures~\ref{fig:derived-optimizers} and~\ref{fig:real-problems}(c)).
On the potential and ResNet-20 its escape takes the large step of \eqref{eq:escape-schedule}, with an isotropic injection on the potential and without one on ResNet-20.
On the two transformers the escape is a one-shot isotropic kick.
On grokking, a curvature-targeted kick along the bottom Hessian eigenvectors is compared with the isotropic one.
The potential ($d=8$) is a two-dimensional negative-log Gaussian mixture with a shallow and a deep basin plus an ill-conditioned quadratic.
Every run starts in the shallow basin.
On grokking, a two-layer transformer learns addition modulo $97$ from half the table, the setting of \cite{power2022grokking}, with AdamW at the learning rate and weight decay of that setting as the baseline.
A vision transformer of six blocks at width $192$ trains from scratch on CIFAR-10 for $25$ epochs, and ResNet-20 on CIFAR-10 for $200$ epochs after a high-learning-rate start shared by every schedule.
On both transformers every optimizer clips gradients at unit norm.
Intervals are paired percentile-bootstrap intervals over seeds.

On the potential the controller's output iterate reaches $f-f^\star<10^{-3}$ on all $16$ seeds, in a median $553$ steps, while no baseline's best iterate reaches it on any seed.
The controller's first escape finds the deep basin.
Its second begins at the minimum, where no lower value exists to end it, so the last iterate sits in escape mode on every seed, at a mean $f-f^\star$ of $4.44$.
On grokking ($5$ seeds, evaluated every $100$ steps) the escape delays AdamW by one evaluation interval on four seeds and leaves the fifth unchanged.
SGD at the fixed step $\eta=0.05$ groks $720$ steps earlier than AdamW ($[-920,-540]$).
Replacing the fixed step by the pressure-target rule moves the grokking step by at most $200$ steps on any seed (mean $+20$).
On the vision transformer ($3$ seeds) the escape lowers AdamW's best validation accuracy on every seed, by $0.009$ to $0.061$. On the same network pressure-target SGD matches the fixed step $\eta=0.1$.
Kicks along the Lanczos bottom eigenvectors do not speed grokking against isotropic kicks or none ($3$ seeds, Figure~\ref{fig:derived-optimizers}(d)).
On ResNet-20 ($2$ seeds) the controller escapes $82$ and $53$ times and trains worst of the four schedules (Figure~\ref{fig:real-problems}(c)).

\begin{figure}[ht]
\centering
\includegraphics[width=\linewidth]{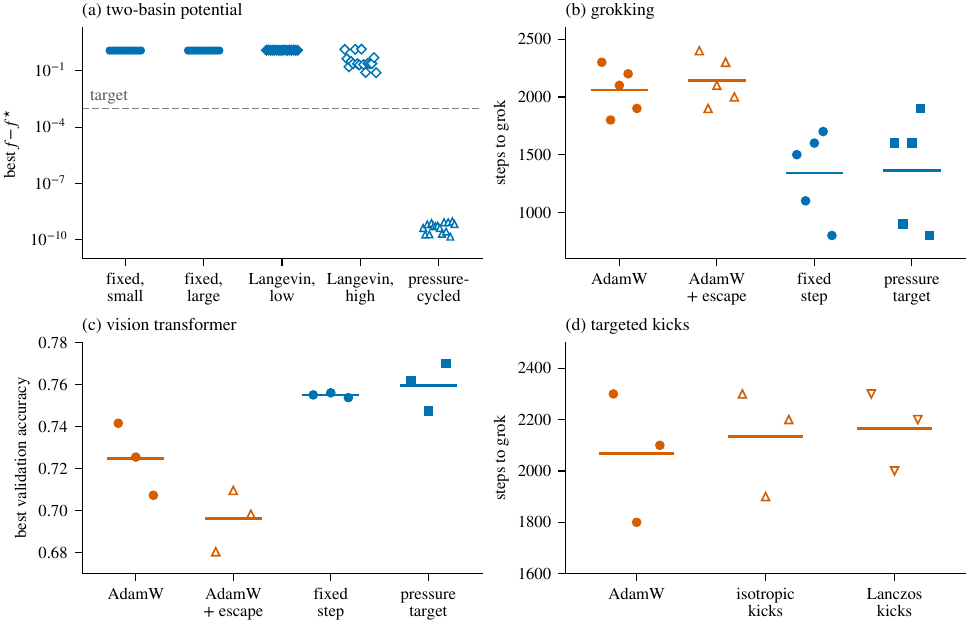}
\caption{\textbf{The pressure-cycled controller finds the deep basin of a synthetic potential on every seed but slows or degrades training on real networks.}
(a) Best-iterate $f-f^\star$ on the two-basin potential for fixed-step and constant-Langevin SGD and for the controller. The dashed line is the target.
(b) Steps to grok for AdamW with and without the escape, and for SGD at a fixed and at a pressure-target step.
(c) Best validation accuracy of the vision transformer on CIFAR-10 for the same four.
(d) Steps to grok for AdamW without kicks and with isotropic or curvature-targeted kicks.
Markers are seeds and bars their means. Color gives the base optimizer, and an open triangle marks a run with the escape phase.}
\label{fig:derived-optimizers}
\end{figure}

On ResNet-20 the trigger computes $\rho_t$ \eqref{eq:rho-trigger} from smoothed consecutive mini-batch losses, so its sign is noise.
The per-step return is negative on $46$--$49\%$ of logged steps for every schedule.
On the two transformers the trigger is a patience rule on the running best training loss.
On the vision transformer it fires although the mean training loss falls in every epoch.
On full-batch grokking it first fires only after the network fits its training set, and more often after grokking than before ($3.9$--$4.9$ against $0$--$3.5$ kicks per thousand steps).
On a deterministic landscape the trigger measures a true objective decrease, but there it does not distinguish a stuck run from a finished one: the objective stops decreasing at a converged minimum just as on a plateau.
The pressure-target step reduces to its stability cap: on grokking it sits at the cap on $90$--$97\%$ of steps, so at matched scale it coincides with the fixed step on both transformers.
Without the unit-norm clip neither the fixed step nor the pressure-target rule groks on any seed.
On ResNet-20 the controller's descent step sits at its floor of $10^{-4}$ on $70\%$ and $81\%$ of descent steps after step $1{,}000$.

\subsection{Sharpness and the generalization gap}\label{app:cifar-forensics}

CIFAR-10 \citep{krizhevsky2009} is used with its standard split, per-channel normalization, random crops with $4$-pixel padding, and horizontal flips.
ResNet-20 \citep{kaiming2016deep} ($272{,}474$ parameters) and Wide-ResNet-28-2 and 28-4 ($1.47$M and $5.85$M) train for $200$ epochs on a cosine schedule unless a row of Figure~\ref{fig:sharpness-gap} says otherwise.
Each endpoint is probed on a fixed image subset, $512$ images on the Wide-ResNets, with a Lanczos top eigenvalue, a Hutchinson trace, and $-\log\det\widetilde H$, the negative log-determinant over the top-$K$ Ritz values with negative values clamped to a small floor (\S\ref{rem:logdet-practical}).
The gap is test minus training cross-entropy.

Along the batch axis every statistic has the sign the flatness bound of \S\ref{sec:flatness} implies, at both depths (Figure~\ref{fig:sharpness-gap}).
On ResNet-20 the batch-$8192$ endpoints are sharper than the batch-$32$ ones ($\lambda_{\max}$ $613$ against $94$) and have the larger gap ($0.399$ against $0.229$).
On Wide-ResNet-28-4, batch $8192$ against $128$ gives $\lambda_{\max}$ $19.0$ against $9.3$ and gap $0.306$ against $0.180$, with learning rates scaled with the square root of the batch size.
On ResNet-20 at one batch size, $20$ seeds show no correlation.
Across them the gap has a standard deviation of $2.5\%$ of its mean and the test error a standard deviation of $16$ of the $10{,}000$ test images, below one model's binomial sampling error of $26$ images.
A two-sided $5\%$ Spearman test on $20$ runs detects a correlation of $0.6$ with probability $0.8$.
In the pooled ResNet-20 sweep of $45$ runs the bootstrap intervals of all three statistics exclude zero.
That sweep varies batch size together with the optimizer, so its correlations do not isolate the optimizer axis.

\begin{figure}[ht]
\centering
\includegraphics[width=\linewidth]{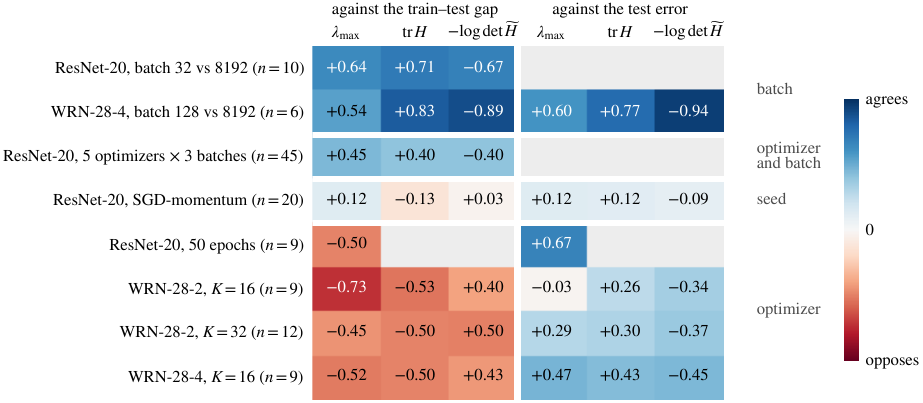}
\caption{\textbf{Sharpness tracks the generalization gap with the sign the flatness bound implies along the batch axis, but not across optimizers.}
Each entry is the Spearman correlation of a sharpness statistic with the train--test gap or with the test error over the $n$ endpoints of one sweep.
Its color is that correlation times the sign the bound implies ($+$ for $\lambda_{\max}$ and $\tr H$, $-$ for $-\log\det\widetilde H$), so blue agrees with the bound and red opposes it.
Gray entries were not measured.
On ResNet-20 at $50$ epochs the gap is the final test loss minus the last logged training loss.
$K$ is the Lanczos rank.}
\label{fig:sharpness-gap}
\end{figure}

Across optimizers the bare statistics are compared at different endpoints.
They also depend on the coordinates in which each endpoint is written.
At a fixed Wide-ResNet-28-2 endpoint a function-preserving rescaling, which leaves the test loss unchanged, moves $-\log\det\widetilde H$ by up to $99$ nats, as far as the distance between adjacent optimizers in a $241$-nat spread across the $12$ endpoints.
At both Wide-ResNet widths every such statistic takes the opposite sign against the gap, none of the intervals excluding zero.
One optimizer accounts for the sign: without Lion, eight of the nine cross-optimizer correlations take the implied sign.
Lion is the flattest optimizer in all three sweeps, with optimizer-mean $\lambda_{\max}$ of $0.6$--$1.7$ against $9$--$103$ for SGD-momentum and Adam.
It has the widest gap ($0.52$--$0.58$ against $0.18$--$0.32$) at a test error no higher than Adam's.
It fits its training set most tightly, with training losses of $6\times10^{-6}$ to $3.5\times10^{-4}$ against about $10^{-3}$ for SGD-momentum and $3$--$4\times10^{-3}$ for Adam.
At a correctly classified training point the curvature of the cross-entropy in the network outputs, $1-\sum_kp_k^2$, lies between $\ell-\ell^2/2$ and $2\ell$ for the example's loss $\ell$, so it scales with the training loss.
Across the endpoints each statistic orders the optimizers by training loss (Spearman $0.82$ to $0.97$ in magnitude).
The training loss orders them against the gap ($-0.47$ to $-0.52$).
Along the Wide-ResNet-28-4 batch axis the training losses agree within about $12\%$ on average and $22\%$ across endpoints.
Along that axis every statistic has the implied sign.
In a ResNet-20 study of momentum at $50$ epochs, plain SGD is the sharpest optimizer ($\lambda_{\max}$ $960$--$1{,}097$ against $158$--$202$) and has the smallest gap, measured against the training loss logged in the last iterations.
Momentum lowers test error there from $11.5\%$ to $9.1$--$9.2\%$.
Against test error eight of the nine cross-optimizer Wide-ResNet correlations take the implied sign, none excluding zero.

In function space, where a function-preserving rescaling moves nothing, the cross-optimizer correlation on ResNet-20 takes the implied sign at batch $32$ and the opposite sign at batch $8192$.
A panel trains SGD, heavy-ball, and Nesterov at batches $32$ and $8192$ for $50$ epochs with five seeds.
A function-preserving rescaling moves none of its function-space observables in any of the $30$ runs and the natural-gradient norm by at most $8.2\%$, against up to $1122\%$ for the log-spectrum of $\mathcal F^{-1}H$ estimated with $10$ conjugate-gradient iterations.
The function-space observables are built from $\bar H_L$, the input-averaged Hessian of the loss in the network outputs.
On the $15$ batch-$32$ runs they correlate with the gap with the implied sign: $+0.58$ for $\log\det\bar H_L$ and $+0.60$ for its top eigenvalue.
The trace of $\bar H_L$ gives $+0.49$, the natural-gradient norm $+0.41$, and the top-$128$ Hessian log-determinant $+0.66$ on the same runs.
The correlations come from the difference between optimizers.
Within an optimizer five seeds give none that is reliable.
On the $15$ batch-$8192$ runs, which stop at training losses of $0.62$ to $1.28$, the three observables built from $\bar H_L$ correlate with the gap at $-0.60$ to $-0.76$.
At both batch sizes these observables track the training loss (Spearman $+0.69$ to $+0.94$), while the gap's correlation with the training loss changes sign, from $+0.44$ at batch $32$ to $-0.55$ at batch $8192$.
Fresh Wide-ResNet-28-4 runs, one per optimizer, keep Lion the flattest optimizer in function space as well.
Lion's $\tr\bar H_L$ is $1.6\times10^{-5}$ against $1.8\times10^{-3}$ for SGD-momentum and $5.5\times10^{-3}$ for Adam.
Lion's gap is $0.51$ against $0.19$ and $0.32$.

At the $12$ Wide-ResNet-28-2 endpoints of the $K=32$ sweep, with $N=50{,}000$, $\delta=0.05$, $\alpha=\sigma^2=1$, and about $78{,}200$ steps, a rank-$32$ spectrum does not determine the relative entropy \eqref{eq:sharpness-kl} of the certificate.
The probed Hessians have negative eigenvalues.
For a positive-definite surrogate the floor over the prior scale depends on all $d$ eigenvalues.
The term $\|\theta_{T+1}\|^2/(2\sigma^2)$ is left out throughout, since the endpoints' parameter norms were not recorded.
Setting the $d-32$ unprobed eigenvalues to one and replacing the inverse trace $\tr(H^{-1})$ by its lower bound $d/\lambda_{\max}$ leaves the common term $\half d(\log(\alpha T\sigma^2)-1)\approx7.5\times10^6$ nats.
That term puts the bound \eqref{eq:pac-bound} within $1.3\times10^{-4}$ of $8.68$ at all $12$ endpoints.
A flat spectrum at the prior's scale, where the exact bound is $0.012$, evaluates to $9.09$ by the same rule, so the value does not measure the endpoints.
Two terms vary across the endpoints: the rank-$32$ basin term $\half\log\det\widetilde H$ over $[-119,+1.5]$ nats and the inverse-trace term over $[0.12,31.5]$ nats.
The certificate ranks the $12$ endpoints as the basin term does, decreasing in $-\log\det\widetilde H$, so it orders the optimizers against the gap ($\rho(\KL,\mathrm{gap})=-0.50$).

\section{Proofs of deferred statements}\label{app:proofs-deferred}

\begin{proof}[Proof of Proposition~\ref{prop:gaussian-tilt}]
The moment generating function of a Gaussian gives
\[
Z_t=\E_{\theta\sim\calN(\theta_t,\Sigma)}\exp(-\eta_t g_t^\top\theta)
=\exp\left(-\eta_t g_t^\top\theta_t+\frac{\eta_t^2}{2}g_t^\top\Sigma g_t\right)
\]
Completing the square in the density of $q_t(d\theta)\exp(-\eta_t g_t^\top\theta)$ yields a Gaussian with covariance $\Sigma$ and mean $\theta_t-\eta_t\Sigma g_t$.
\end{proof}

\begin{proof}[Proof of Theorem~\ref{thm:sgd-ledger}]
Since the comparator and the iterate share the covariance $\Sigma$, $I_t(u)=\KL(\rho_u\|q_t)=\half(u-\theta_t)^\top\Sigma^{-1}(u-\theta_t)$. Expanding $I_{t+1}(u)$ at the updated iterate $\theta_{t+1}=\theta_t-\eta_t\Sigma g_t$,
\[
I_{t+1}(u)
=\half(u-\theta_t+\eta_t\Sigma g_t)^\top\Sigma^{-1}(u-\theta_t+\eta_t\Sigma g_t)
=I_t(u)+\eta_t g_t^\top(u-\theta_t)+\half\eta_t^2 g_t^\top\Sigma g_t
\]
using $\Sigma^\top=\Sigma$ and $(\Sigma g_t)^\top\Sigma^{-1}=g_t^\top$. Rearranging and writing $Q_t=\half g_t^\top\Sigma g_t$ gives $I_t(u)-I_{t+1}(u)=\eta_t g_t^\top(\theta_t-u)-\eta_t^2Q_t$, i.e.\ \eqref{eq:weighted-ledger}; dividing by $\eta_t$ gives \eqref{eq:one-step-ledger} and summing gives \eqref{eq:sum-weighted-ledger}.
\end{proof}

\begin{proof}[Proof of Theorem~\ref{thm:loss-identity}]
Apply \eqref{eq:bregman-exact} with $f=f_t$ and $\theta=\theta_t$, then substitute into \eqref{eq:sum-weighted-ledger}.
\end{proof}

\begin{proof}[Proof of Corollary~\ref{cor:stochastic-identity}]
Take conditional expectations given $\calF_t$ in the pathwise identity \eqref{eq:weighted-ledger}. Since $\theta_t$ is $\calF_t$-measurable and $u$ is a fixed constant, $I_t(u)=\half\|u-\theta_t\|_{\Sigma^{-1}}^2$ is $\calF_t$-measurable and passes through the expectation unchanged. On the right side of \eqref{eq:weighted-ledger}, the term $I_{t+1}(u)=\half\|u-\theta_{t+1}\|_{\Sigma^{-1}}^2$ depends on $g_t$ through the update $\theta_{t+1}=\theta_t-\eta_t\Sigma g_t$, so its expectation $\E[I_{t+1}(u)\mid\calF_t]$ is taken over $g_t$. On the left side, $\E[\ip{g_t}{\theta_t-u}\mid\calF_t]=\ip{\E[g_t\mid\calF_t]}{\theta_t-u}=\ip{\nabla F(\theta_t)}{\theta_t-u}$ by unbiasedness and $\calF_t$-measurability of $\theta_t$; applying \eqref{eq:bregman-exact} with $\theta=\theta_t$, $f=F$ gives $\ip{\nabla F(\theta_t)}{\theta_t-u}=F(\theta_t)-F(u)+B_F(u,\theta_t)$. For the intrinsic-time term, the Gaussian-quadratic-form identity and the variance decomposition give
\[
\E[Q_t\mid\calF_t]
=\half\E[g_t^\top\Sigma g_t\mid\calF_t]
=\half\nabla F(\theta_t)^\top\Sigma\nabla F(\theta_t)+\half\tr(\Sigma C_t)
\]
where the second equality uses $\E[\|g_t\|_\Sigma^2\mid\calF_t]=\|\E[g_t\mid\calF_t]\|_\Sigma^2+\tr(\Sigma\Cov(g_t\mid\calF_t))$. Taking unconditional expectation (the tower property gives $\E[I_t(u)-\E[I_{t+1}(u)\mid\calF_t]]=\E[I_t(u)]-\E[I_{t+1}(u)]$) and summing $t=1,\dots,T$ telescopes the information terms to $I_1(u)-\E[I_{T+1}(u)]$ and gives the claimed identity.
\end{proof}

\begin{proof}[Proof of Proposition~\ref{prop:tracking-ledger}]
Apply the one-step identity \eqref{eq:weighted-ledger} at round $t$ with the comparator $u_t$ and sum over $t$:
\[
\sum_{t=1}^T\eta_t\ip{g_t}{\theta_t-u_t}
=\sum_{t=1}^T\bigl[I_t(u_t)-I_{t+1}(u_t)\bigr]+\sum_{t=1}^T\eta_t^2Q_t
\]
Re-indexing the first sum,
\[
\sum_{t=1}^T\bigl[I_t(u_t)-I_{t+1}(u_t)\bigr]
=I_1(u_1)-I_{T+1}(u_T)+\sum_{t=1}^{T-1}\bigl[I_{t+1}(u_{t+1})-I_{t+1}(u_t)\bigr]
\]
and the trailing sum is $D_T$, which gives \eqref{eq:tracking-ledger}. The closed form for each drift summand is the difference of two squared $\Sigma^{-1}$-norms sharing the center $\theta_{t+1}$: $I_{t+1}(u_{t+1})-I_{t+1}(u_t)=\half\ip{u_{t+1}-u_t}{u_{t+1}+u_t-2\theta_{t+1}}_{\Sigma^{-1}}$. Cauchy--Schwarz in the $\Sigma^{-1}$ inner product bounds each summand by $\half\|u_{t+1}-u_t\|_{\Sigma^{-1}}\,\|u_{t+1}+u_t-2\theta_{t+1}\|_{\Sigma^{-1}}$, and summing over $t$ gives the stated bound on $|D_T|$.
\end{proof}

\begin{proof}[Proof of Proposition~\ref{prop:zo-moments}]
The mean is $\E[\ip{g_t}{\nu_t}\nu_t]=\E[\nu_t\nu_t^\top]g_t=Cg_t$, using only the second moment, so \eqref{eq:zo-mean} holds for any $\calP$. The finite-$\mu$ form is Stein's identity: $\E[f(\theta_t+\mu \nu_t)\nu_t]=\mu C\,\E[\nabla f(\theta_t+\mu \nu_t)]=\mu C\nabla f_\mu(\theta_t)$ for $\nu_t\sim\calN(0,C)$, and $\E[f(\theta_t)\nu_t]=0$. For the second moment with $\nu_t\sim\calN(0,C)$, Isserlis' theorem gives $\E[\nu_i\nu_j(g^\top \nu)^2]=C_{ij}(g^\top Cg)+2(Cg)_i(Cg)_j$, which is \eqref{eq:zo-second}; \eqref{eq:zo-intrinsic} is $\half\tr(\Sigma\,\E[\hat g_t\hat g_t^\top])$.
\end{proof}

\begin{proof}[Proof of Proposition~\ref{prop:saddle}]
In the eigenbasis, $z_{t+1,j}=(1-\eta\lambda_j)z_{t,j}$ deterministically, which proves \eqref{eq:saddle-det}. With gradient noise $z_{t+1,j}=(1-\eta\lambda_j)z_{t,j}-\eta\xi_{t,j}$.
Squaring, multiplying by $1/2$, and taking conditional expectation gives \eqref{eq:saddle-stoch}.
\end{proof}

\begin{proof}[Proof of Proposition~\ref{prop:proper-gaussian-likelihood}]
The Kalman gain for the likelihood covariance $R_t=\chi_t(1-\chi_t)^{-1}\Sigma$ is $K_t=\Sigma(\Sigma+R_t)^{-1}=(1-\chi_t)I$.
Therefore the posterior mean is
\[
\theta_t+K_t(y_t-\theta_t)=\theta_t-(1-\chi_t)\frac{\eta_t}{1-\chi_t}\Sigma g_t=\theta_t-\eta_t\Sigma g_t
\]
The posterior covariance is $(I-K_t)\Sigma=\chi_t\Sigma$. Adding independent process noise with covariance $(1-\chi_t)\Sigma$ restores covariance $\Sigma$.

For the limiting statement, expand the Gaussian likelihood:
\[
-\half(\theta-y_t)^\top R_t^{-1}(\theta-y_t)
=-\half\theta^\top R_t^{-1}\theta+y_t^\top R_t^{-1}\theta+\text{const}
\]
Since $R_t^{-1}=(1-\chi_t)\chi_t^{-1}\Sigma^{-1}$, the quadratic term vanishes as $\chi_t\uparrow1$. Also
\[
y_t^\top R_t^{-1}\theta
=\left(\theta_t-\frac{\eta_t}{1-\chi_t}\Sigma g_t\right)^\top
\frac{1-\chi_t}{\chi_t}\Sigma^{-1}\theta
\to -\eta_t g_t^\top\theta
\]
Thus the likelihood factor converges to the linear-score factor $\exp(-\eta_t\ip{g_t}{\theta})$.
\end{proof}

\begin{proof}[Proof of Proposition~\ref{prop:mixture-prior}]
Multiply each component by $\exp(-\eta\ip{g}{\theta})$ and normalize. For component $k$, Proposition~\ref{prop:gaussian-tilt} gives
\[
\calN(\mu_k,\Sigma)(d\theta)\,e^{-\eta\ip{g}{\theta}}
=Z_k\,\calN(\mu_k-\eta\Sigma g,\Sigma)(d\theta),
\qquad
\log Z_k=-\eta\ip{g}{\mu_k}+\half\eta^2g^\top\Sigma g
\]
so the reweighted mixture is $\sum_kw_kZ_k\,\calN(\mu_k-\eta\Sigma g,\Sigma)$ up to normalization; dividing by $\sum_jw_jZ_j$ yields the stated posterior and the weights \eqref{eq:mixture-weights}. The per-round normalizer is $Z^{\mathrm{mix}}=\sum_kw_kZ_k$, and normalizers of successive linear-score updates compose multiplicatively, so the cumulative cross-entropy correction over a run is $-\sum_t\log\bigl(\sum_kw_k^{(t)}Z_k^{(t)}\bigr)$, the stated log-sum-exp of per-component free energies.
\end{proof}

\begin{proof}[Proof of Proposition~\ref{prop:cheby-cyclic}]
\textbf{Step 1 (Iterate factorization).} In eigencoordinates the deterministic gradient step $z_{t,i}=(1-\eta_t\lambda_i)z_{t-1,i}=(1-\lambda_i/r_t)z_{t-1,i}$ telescopes to
\[
z_{T,i}=z_{0,i}\cdot\prod_{t=1}^T\!\left(1-\frac{\lambda_i}{r_t}\right)
=z_{0,i}\cdot\frac{\prod_{t=1}^T(r_t-\lambda_i)}{\prod_{t=1}^T r_t}
\]
The numerator is $(-1)^T$ times the monic polynomial in $\lambda_i$ of degree $T$ with roots $\{r_t\}_{t=1}^T$. By construction \eqref{eq:cheby-roots} those are the affine Chebyshev nodes on $[\lambda_{\min},\lambda_{\max}]$, so
\[
\prod_{t=1}^T(r_t-\lambda_i)
=\left(\frac{\lambda_{\max}-\lambda_{\min}}{4}\right)^T\!\cdot 2\,T_T\!\left(\frac{\lambda_{\max}+\lambda_{\min}-2\lambda_i}{\lambda_{\max}-\lambda_{\min}}\right)
\]
(the standard monic-Chebyshev factorization, with the normalization restored by the leading coefficient $2^{T-1}$ of $T_T$ rescaled to the interval $[\lambda_{\min},\lambda_{\max}]$). The denominator $\prod r_t$ evaluates the same identity at $\lambda_i=0$, giving the divisor $T_T((\lambda_{\max}+\lambda_{\min})/(\lambda_{\max}-\lambda_{\min}))$. Combining numerator and denominator yields \eqref{eq:cheby-iterate}.

\textbf{Step 2 (Ledger identity).} The pathwise identity \eqref{eq:weighted-ledger} of Theorem~\ref{thm:sgd-ledger} holds on every round with $\Sigma=I$ and arbitrary $\eta_t$, and summing it gives the linearized form $\sum_t\eta_t\ip{g_{t-1}}{\theta_{t-1}-\theta_\star}=I_0(\theta_\star)-I_T(\theta_\star)+\sum_t\eta_t^2Q_{t-1}$. Substituting the Bregman identity \eqref{eq:bregman-exact} for each round at $\theta=\theta_{t-1}$, $f=F$, $u=\theta_\star$ (which rewrites $\ip{g_{t-1}}{\theta_{t-1}-\theta_\star}$ as $F(\theta_{t-1})-F(\theta_\star)+B_F(\theta_\star,\theta_{t-1})$ on the left, moving the curvature gain onto the objective-level regret) recovers \eqref{eq:cheby-ledger}; this is the Chebyshev-cycle specialization of \eqref{eq:loss-path} restricted to $f_t\equiv F$. The Bregman term sits on the left with the objective regret, and the intrinsic-time cost stays on the right.

\textbf{Step 3 (Geometric-mean contraction).} On the spectral boundary $\lambda_i=\lambda_{\min}$, the affine argument of $T_T$ in \eqref{eq:cheby-iterate} equals $1$, so $T_T(1)=1$ and the contraction factor reduces to $1/T_T(\xi)$ with $\xi=(\lambda_{\max}+\lambda_{\min})/(\lambda_{\max}-\lambda_{\min})=(\kappa+1)/(\kappa-1)$. Writing $\xi=\cosh\phi$ with $\phi=2\,\mathrm{arctanh}(1/\sqrt\kappa)$ gives $T_T(\xi)=\cosh(T\phi)$, so the per-round geometric-mean factor along this eigencoordinate is
\[
\left(\frac{1}{T_T(\xi)}\right)^{1/T}=\left(\cosh(T\phi)\right)^{-1/T}\xrightarrow[T\to\infty]{}\,e^{-\phi}=\frac{\sqrt\kappa-1}{\sqrt\kappa+1}=\rho_{\mathrm{cyc}}
\]
The same limit holds for $\lambda_i=\lambda_{\max}$ (sign flip in the Chebyshev argument). For $\lambda_i$ in the interior of $[\lambda_{\min},\lambda_{\max}]$, $|T_T((\lambda_{\max}+\lambda_{\min}-2\lambda_i)/(\lambda_{\max}-\lambda_{\min}))|\le 1$ for all $T$, so $|z_{T,i}|\le|z_{0,i}|/T_T(\xi)$ in every coordinate and $I_T/I_0\le T_T(\xi)^{-2}$. If $z_0$ has a nonzero component $z_{0,j}$ on an endpoint coordinate, then also $I_T/I_0\ge (z_{0,j}^2/(2I_0))\,T_T(\xi)^{-2}$, and taking the $1/(2T)$ power of both bounds gives \eqref{eq:cheby-geomean}.
\end{proof}

\begin{proof}[Proof of Proposition~\ref{prop:zo-second-diff}]
Equation~\eqref{eq:zo-second-diff} is the second-order Taylor expansion, with the odd terms canceling under the antithetic sum and the remainder vanishing identically for quadratic $f$. Equation~\eqref{eq:zo-hutch} is Isserlis' theorem: $\E[\nu_i\nu_j\nu_k\nu_l]=\delta_{ij}\delta_{kl}+\delta_{ik}\delta_{jl}+\delta_{il}\delta_{jk}$ gives $\E[(\nu^\top H\nu)\nu_i^2]=\sum_{jk}H_{jk}\E[\nu_j\nu_k\nu_i^2]=\tr(H)+2H_{ii}$; summing over $i$ gives $(d+2)\tr(H)$.
\end{proof}

\begin{proof}[Proof of Proposition~\ref{prop:surface-quadratic}]
\textbf{Step 1 (Diagonalization).}
Write $H=U\Lambda U^\top$ with $\Lambda=\diag(\lambda_1,\dots,\lambda_d)$, $\lambda_i>0$, and set $z_t=U^\top(\theta_t-\theta_\star)$.  Under fixed-rate gradient descent with $\Sigma=I$ and step $\eta$, each coordinate satisfies $z_{t+1,i}=(1-\eta\lambda_i)z_{t,i}$.
Stability requires $0<\eta\lambda_i<2$ for all $i$; the choice $\eta_\star=2/(\lambda_{\max}+\lambda_{\min})$ minimizes $\max_i|1-\eta\lambda_i|$ over the spectrum, giving the standard Chebyshev-optimal step.

\textbf{Step 2 (Intrinsic-time sum).}
Since $Q_t=\half\|g_t\|^2=\half\|\Lambda z_t\|^2=\half\sum_i\lambda_i^2 z_{t,i}^2$ and $z_{t,i}=(1-\eta\lambda_i)^tz_{0,i}$,
\[
\sum_{t=0}^{T-1}\eta^2Q_t
=\frac{\eta^2}{2}\sum_i\lambda_i^2 z_{0,i}^2\sum_{t=0}^{T-1}(1-\eta\lambda_i)^{2t}
=\frac{\eta^2}{2}\sum_i\lambda_i^2 z_{0,i}^2\cdot\frac{1-(1-\eta\lambda_i)^{2T}}{\eta\lambda_i(2-\eta\lambda_i)}
\]
As $T\to\infty$ (and using $0<\eta\lambda_i<2$ for stability), this is a convergent geometric series in each coordinate, so the cumulative intrinsic-time cost is finite and uniformly bounded in $T$:
\[
\sum_{t=0}^{\infty}\eta^2Q_t
=\frac{\eta}{2}\sum_i\frac{\lambda_i}{2-\eta\lambda_i}\,z_{0,i}^2
=\half\sum_i\frac{\eta\lambda_i}{2-\eta\lambda_i}\,z_{0,i}^2
\]
Each summand is finite at $\eta_\star=2/(\lambda_{\max}+\lambda_{\min})$ because $2-\eta_\star\lambda_i>0$ for every eigenvalue in the stable range; the cost is therefore a fixed $O(1)$ quantity that does not grow with the horizon.

\textbf{Step 3 (Net path term).}
The first-order Bregman remainder is $B_F(\theta_\star,\theta_t)=\half\|\theta_t-\theta_\star\|_H^2=\half\sum_i\lambda_i z_{t,i}^2$, and $\sum_{t=0}^{\infty}\eta B_F(\theta_\star,\theta_t)=\half\sum_i z_{0,i}^2/(2-\eta\lambda_i)$ is likewise a convergent geometric sum.  The two appear in $\calK_T(\varepsilon)$ with opposite signs, so the relevant quantity is the net path term, which has the closed form
\[
\sum_{t=0}^{\infty}\eta^2Q_t-\sum_{t=0}^{\infty}\eta B_F(\theta_\star,\theta_t)
=\half\sum_i\frac{\eta\lambda_i-1}{2-\eta\lambda_i}\,z_{0,i}^2
\]
a finite $O(1)$ quantity (signed, of either sign coordinate-by-coordinate). It is not dominated by the terminal log-determinant of Step 4 once the spectrum is ill-conditioned; see the discussion following the proposition.

\textbf{Step 4 (Terminal relative entropy and feasibility).}
Take $q_T=\calN(\theta_T,\Sigma_\star)$ with $\Sigma_\star=U\diag\bigl(1/(\beta\lambda_i+\tau^{-2})\bigr)U^\top$, the free-energy-optimal covariance \eqref{eq:free-energy-cov}. Then $\E_{q_T}F-F_\star=\half(\theta_T-\theta_\star)^\top H(\theta_T-\theta_\star)+\half\sum_i\lambda_i/(\beta\lambda_i+\tau^{-2})$, and since $\theta_T\to\theta_\star$ the strict feasibility hypothesis makes $q_T$ admissible for all large $T$. The relative entropy to $\pi=\calN(0,\tau^2I)$ is
\[
\KL(q_T\|\pi)
=\frac{\|\theta_T\|^2}{2\tau^2}+\half\sum_{i=1}^d\left[\log(1+\beta\tau^2\lambda_i)+\frac{1}{1+\beta\tau^2\lambda_i}-1\right]
\le\frac{\|\theta_T\|^2}{2\tau^2}+\half\sum_{i=1}^d\log(1+\beta\tau^2\lambda_i)
\]
Adding the net path term of Step 3, which the fixed-rate policy at $\eta_\star$ attains as $T\to\infty$, and letting $\theta_T\to\theta_\star$ gives the stated bound on $\limsup_{T\to\infty}\calK_T(\varepsilon)$.
\end{proof}

\begin{proof}[Proof of Corollary~\ref{cor:sam-ledger}]
The pathwise identity \eqref{eq:weighted-ledger} of Theorem~\ref{thm:sgd-ledger} holds for any score $c_t$ used to drive the update; in particular, applied with the score $\check c_t(\theta)=\ip{\check g_t}{\theta}$ and update \eqref{eq:sam-update} it gives
\[
\eta_t\ip{\check g_t}{\theta_t-u}=I_t(u)-I_{t+1}(u)+\eta_t^2\cdot\half\check g_t^\top\Sigma\check g_t
\]
For the closed form of $\half\check g_t^\top\Sigma\check g_t$, write $\check g_t=g_t+\rho_t\bar H_t\hat g_t$ from the fundamental theorem of calculus applied to $\nabla f_t$ along the segment $\theta_t+s\rho_t\hat g_t$, $s\in[0,1]$ (exact, since $\bar H_t$ is the path-averaged Hessian along that segment). Expanding the quadratic and using $\bar H_t^\top=\bar H_t$ gives
\[
\half\check g_t^\top\Sigma\check g_t
=\half g_t^\top\Sigma g_t
+\rho_t\hat g_t^\top\bar H_t\Sigma g_t
+\half\rho_t^2\hat g_t^\top\bar H_t\Sigma\bar H_t\hat g_t
\]
Substituting $\hat g_t=g_t/\|g_t\|$ in the last two terms gives \eqref{eq:sam-Q}. (When $g_t=0$, $\hat g_t=0$, $\check g_t=g_t=0$, and the corollary reduces to a vacuous step.)
\end{proof}

\begin{proof}[Proof of Theorem~\ref{thm:add-noisy-ledger}]
Let \(e_t=\theta_t-u\).  Since \(\theta_{t+1}-u=e_t-\eta_t\Sigma g_t+\xi_t\), \(I_{t+1}(u)=\half\|e_t-\eta_t\Sigma g_t+\xi_t\|_{\Sigma^{-1}}^2\).
Expanding the square and rearranging gives \eqref{eq:add-noisy-ledger-expanded}.  The compact form \eqref{eq:add-noisy-ledger} follows from \(\bar\theta_{t+1}-u=e_t-\eta_t\Sigma g_t\).  The conditional expectation removes the cross term and replaces \(\half\|\xi_t\|_{\Sigma^{-1}}^2\) by \(\half\tr(\Sigma^{-1}\Omega_t)\).
\end{proof}

\begin{proof}[Proof of Corollary~\ref{cor:add-noisy-objective-ledger}]
Use $\ip{\nabla F(\theta_t)}{\theta_t-u} =F(\theta_t)-F(u)+B_F(u,\theta_t)$ and subtract \(\eta_t\ip{\varepsilon_t}{\theta_t-u}\) from the right side of \eqref{eq:add-noisy-ledger}.
Taking conditional expectations gives \eqref{eq:add-noisy-objective-expected}; the identity \eqref{eq:add-expected-gradient-clock} follows by expanding \(\E[g_t^\top\Sigma g_t\mid\calF_t]\).
\end{proof}

\begin{proof}[Proof of Theorem~\ref{thm:add-channel-kl}]
Apply the chain rule for relative entropy to the joint $(X,Y)$ in two different orderings.

\emph{First ordering (condition on $X$):} The joint $P\otimes K$ has marginal $P$ for $X$ and conditional $K(X,\cdot)$ for $Y\mid X$.  The joint $Q\otimes K$ has marginal $Q$ for $X$ and the same conditional $K(X,\cdot)$ for $Y\mid X$.  By the chain rule for relative entropy in the $(X,Y)$ direction,
\[
        \KL(P\otimes K\|Q\otimes K)
        =\KL(P\|Q)+\E_{X\sim P}\underbrace{\KL(K(X,\cdot)\|K(X,\cdot))}_{=0}
        =\KL(P\|Q)
\]

\emph{Second ordering (condition on $Y$):} By the chain rule for relative entropy in the $(Y,X)$ direction,
\[
        \KL(P\otimes K\|Q\otimes K)
        =\KL(PK\|QK)+
        \E_{Y\sim PK}\KL\!\left((P\otimes K)_{X\mid Y}\|(Q\otimes K)_{X\mid Y}\right)
\]
Equating the two expressions and noting $\mathcal D_K(P,Q)=\E_{Y\sim PK}\KL((P\otimes K)_{X\mid Y}\|(Q\otimes K)_{X\mid Y})\ge0$ proves \eqref{eq:add-channel-kl}.
\end{proof}

\begin{proof}[Proof of Corollary~\ref{cor:add-gaussian-contraction}]
Since $P=\calN(u,\Sigma)$ and $K(x,\cdot)=\calN(x,\Omega)$, the convolution gives $PK=\calN(u,\Sigma+\Omega)$.  Similarly $QK=\calN(m,\Sigma+\Omega)$.  Applying the equal-covariance Gaussian relative-entropy formula $\KL(\calN(a,V)\|\calN(b,V))=\half(a-b)^\top V^{-1}(a-b)$ with $V=\Sigma+\Omega$ proves the first line of \eqref{eq:add-gaussian-kl-contraction}.

For the gap, $\KL(P\|Q)=\half(u-m)^\top\Sigma^{-1}(u-m)$ (same formula with $V=\Sigma$).  The chain-rule identity \eqref{eq:add-channel-kl} gives $\mathcal D_K(P,Q)=\KL(P\|Q)-\KL(PK\|QK)=\half(u-m)^\top[\Sigma^{-1}-(\Sigma+\Omega)^{-1}](u-m)$, which is nonneg\-ative since $\Sigma^{-1}\succeq(\Sigma+\Omega)^{-1}$ whenever $\Omega\succeq0$.
\end{proof}

\begin{proof}[Proof of Corollary~\ref{cor:add-langevin-ledger}]
Apply \eqref{eq:add-noisy-objective-expected} with \(F=U\), \(g_t=\widehat g_t\), and \(\Omega_t=2\beta^{-1}\eta_t\Sigma\).  Then \(\half\tr(\Sigma^{-1}\Omega_t)=\beta^{-1}\eta_t d\)
and \eqref{eq:add-expected-gradient-clock} gives the two gradient terms.
\end{proof}

\begin{proof}[Proof of Proposition~\ref{prop:add-1d-noise-floor}]
The stationary variance solves \(s=(1-
\eta h)^2s+\omega^2\).  The objective identity follows from \(\E U=hs/2\).  Substituting \(\omega^2=2\beta^{-1}\eta\) gives \eqref{eq:add-1d-langevin-excess}.
\end{proof}

\begin{proof}[Proof of Corollary~\ref{cor:add-finite-variance-noise}]
Take conditional expectations in \eqref{eq:add-noisy-ledger} and use $\E\|\xi_t\|_{\Sigma^{-1}}^2 =\tr(\Sigma^{-1}\Omega_t)+\|\bar\xi_t\|_{\Sigma^{-1}}^2$.
\end{proof}

\begin{proof}[Proof of Proposition~\ref{prop:add-cgf-noise}]
Conditional independence gives
\[
\E\left[e^{\ip{v}{\bar\theta_{t+1}+\xi_t}}\mid\calF_t,g_t\right]
=e^{\ip{v}{\bar\theta_{t+1}}}\E\left[e^{\ip{v}{\xi_t}}\mid\calF_t\right]
\]
Taking logarithms gives \eqref{eq:add-cgf-noise-normalizer}.
\end{proof}

\begin{proof}[Proof of Proposition~\ref{prop:add-noisy-pressure}]
By the conditional identity \eqref{eq:add-noisy-ledger-conditional}, the conditional expected point-comparator ledger equals $\eta_t^2Q_t+\half\tr(\Sigma^{-1}\Omega_t)$ whatever the target, so it equals $a_t$ if and only if \eqref{eq:add-noisy-pressure} holds. When $Q_t>0$ and $\Omega_t$ is free of $\eta_t$, \eqref{eq:add-noisy-pressure} is a quadratic in $\eta_t$ with no linear term, whose unique positive root is \eqref{eq:add-noisy-pressure-solution}, real precisely when the numerator $a_t-\half\tr(\Sigma^{-1}\Omega_t)$ is positive. For the Langevin scaling $\Omega_t=2\beta^{-1}\eta_t\Sigma$, the noise term is $\half\tr(\Sigma^{-1}\cdot2\beta^{-1}\eta_t\Sigma)=\beta^{-1}d\,\eta_t$, the balance becomes $\eta_t^2Q_t+\beta^{-1}d\,\eta_t=a_t$, and the quadratic formula gives the stated root.
\end{proof}

\begin{proof}[Proof of Proposition~\ref{prop:add-saddle-second-moment}]
Iterate \eqref{eq:add-saddle-1d} $x_t=(1+ \eta\gamma)^tx_0+ \sum_{k=0}^{t-1}(1+ \eta\gamma)^{t-1-k}\xi_k$.
Taking second moments gives \eqref{eq:add-saddle-second-moment}.
\end{proof}

\begin{proof}[Proof of Proposition~\ref{prop:add-phase-info-recursion}]
Expand \(\mathcal I_{t+1}=\half(A_ts_t+b_t)^\top P(A_ts_t+b_t)\).  Taking conditional expectations removes the cross term and replaces \(b_t^\top Pb_t\) by \(\tr(P\Xi_t)\).
\end{proof}

\section{Notation}\label{app:notation}

\begin{center}
\renewcommand{\arraystretch}{1.2}
\begin{tabular}{@{}ll@{}}
\toprule
\textbf{Symbol} & \textbf{Meaning} \\
\midrule
$d$ & parameter dimension \\
$\theta_t\in\R^d$ & iterate at round $t$, the mean of the Gaussian posterior $q_t$ \\
$f_t$, $F$ & round loss; objective \\
$g_t=\nabla f_t(\theta_t)$, $\hat g_t$ & (sub)gradient at $\theta_t$; its stochastic estimate \\
$\Sigma$, $\Sigma_t$ & fixed covariance; adaptive covariance at round $t$ \\
$\pi=\calN(0,\Sigma)$ & Gaussian prior, with its covariance restated where a section changes it \\
$q_t=\calN(\theta_t,\Sigma)$ & Bayesian posterior at round $t$ \\
$\rho_u=\calN(u,\Sigma)$ & Gaussian comparator centered at $u$ \\
$\eta_t$ & step size, read as the inverse temperature of round $t$ \\
$\alpha_t=\eta_t\sigma^2$ & effective step size under the scalar prior $\Sigma=\sigma^2I$ \\
$c_t(\theta)=\ip{g_t}{\theta}$ & linearized composite loss \\
$Q_t=\half g_t^\top\Sigma_t g_t$ & intrinsic-time density \\
$V_T=\sum_tQ_t$ & cumulative intrinsic time \\
$I_t(u)=\half\|u-\theta_t\|_{\Sigma^{-1}}^2$ & comparator information \\
$B_f(u,\theta)=f(u)-f(\theta)-\ip{\nabla f(\theta)}{u-\theta}$ & first-order Bregman remainder \\
$S_t=\ip{g_t}{\theta_t-r_t}$ & pressure against a reference point $r_t$ \\
$\Gamma$ & comparator budget standing in for $I_1(u)$ \\
$D_T$, $P_T$ & drift and path length of a moving comparator (Proposition~\ref{prop:tracking-ledger}) \\
$\Xi_{t+1}(\rho)$ & covariance-drift term (\S\ref{sec:cov-drift}) \\
$M_T^{\mathrm{est}}$, $\Bias_T$ & estimation martingale (\S\ref{sec:stochastic}) and bias term (Appendix~\ref{app:identities}) \\
$P_-$, $V_T^-$ & projector onto the negative-curvature eigenspace; unstable intrinsic time \\
$H$ & Hessian, at a critical point unless stated otherwise \\
$\nu$, $\nu_t$ & zeroth-order probe direction (\S\ref{sec:zeroth}) \\
$C$ & covariance of a zeroth-order probe \\
$\calK_T(\varepsilon)$ & surface complexity (Appendix~\ref{app:surface-full}) \\
\bottomrule
\end{tabular}
\end{center}

\section*{Acknowledgments}

We acknowledge the use of large language models in preparing this work.

\bibliographystyle{plainnat}
\bibliography{sgd_information_accounting_unified}

\end{document}